%% file: CMRDY_neurips_2026.tex
\documentclass{article}

    \PassOptionsToPackage{numbers, compress}{natbib}
\usepackage[main, final]{neurips_2026}

\usepackage[utf8]{inputenc} % allow utf-8 input
\usepackage[T1]{fontenc}    % use 8-bit T1 fonts
\usepackage{hyperref}       % hyperlinks
\usepackage{url}            % simple URL typesetting
\usepackage{booktabs}       % professional-quality tables
\usepackage{amsfonts}       % blackboard math symbols
\usepackage{nicefrac}       % compact symbols for 1/2, etc.
\usepackage{microtype}      % microtypography
\usepackage{subcaption}
\usepackage{graphicx}
\usepackage{wrapfig}
\usepackage{amsmath}
\usepackage{amssymb}
\usepackage{mathtools}
\usepackage{amsthm}
\usepackage{multirow}
\usepackage{bm}
\usepackage[table]{xcolor}

\definecolor{linkgreen}{RGB}{0,135,85}
\hypersetup{
    colorlinks=true,
    allcolors=linkgreen,
    pdfborder={0 0 0}
}

\newtheorem{assumption}{Assumption}
\newtheorem{proposition}{Proposition}

\title{Tail Cues, Principal Corrections: Plug-and-Play \\ Rectification for Open-Set Test-Time Adaptation}

\author{%
  Yingkai Yang \qquad
  Chaoqi Chen\thanks{Corresponding author.} \qquad
  Hui Huang \\
  Guangdong Provincial Key Laboratory of Visual Media and Multidimensional Intelligence \\ College of Computer Science and Software Engineering, Shenzhen University \\
  \texttt{ekyleyang@gmail.com, cqchen1994@gmail.com, hhzhiyan@gmail.com}
}

\begin{document}

\maketitle

% \vspace{-10pt}
\begin{abstract}
Open-Set Test-Time Adaptation (OSTTA) operates on non-stationary target streams where covariate and semantic shifts coexist. Existing TTA methods often obtain adaptation signals by updating model parameters or maintaining mutable memories, which can be costly and vulnerable to unknown samples. We present \textbf{TALOR}, a lightweight, backward-free rectification module that uses the source linear head as a geometric anchor. The head-induced basis decomposes each normalized feature into high-gain principal coordinates and low-gain tail coordinates. TALOR estimates a soft-routed affine correction in the principal subspace, combining a \textit{regime-level} principal bias with a tail-conditioned slope that captures \textit{sample-level} tail-to-principal coupling. Using weighted ridge regression, it subtracts this drift estimate from the principal coordinates, reconstructs the feature and feeds it into the source head. On the ImageNet-C benchmark with six csOOD datasets, TALOR outperforms the next-best UniEnt by $\mathbf{2.2\%}$ H-score while running $\mathbf{2.5\times}$ faster and using only $\mathbf{12\%}$ of its GPU memory; as a post-correction plug-in, it further boosts COME to $\mathbf{68.4\%}$ H-score with only $\bm{2\%}$ additional memory overhead. Code is available at \url{https://github.com/ekyle0522/TALOR}.
\end{abstract}

\vspace{-5pt}
\section{Introduction}
\vspace{-5pt}
Modern recognition systems are increasingly deployed in non-stationary environments, where test inputs may drift away from the source distribution and may also contain previously unseen semantics. Test-Time Adaptation (TTA) \citep{chen2022contrastive,tang2024domain,yuan2024tea,lee2024entropy,boudiaf2022parameter,niu2024test,gui2024active,wang2025effortless,lee2025test, Gong20264DPC} aims to improve robustness using unlabeled test data at deployment time. Open-Set TTA further requires adapting to covariate-shifted in-distribution (csID) samples while simultaneously rejecting covariate-shifted out-of-distribution (csOOD) samples \citep{gao2024unified}. This setting is challenging for common TTA strategies: gradient updates and cache-based modules rely on online adaptation signals that can be distorted by unknown samples, while their additional optimization or memory overhead can be undesirable in real-time deployment. These limitations motivate a lightweight alternative that corrects the test-time representation itself without updating model parameters.

A domain shift is expressed not only in the input space but also in the latent space. When the feature direction deviates from the source geometry, even a small displacement can damage predictions if it lies along directions that the classifier amplifies. This suggests that feature correction should not be treated as an unconstrained high-dimensional update: the linear head already identifies latent space directions to which its outputs are most sensitive. This observation raises a feature-level question:
\begin{center}
\textit{Can the response of latent features to domain shift be approximated from the unlabeled non-stationary stream  and then removed without updating the model?}
\end{center}

Most classification models end with a linear head $h(v)=Wv+b$, which is usually kept fixed during test-time adaptation to preserve the source semantic decision boundary. We observe that this fixed head is not merely a readout layer. It induces a task-aligned coordinate system over penultimate features, and we use its singular spectrum to partition the axes into high-gain and low-gain components. Concretely, the dominant right-singular vectors of $W$ form a \textbf{principal subspace}: changes along these axes are magnified by large singular values and therefore can strongly affect logits, margins, and confidence. The remaining basis vectors form a \textbf{tail subspace}: they have lower direct gain under the classifier, but remain sensitive to target shifts. In online batches, the principal coordinates can contain a removable domain response: a shared principal mean together with sample-specific deviations that are statistically predictable from tail coordinates. We call the latter coupling \textbf{tail-to-principal leakage}. This leads to a latent view of adaptation: domain shift \textit{displaces features relative to the geometry} assumed by the fixed source head, and adaptation should restore their compatibility with this geometry.

Building on this motivation, we propose \texttt{TALOR} (\textbf{T}est-time \textbf{A}daptation via \textbf{L}inear \textbf{O}nline \textbf{R}ectification), a plug-and-play rectification module inserted between the backbone feature and the linear head.
TALOR uses the head to decompose normalized features into high-gain principal coordinates and low-gain tail coordinates.
It then treats the principal coordinates as correction targets and the tail coordinates as shift-sensitive cues, estimating an affine principal drift online in mixed csID/csOOD streams.
The rectified feature is then evaluated by the unchanged source head. TALOR is lightweight because it avoids backward passes, augmentation loops, and long-horizon caches. It is also composable: since it only post-corrects features, it can be applied to a frozen source model or stacked on top of backbone-updating TTA methods as a stabilizing post-correction (Fig.~\ref{fig:talor-complexity-bubble}).

Our contributions are summarized as follows:
\begin{itemize}
    \item We identify a head-induced feature structure for OSTTA: domain shift induces a removable response in classifier-amplified principal coordinates, while low-gain tail coordinates provide shift-sensitive cues for estimating this response.

    \item We propose \textbf{TALOR}, a plug-and-play, backward-free rectification module that couples this classifier-head-induced sensor--target decomposition with soft-routed, tail-conditioned estimation of affine principal drift, removing the drift before the unchanged source head without source data or source statistics.

    \item TALOR is lightweight and broadly effective: it surpasses next-best UniEnt as a standalone method by $\bm{2.2\%}$ H-score, boosts COME by $\bm{4.7\%}$ with only $\bm{2\%}$ memory overhead, and generalizes beyond open-set adaptation to closed-set and dynamic continual protocols.
\end{itemize}

\begin{figure*}[t]
    \centering
    \includegraphics[width=0.95\textwidth]{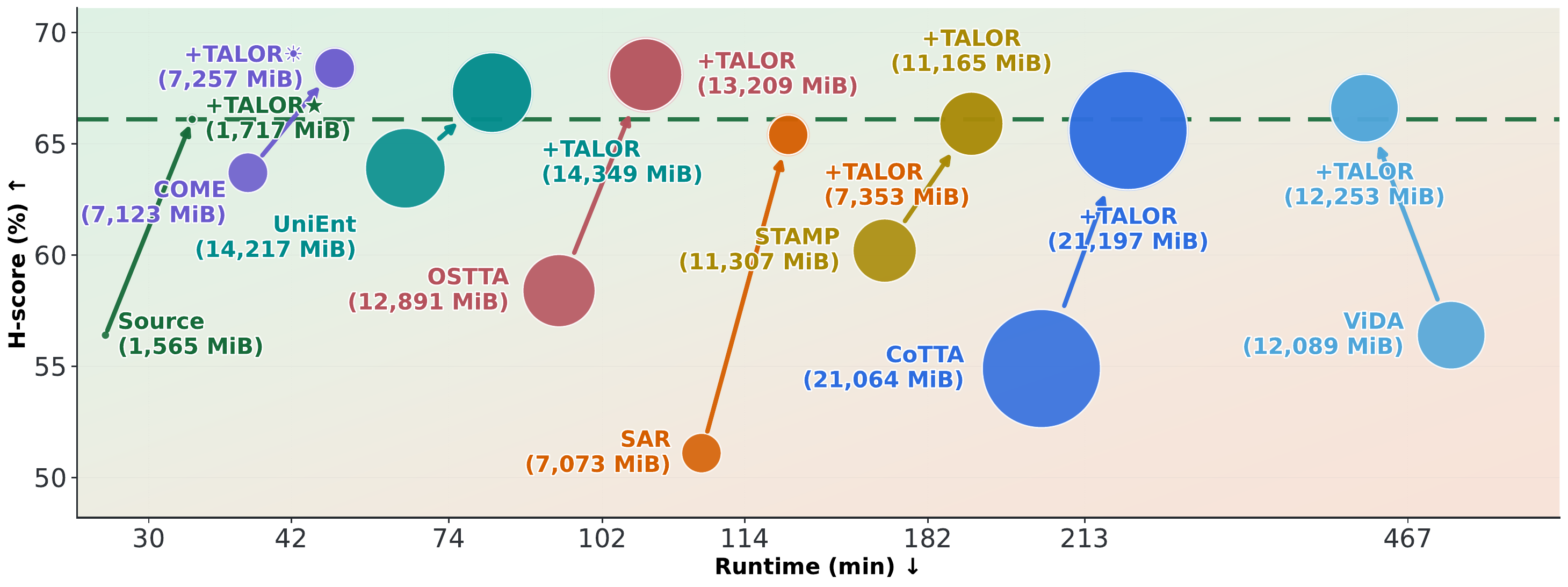}
    \caption{Runtime--performance--memory comparison of TALOR and representative baselines. Circle size denotes GPU usage. The dashed line marks \textbf{Source+TALOR}: even as a standalone plug-in, it surpasses existing methods while using the least runtime and memory. When attached to other TTA methods, TALOR substantially improves their performance with only marginal additional budgets.}
    \label{fig:talor-complexity-bubble}
    \vspace{-10pt}
\end{figure*}

\vspace{-5pt}
\section{Related Work}
\vspace{-5pt}
Prior works on test-time adaptation can be broadly grouped into two lines: continual TTA (CTTA) and open-set TTA (OSTTA). CTTA methods aim to adapt online to evolving covariate shifts while avoiding error accumulation and catastrophic forgetting, typically through regularization, parameter-efficient updates, or reset-based stabilization. OSTTA methods further consider semantic shifts, requiring the model to adapt to known classes while reliably rejecting unknown samples. Closest in spirit, SSA~\citep{adachi2025testtime} aligns source and target statistics within an output-significant PCA subspace of source features; TALOR instead uses a head-induced basis, keeps the complementary tail as a drift sensor, and removes drift sample-wise without backpropagation or source statistics. Appendix~\ref{app:related_work} details representative methods.

\vspace{-5pt}
\section{Preliminaries}
\vspace{-5pt}
\textbf{Problem Setup.} We study Open-Set TTA on non-stationary target streams where covariate shift and unknown semantics coexist. A model consists of a backbone $g:\mathcal{X}\rightarrow\mathbb{R}^D$ mapping an input $x$ to a feature vector $v=g(x)\in\mathbb{R}^D$, and a linear classifier head $h(v)=Wv+b$, with $W\in\mathbb{R}^{C\times D}$ and $b\in\mathbb{R}^{C}$.
In general, the head parameters $\{W, b\}$ typically remain fixed to preserve the semantic decision boundary.
However, the backbone $g$ may essentially operate in two modes: it is either \textit{frozen} or \textit{continuously updated}. Our method is agnostic to this choice, as it operates directly on the streaming feature $v$. The source logits are $z(x)=Wg(x)+b$. At test time $t$, each step provides an unlabeled mini-batch $\mathcal{B}_t=\{x_i\}_{i=1}^{B}$ drawn from a mixture distribution
\vspace{-2pt}
\begin{equation}
p_t(x)=(1-\pi_t)p^{\mathrm{id}}_t(x)+\pi_t p^{\mathrm{ood}}_t(x),
\end{equation}
\vspace{-2pt}
where $p^{\mathrm{id}}_t$ denotes the csID component whose labels belong to the source label set $\mathcal{Y}_{\mathrm{id}}=\{1,\ldots,C\}$, and $p^{\mathrm{ood}}_t$ denotes the csOOD component with unknown semantics relative to $\mathcal{Y}_{\mathrm{id}}$. The stream is non-stationary because both $p^{\mathrm{id}}_t$ and $p^{\mathrm{ood}}_t$ may change over time. The goal of OSTTA is to produce reliable ID predictions for samples from $p^{\mathrm{id}}_t$ while rejecting samples from $p^{\mathrm{ood}}_t$, using only unlabeled test inputs and without access to source training data.

\begin{figure*}[t]
  \centering
  \includegraphics[width=0.95\textwidth]{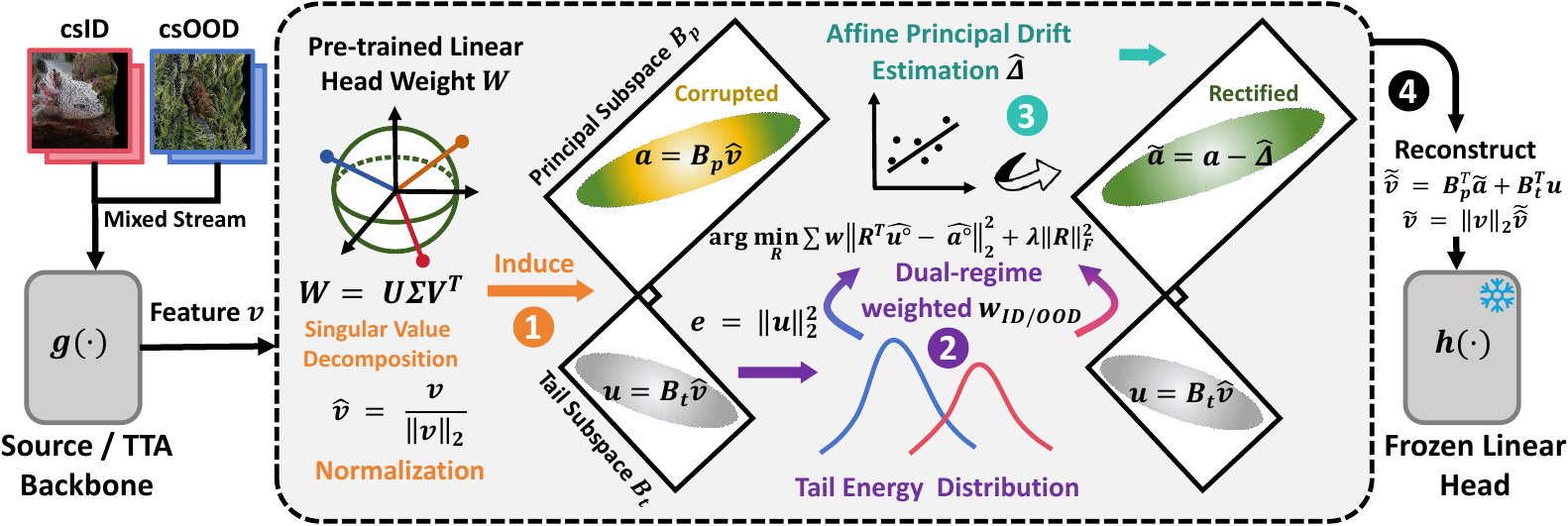}
  \caption{
  Pipeline.
  TALOR is a gradient-free layer between feature extractor and source linear head.
  \textcircled{1} \textbf{Head-induced coordinates:} the normalized feature $\hat v$ is projected into principal coordinates $a$ and tail coordinates $u$ using the right-singular basis of $W$.
  \textcircled{2} \& \textcircled{3}  \textbf{Affine principal drift:} TALOR establishes the dual-regime soft routing via tail-energy, then models removable principal drift.
  \textcircled{4} \textbf{Rectified inference:} The estimated principal drift is subtracted, then the feature is reconstructed from the unchanged tail coordinates and the original feature-norm scale before applying the classifier.}
  \label{fig:main_pipeline}
  \vspace{-10pt}
  
\end{figure*}

\vspace{-5pt}
\section{Proposed Method}
\vspace{-5pt}

\textbf{Overview.}
TALOR performs test-time adaptation by post-correcting features (Fig.~\ref{fig:main_pipeline}). The module relies on a head-induced asymmetry: high-gain principal coordinates serve as correction targets, while low-gain tail coordinates provide cues for estimating the domain-induced response.

% \vspace{-5pt}
\subsection{Head-induced sensor--target decomposition}
% \vspace{-2pt}
\label{sec:head-induced-coordinates}

The frozen linear head is usually treated as a readout layer in test-time adaptation, but its weight matrix also defines an anisotropic coordinate system over classifier-input features.
Directions associated with large singular values are strongly amplified in the logits, so small perturbations along these axes can substantially change margins, confidence, and open-set scores.
In contrast, low-gain directions have weaker direct logit influence, but they can still respond to distributional shift.
This motivates a sensor--target decomposition: principal coordinates are the correction targets, while tail coordinates are used as cues for estimating how the principal coordinates should be corrected.

Formally, let the singular value decomposition of the frozen head weight be
\vspace{-2pt}
\begin{equation}
W=U\Sigma V^{\top},
\label{eq:head-svd}
\end{equation}
% \vspace{-2pt}
where $W\in\mathbb{R}^{C\times D}$, $V\in\mathbb{R}^{D\times D}$ contains the right singular vectors, and the singular values are zero-padded to $\sigma_1,\dots,\sigma_D$ with $\sigma_1\ge \cdots \ge \sigma_D\ge 0$.
Let $s=\min(C,D)$. We choose a split point $k$ using an energy threshold $\tau\in(0,1)$:
\begin{equation}
k_0=\min\left\{m:\ \frac{\sum_{j=1}^{m}\sigma_j^2}{\sum_{j=1}^{D}\sigma_j^2}\ge \tau \right\},\qquad k=\min\{k_0,s-1\}.
\label{eq:split-k}
\end{equation}
The first $k$ right-singular vectors span the principal basis, and the remaining vectors form the tail basis:
\begin{equation}
B_p = V_{:,1:k}^{\top}\in\mathbb{R}^{k\times D},\qquad
B_t = V_{:,k+1:D}^{\top}\in\mathbb{R}^{(D-k)\times D}.
\label{eq:basis}
\end{equation}
For later use, denote $d=D-k$.
The two bases satisfy $B_pB_p^{\top}=I_k$, $B_tB_t^{\top}=I_d$, and $B_p^{\top}B_p+B_t^{\top}B_t=I_D$ \citep{eckart1936approximation}.
When $D>C$, head-null directions are automatically included in the tail basis; App.~\ref{app:implementation-details} details how the same construction applies to both $D<C$ and $D>C$ backbones.

For a streaming feature $v_i=g(x_i)$, TALOR first separates its norm and direction:
\begin{equation}
\rho_i=\|v_i\|_2,\qquad
\hat v_i=\frac{v_i}{\rho_i},\qquad
a_i=B_p\hat v_i\in\mathbb{R}^{k},\qquad
u_i=B_t\hat v_i\in\mathbb{R}^{d}.
\label{eq:coordinates}
\end{equation}
Here $a_i$ and $u_i$ are the principal and tail coordinates of the feature direction.
Since $(B_p,B_t)$ form an orthonormal basis of the feature space, the normalized feature can be reconstructed as
\begin{equation}
\hat v_i = B_p^{\top}a_i+B_t^{\top}u_i.
\label{eq:normalized-reconstruction}
\end{equation}
The corresponding source logits admit the decomposition
\vspace{-2pt}
\begin{equation}
z_i=Wv_i+b=\rho_i\,U\Sigma
\begin{bmatrix}
a_i\\ u_i
\end{bmatrix}+b.
\label{eq:logit-decomposition}
\end{equation}

\begin{figure*}[t]
    \centering
    \begin{subfigure}[t]{0.31\textwidth}
        \centering
        \includegraphics[width=\linewidth]{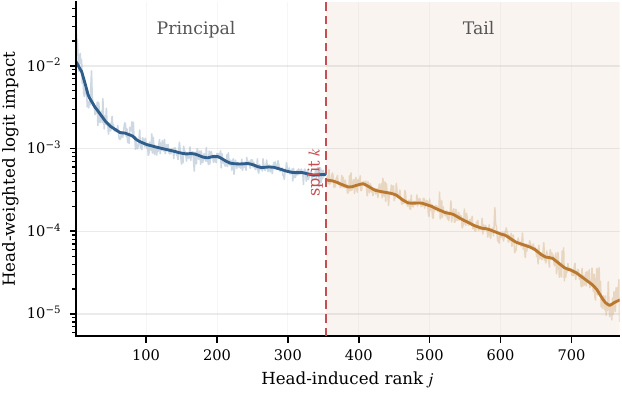}
        \caption{Logit impact.}
        \label{fig:shift-spectrum-logit}
    \end{subfigure}
    \hfill
    \begin{subfigure}[t]{0.31\textwidth}
        \centering
        \includegraphics[width=\linewidth]{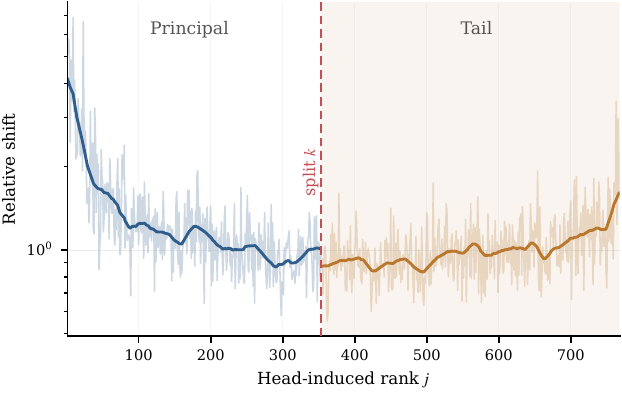}
        \caption{Relative shift.}
        \label{fig:shift-spectrum-relshift}
    \end{subfigure}
    \hfill
    \begin{subfigure}[t]{0.31\textwidth}
        \centering
        \includegraphics[width=\linewidth]{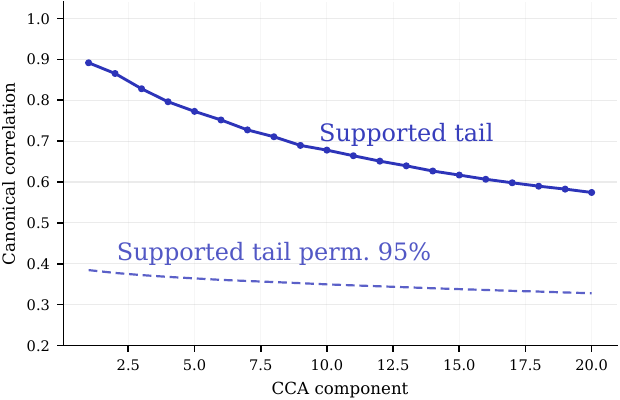}
        \caption{Tail-to-principal CCA.}
        \label{fig:vit-cca-leakage}
    \end{subfigure}
    \caption{Paired clean-corrupted diagnostics in the head-induced basis. The results show that principal coordinates dominate logit impact (a), while tail coordinates remain shift-responsive (b) and linearly predictive of principal drift (c). 
The vertical dashed line marks the $\tau$ split.}
    \label{fig:shift-spectrum}
      \vspace{-10pt}
    
\end{figure*}

Eq.~\ref{eq:logit-decomposition} makes the asymmetry explicit: the same amount of coordinate drift has a much larger logit effect in high-singular-value principal directions than in low-gain tail directions.

We verify this structure with paired clean-corrupted diagnostics in Fig.~\ref{fig:shift-spectrum}; App.~\ref{app:paired-diagnostics} defines the paired protocol and diagnostic metrics.
Let $c_i=V^{\top}\hat v_i$ denote the full head-induced coordinate vector and define the coordinate drift $\Delta c_{ij}=c_{ij}^{\mathrm{corr}}-c_{ij}^{\mathrm{clean}}$.
We measure the head-weighted logit impact of the normalized feature drift $\sigma_j^2\mathbb{E}_i[(\Delta c_{ij})^2]$, the relative shift $\mathbb{E}_i[(\Delta c_{ij})^2]/(\mathrm{Var}_i(c_{ij}^{\mathrm{clean}})+\epsilon)$, and the canonical correlation between tail drift $\Delta u$ and principal drift $\Delta a$.
The diagnostics show that principal coordinates dominate logit impact, tail coordinates remain shift-responsive, and tail drift linearly decodes a nontrivial component of principal drift.
TALOR therefore corrects only the principal coordinates and keeps the tail coordinates as shift sensors rather than correction targets: in the oracle ablation of App.~\ref{app:oracle-target-ablation}, restoring the principal coordinates alone closes $97.0\%$ (ViT-B/16) and $96.7\%$ (ResNet-50) of the clean--corrupt accuracy gap, while additionally restoring the tail adds only $1.2$ and $1.8$ pp.

Given an estimated principal drift $\hat\Delta_i\in\mathbb{R}^{k}$, TALOR applies the principal correction
\begin{equation}
\tilde a_i=a_i-\hat\Delta_i.
\label{eq:target-principal-rectification}
\end{equation}
We next model $\hat\Delta_i$, estimate it from unlabeled mixed streams, and remove it at inference.

\vspace{-5pt}
\subsection{Regime-conditioned affine principal drift}
\label{sec:affine-leakage}

Open-set target batches contain two latent regimes: covariate-shifted in-distribution samples and covariate-shifted out-of-distribution samples.
We denote the regimes by $\mathcal{Q}=\{\mathrm{ID},\mathrm{OOD}\}$. These regimes need not share the same principal-coordinate mean or the same tail-to-principal coupling.
Known and unknown class samples can induce different shared principal offsets, and the principal deviation predictable from tail coordinates can also vary across regimes.
A single global rectifier would average these effects; TALOR therefore maintains a regime-conditioned affine drift model.

For a regime $q\in\mathcal{Q}$, TALOR models the removable principal drift as
\begin{equation}
\Delta_i^{(q)} \approx
\underbrace{\mu_a^{(q)}}_{\text{principal mean}}
+
\underbrace{R^{(q)\top}\left(u_i-\mu_u^{(q)}\right)}_{\text{tail-to-principal slope}},
\label{eq:leakage-model}
\end{equation}
where $\mu_a^{(q)}\in\mathbb{R}^{k}$ is the regime-level principal mean, $\mu_u^{(q)}\in\mathbb{R}^{d}$ is the regime-level tail mean, and $R^{(q)}\in\mathbb{R}^{d\times k}$ is a tail-to-principal regression matrix.
We use \emph{affine principal drift} to denote the whole removable term in Eq.~\ref{eq:leakage-model}; we use \emph{tail-to-principal leakage} specifically for the tail-conditioned slope component.
Because the slope is fitted on centered coordinates, it captures only sample-specific deviations and cannot absorb a branch-level shift.
The principal mean $\mu_a^{(q)}$ therefore provides a branch-specific translation in the normalized principal coordinates; samples with identical routing weights receive the same translation and retain their pairwise differences (App.~\ref{app:mean-safety}).
The slope term $R^{(q)\top}(u_i-\mu_u^{(q)})$ captures the sample-specific principal deviation linearly predictable from centered tail coordinates.

\textbf{Why tail cues?}
A centered principal deviation $a_i-\mu_a^{(q)}$ mixes useful class/instance variation with domain-induced nuisance; without labels, paired clean samples, or source statistics, a principal-only estimator has no external evidence for which part of this deviation to remove.
Tail coordinates instead offer an auxiliary low-gain view
of the same shift, allowing the tail-conditioned slope
to target principal variation supported by
cross-subspace evidence.
Since singular values establish only the impact asymmetry, this requires an explicit condition on the centered coordinates $a_i^{\circ}=a_i-\mu_a^{(q)}$ and $u_i^{\circ}=u_i-\mu_u^{(q)}$.

\begin{assumption}[Shared nuisance]
\label{asm:shared-nuisance}
\normalfont
Within each regime $q$, the centered coordinates admit
the shared-nuisance decomposition
\[
a_i^\circ=L_q\xi_i+s_i+\varepsilon_i^a,
\qquad
u_i^\circ=T_q\xi_i+\varepsilon_i^u.
\]
Here $\xi_i$ is a zero-mean shift-induced nuisance,
$s_i$ represents principal semantic (class/instance)
variation, and $\varepsilon_i^a,\varepsilon_i^u$
are residuals. The tail cue is uncorrelated with
$s_i$ and $\varepsilon_i^a$, i.e.,
$\mathbb{E}_q[u_i^\circ s_i^\top]
=\mathbb{E}_q[u_i^\circ(\varepsilon_i^a)^\top]=0$.
\end{assumption}

\begin{proposition}[Tail-conditioned nuisance identification]
\label{prop:nuisance-regression}
Under Assumption~\ref{asm:shared-nuisance}, let
$\Sigma_{\xi u}=\mathbb{E}_q[\xi_i u_i^{\circ\top}]$
and
$\Sigma_{uu}=\mathbb{E}_q[u_i^\circ u_i^{\circ\top}]$.
For any $\lambda\geq 0$, the population ridge estimator
\begin{equation}
R_{q,\lambda}^{*}
=
\arg\min_R
\mathbb{E}_q\big[
\|a_i^\circ-R^\top u_i^\circ\|_2^2
\big]
+\lambda\|R\|_F^2
\end{equation}
(with the minimum-norm solution at $\lambda=0$)
satisfies
\begin{equation}
R_{q,\lambda}^{*\top}u_i^\circ
=
L_q\Sigma_{\xi u}
(\Sigma_{uu}+\lambda I_d)^\dagger u_i^\circ
\in\mathrm{col}(L_q).
\end{equation}
Thus, the predicted component is uncorrelated with
$s_i$, which remains in the residual.
At $\lambda=0$, the prediction is the minimum-MSE
linear estimate of $L_q\xi_i$ from $u_i^\circ$.
It recovers $L_q\xi_i$ exactly if
$\varepsilon_i^u=0$,
$\operatorname{Cov}_q(\xi_i)\succ 0$,
and $T_q$ has full column rank.
\end{proposition}

Proposition~\ref{prop:nuisance-regression} characterizes
the population target of tail-conditioned regression.
TALOR approximates this mechanism through finite-batch,
soft-weighted ridge regression and EMA updates;
the complete proof and scope are given in
App.~\ref{app:proofs}.

Independent structural support comes from SoMA~\citep{Yun_2025_CVPR}, whose SVD analysis of pre-trained weights finds general, class-shared features in large singular components and texture-biased, context-specific features in minor ones (App.~\ref{app:related_work}).
Since tail-correlated semantic variation can also be removed, we assess the empirical plausibility of the shared-nuisance interpretation through complementary covariance, correction, and negative-control diagnostics (App.~\ref{app:drift-checks}): tail and principal drifts are coupled well above a permutation baseline (Fig.~\ref{fig:vit-cca-leakage}), the observed tail covaries more strongly with the corruption-induced principal displacement than with the clean principal coordinates for every corruption, and without covariate shift TALOR stays within $0.5$ pp of Source accuracy (App.~\ref{app:negative-controls}).

\vspace{-5pt}
\subsection{Soft routing and online affine estimation}
\label{sec:routing-estimation}
\textbf{Tail-energy responsibilities.}
The affine model in Eq.~\ref{eq:leakage-model} is regime-conditioned, but the online stream is unlabeled.
TALOR therefore estimates the two regime models with soft responsibilities rather than hard pseudo-labels.
Following the use of residual-subspace energy in \citep{sanqing2024LEAD}, we use the tail energy
\[
e_i=\|u_i\|_2^2
\]
as a weak routing statistic.
Given a mini-batch $\mathcal{B}_t$, TALOR fits a two-component Gaussian mixture model to $\{e_i\}_{i=1}^{B}$.
Let the mixture parameters be $(\alpha_1,m_1,s_1^2)$ and $(\alpha_2,m_2,s_2^2)$, where $\alpha_j$ is the mixture weight, $m_j$ is the component mean, and $s_j^2$ is the component variance.
The posterior responsibility of component $j$ for sample $i$ is
\begin{equation}
r_{ij}=\frac{\alpha_j\,\mathcal{N}(e_i\mid m_j,s_j^2)}
{\sum_{\ell=1}^{2}\alpha_\ell\,\mathcal{N}(e_i\mid m_\ell,s_\ell^2)}.
\label{eq:gmm-responsibility}
\end{equation}
The tail energy measures how much the normalized feature occupies the low-gain tail subspace induced by the source head.
Although corrupted csID samples are shifted, they still belong to the source label space and empirically tend to preserve stronger occupancy in this tail subspace.
In contrast, csOOD samples lack a matching source class and typically exhibit weaker normalized occupancy there.
We therefore assign the lower-energy component to the OOD-dominated branch,
$j_{\mathrm{OOD}}=\arg\min_{j\in\{1,2\}}m_j$, and define
\begin{equation}
w_i^{(\mathrm{OOD})}:=r_{i j_{\mathrm{OOD}}},\qquad
w_i^{(\mathrm{ID})}:=1-w_i^{(\mathrm{OOD})}.
\label{eq:soft-weights}
\end{equation}
The GMM is only a weak router on the scalar $e_i$, not an assumption that classes or corruption modes are Gaussian~\citep{gao2024unified,li2023robustness}; its weights serve solely as regression responsibilities and to mix drift estimates, so tail energy is a lightweight soft-routing cue rather than a universally calibrated ID/OOD score, and TALOR does not require perfect open-set routing (Sec.~\ref{sec:mechanism-analysis}).

\textbf{Weighted ridge estimation.}
For each branch $q\in\mathcal{Q}$, TALOR first computes branch-weighted batch means:
\begin{equation}
\bar u^{(q)}=\frac{\sum_{i=1}^{B} w_i^{(q)}u_i}{\sum_{i=1}^{B} w_i^{(q)}},\qquad
\bar a^{(q)}=\frac{\sum_{i=1}^{B} w_i^{(q)}a_i}{\sum_{i=1}^{B} w_i^{(q)}}.
\label{eq:weighted-means}
\end{equation}
Define the batch-centered coordinates
$\hat{u}_i^{\circ(q)}=u_i-\bar u^{(q)}$ and
$\hat{a}_i^{\circ(q)}=a_i-\bar a^{(q)}$,
as empirical counterparts of the regime-centered coordinates
in Sec.~\ref{sec:affine-leakage}.
Let $X^{(q)}\in\mathbb{R}^{B\times d}$ and $Y^{(q)}\in\mathbb{R}^{B\times k}$ be row-stacked matrices whose $i$-th rows are $(\hat{u}_i^{\circ(q)})^{\top}$ and $(\hat{a}_i^{\circ(q)})^{\top}$, respectively, and let
$\Omega^{(q)}=\mathrm{diag}(w_1^{(q)},\dots,w_B^{(q)})$.
For $\lambda>0$, TALOR estimates the tail-conditioned slope by weighted ridge regression:
\begin{equation}
\hat R^{(q)}=\arg\min_R
\left\| (\Omega^{(q)})^{1/2}\left(X^{(q)}R-Y^{(q)}\right)\right\|_F^2
+\lambda\|R\|_F^2.
\label{eq:ridge-objective}
\end{equation}
The primal closed-form solution is
\begin{equation}
\hat R^{(q)}=
\left(X^{(q)\top}\Omega^{(q)}X^{(q)}+\lambda I_d\right)^{-1}
X^{(q)\top}\Omega^{(q)}Y^{(q)}.
\label{eq:close_form}
\end{equation}
\vspace{-2pt}
Because the ridge problem is solved on centered coordinates, $\hat R^{(q)}$ estimates only the slope between tail and principal deviations.
The principal mean $\bar a^{(q)}$ is retained separately as the affine intercept.

\textbf{Efficient small-batch solver.}
Streaming TTA often uses small mini-batches, while the tail dimension $d$ can be several hundred or larger.
We therefore solve Eq.~\ref{eq:ridge-objective} in the dual.
Let
$X_w^{(q)}=(\Omega^{(q)})^{1/2}X^{(q)}$ and
$Y_w^{(q)}=(\Omega^{(q)})^{1/2}Y^{(q)}$.
Using the push-through identity \citep{woodbury1950inverting,henderson1981deriving}, we obtain
\begin{equation}
\hat R^{(q)}=(X_w^{(q)})^{\top}
\left(X_w^{(q)}(X_w^{(q)})^{\top}+\lambda I_B\right)^{-1}Y_w^{(q)}.
\label{eq:dual-ridge}
\end{equation}
In practice, we solve the corresponding symmetric positive definite linear system without forming an explicit inverse \citep{higham2002accuracy}.
The dual formulation reduces the matrix factorization cost from $O(d^3)$ to $O(B^3)$ when $B\ll d$ \citep{hastie2005elements}.

\textbf{Online update.}
To smooth noisy mini-batch estimates while tracking non-stationary streams, TALOR updates the branch states with exponential moving averages:
\begin{equation}
\mathbf{s}_t^{(q)} \leftarrow \beta \mathbf{s}_{t-1}^{(q)} + (1-\beta)\hat{\mathbf{s}}^{(q)},
\label{eq:ema-update}
\end{equation}
where $\mathbf{s}_t^{(q)}=(R_t^{(q)},\mu_{u,t}^{(q)},\mu_{a,t}^{(q)})$ denotes the branch states, $\hat{\mathbf{s}}^{(q)}=(\hat R^{(q)},\bar u^{(q)},\bar a^{(q)})$ contains the corresponding mini-batch estimates, and $\beta\in[0,1)$ is the EMA momentum.

\vspace{-5pt}
\subsection{Rectified inference}
\label{sec:rectified-inference}

After the online states are updated, TALOR applies the learned affine rectifier without modifying either the backbone or the source head.
For each sample, it first predicts branch-specific principal drift estimates and then averages them using the current soft responsibilities:
\begin{equation}
\hat\Delta_i^{(q)}
=
\mu_{a,t}^{(q)}
+
R_t^{(q)\top}\!\left(u_i-\mu_{u,t}^{(q)}\right),
\quad q\in\mathcal{Q},
\qquad
\hat\Delta_i
=
\sum_{q\in\mathcal{Q}} w_i^{(q)}\hat\Delta_i^{(q)} .
\label{eq:affine-correction-prediction}
\end{equation}
The resulting drift estimate is subtracted only from the principal coordinates.
The tail coordinates are left unchanged; the corrected coordinate vector is mapped back to feature space and rescaled by the original feature norm $\rho_i$, which acts as a norm anchor rather than an exact constraint since $\|\tilde{\hat v}_i\|_2$ need not equal $1$, before applying the frozen classifier:
\begin{equation}
\tilde a_i = a_i-\hat\Delta_i,\qquad
\tilde{\hat v}_i = B_p^{\top}\tilde a_i+B_t^{\top}u_i,\qquad
\tilde v_i = \rho_i\tilde{\hat v}_i,\quad
\tilde z_i = W\tilde v_i+b .
\label{eq:rectified-inference}
\end{equation}

\vspace{-5pt}
\subsection{Mechanism Analysis}
\label{sec:mechanism-analysis}

\begin{wrapfigure}[16]{r}{0.64\textwidth}
    \vspace{-1.0em}
    \centering
    \begin{subfigure}[t]{0.49\linewidth}
        \centering
        \includegraphics[width=\linewidth]{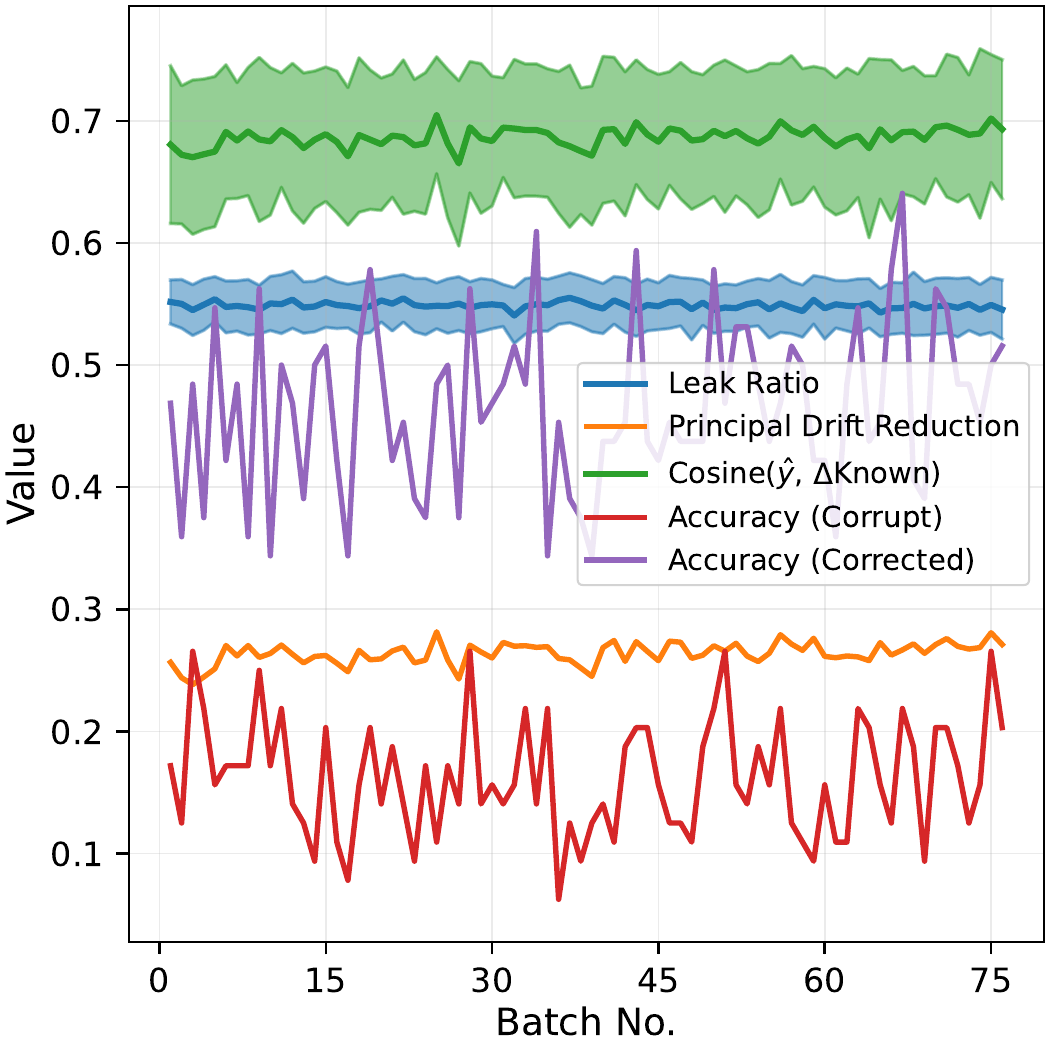}
        \caption{Correction diagnostics.}
        \label{fig:leakage_and_correct}
    \end{subfigure}
    \hfill
    \begin{subfigure}[t]{0.49\linewidth}
        \centering
        \includegraphics[width=\linewidth]{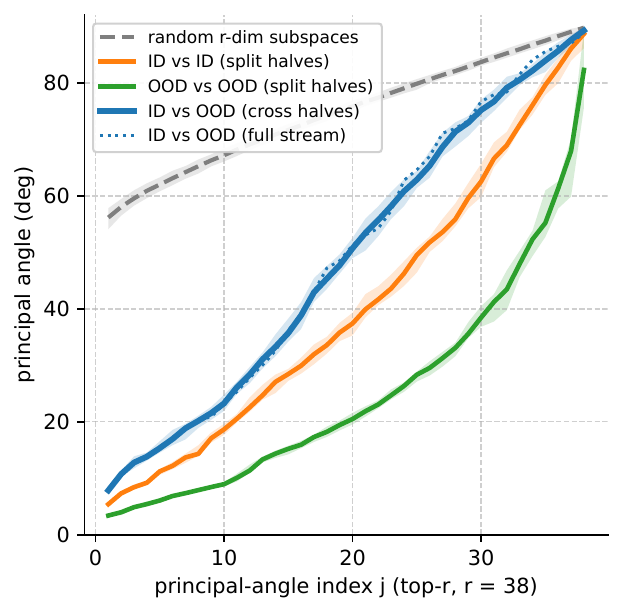}
        \caption{ID/OOD slope subspaces.}
        \label{fig:dual_leakage_mode}
    \end{subfigure}
    \caption{
    (a) Contrast's clean-corrupted analysis of the predicted affine principal drift.
    (b) Energy-truncated principal-angle spectrum between ID and OOD tail-to-principal slope subspaces.}
    \label{fig:leakage-diagnostics-wrap}
    \vspace{-10pt}
\end{wrapfigure}

We further examine whether the online affine drift estimated by TALOR captures a removable principal component.

\textbf{Protocol.}
A frozen ViT-B/16 processes the 15 ImageNet-C corruptions at severity 5 (5{,}000 clean/corrupted pairs each, batch size 64); only the unlabeled corrupted stream drives TALOR, while clean views define the oracle drift $\Delta a_i=a_i^{\mathrm{corr}}-a_i^{\mathrm{clean}}$.
We report the cosine between $\hat\Delta_i$ and $\Delta a_i$ and the principal-drift reduction $\mathrm{PDR}=(\|\Delta a_i\|_2-\|\Delta a_i-\hat\Delta_i\|_2)/(\|\Delta a_i\|_2+\epsilon)$ (App.~\ref{app:paired-diagnostics}).

\textbf{Removing principal drift.}
Fig.~\ref{fig:leakage-diagnostics-wrap}(a) shows that the predicted drift aligns with the oracle principal drift, reduces the measured drift after subtraction, and improves corrupted-sample accuracy, supporting the central modeling assumption that a substantial part of the harmful high-gain shift is predictable from online tail-to-principal statistics.
The slope alone reduces the squared principal drift for all 15 corruptions, and the pattern persists in the open-set benchmark stream (App.~\ref{app:component-alignment} and~\ref{app:drift-checks}).

\textbf{Shared and regime-specific slopes.}
Fig.~\ref{fig:leakage-diagnostics-wrap}(b) compares the energy-truncated top-$r$ left-singular subspaces of the ID and OOD slope maps (App.~\ref{app:leakage-modes}).
Each subspace retains at least 90\% of its slope-map Frobenius energy ($r=38$ for the displayed cell), which avoids dimension-forced zero angles in the full column spaces.
Over the first ${\sim}10$ directions, the ID/OOD cross-half spectrum stays far below the same-dimensional Haar-random reference and close to the ID split-half spectrum; beyond them, it departs toward the random reference.
This indicates a shared leading-energy component together with residual regime-specific directions rather than equality of the full slope maps.

This evidence is consistent with the routing ablation in Tab.~\ref{tab:moduleabla}: a single affine rectifier can exploit the shared leading component, while soft routing can refine residual ID/OOD-specific structure.
It also helps explain why TALOR is less brittle when unsupervised routing is imperfect.

\vspace{-5pt}
\section{Experiments}
\label{sec:experiments}
% \vspace{-5pt}

% Requires: \usepackage{booktabs,multirow,xcolor}
% Softer colors (deep & less neon)
\definecolor{talorgreen}{HTML}{1B5E20}
\definecolor{talorred}{HTML}{8B1A1A}
\begin{table*}[t]
\centering
\caption{Results ($\%$) for ImageNet-C benchmark across six csOOD datasets. \textcolor{talorgreen}{Improvement} and \textcolor{talorred}{degradation} compared to the baseline are highlighted in green or red respectively.  All the results are \textbf{averaged over 15} corruptions.}
\label{tab1:imgc}
\setlength{\tabcolsep}{1.0pt}
\renewcommand{\arraystretch}{0.3}
\setlength{\aboverulesep}{0.5pt}
\setlength{\belowrulesep}{0.5pt}
\setlength{\cmidrulesep}{0.5pt}
{\small
\begin{tabular}{ccc@{\hspace{8pt}}cc@{\hspace{8pt}}cc@{\hspace{8pt}}cc@{\hspace{8pt}}cc@{\hspace{8pt}}cc@{\hspace{12pt}}ccc}
\toprule
\multicolumn{1}{c}{\multirow{2}{*}{Method}} & \multicolumn{2}{c}{Places.-C} & \multicolumn{2}{c}{Textures-C} & \multicolumn{2}{c}{iNatur.-C} & \multicolumn{2}{c}{SUN-C} & \multicolumn{2}{c}{SSB-H.C} & \multicolumn{2}{c}{NINCO--C} & \multicolumn{3}{c}{Mean} \\
\cmidrule(lr){2-3} \cmidrule(lr){4-5} \cmidrule(lr){6-7} \cmidrule(lr){8-9} \cmidrule(lr){10-11} \cmidrule(lr){12-13} \cmidrule(lr){14-16}
\multicolumn{1}{c}{} & Acc & AUC & Acc & AUC & Acc & AUC & Acc & AUC & Acc & AUC & Acc & AUC & Acc & AUC & H-s \\
\midrule
Source & 49.8 & 66.9 & 49.8 & 70.6 & 49.8 & 78.6 & 49.8 & 71.4 & 49.8 & 56.2 & 49.8 & 64.5 & 49.8 & 68.0 & 56.4 \\
\rowcolor{gray!20}
+ TALOR & 56.0 & 80.1 & 55.8 & 85.7 & 55.8 & 93.4 & 55.9 & 91.6 & 55.9 & 70.5 & 55.9 & 79.6 & 55.9 & 83.5 & 66.1 \\
Improve & \textcolor{talorgreen}{+6.2} & \textcolor{talorgreen}{+13.2} & \textcolor{talorgreen}{+6.0} & \textcolor{talorgreen}{+15.1} & \textcolor{talorgreen}{+6.0} & \textcolor{talorgreen}{+14.8} & \textcolor{talorgreen}{+6.1} & \textcolor{talorgreen}{+20.3} & \textcolor{talorgreen}{+6.1} & \textcolor{talorgreen}{+14.4} & \textcolor{talorgreen}{+6.1} & \textcolor{talorgreen}{+15.1} & \textcolor{talorgreen}{+6.1} & \textcolor{talorgreen}{+15.5} & \textcolor{talorgreen}{+9.7} \\
\midrule
CoTTA & 49.8 & 63.2 & 49.1 & 65.3 & 49.5 & 75.3 & 50.0 & 66.1 & 50.1 & 55.3 & 49.4 & 62.2 & 49.6 & 64.6 & 54.9 \\
+ TALOR & 55.7 & 79.4 & 55.4 & 85.3 & 55.6 & 92.3 & 55.8 & 91.7 & 55.7 & 69.2 & 55.4 & 78.6 & 55.6 & 82.7 & 65.6 \\
Improve & \textcolor{talorgreen}{+5.9} & \textcolor{talorgreen}{+16.2} & \textcolor{talorgreen}{+6.3} & \textcolor{talorgreen}{+19.9} & \textcolor{talorgreen}{+6.1} & \textcolor{talorgreen}{+17.0} & \textcolor{talorgreen}{+5.8} & \textcolor{talorgreen}{+25.6} & \textcolor{talorgreen}{+5.6} & \textcolor{talorgreen}{+13.9} & \textcolor{talorgreen}{+6.0} & \textcolor{talorgreen}{+16.4} & \textcolor{talorgreen}{+6.0} & \textcolor{talorgreen}{+18.2} & \textcolor{talorgreen}{+10.8} \\
\midrule
EATA & 55.0 & 61.7 & 17.7 & 62.7 & 39.9 & 63.3 & 51.0 & 66.1 & 55.8 & 58.2 & 53.5 & 63.6 & 45.5 & 62.6 & 48.2 \\
+ TALOR & 55.5 & 75.3 & 17.9 & 61.0 & 40.4 & 78.3 & 51.5 & 83.1 & 56.4 & 69.0 & 54.2 & 76.4 & 46.0 & 73.9 & 53.4 \\
Improve & \textcolor{talorgreen}{+0.5} & \textcolor{talorgreen}{+13.6} & \textcolor{talorgreen}{+0.3} & \textcolor{talorred}{-1.7} & \textcolor{talorgreen}{+0.5} & \textcolor{talorgreen}{+15.0} & \textcolor{talorgreen}{+0.5} & \textcolor{talorgreen}{+17.0} & \textcolor{talorgreen}{+0.7} & \textcolor{talorgreen}{+10.9} & \textcolor{talorgreen}{+0.6} & \textcolor{talorgreen}{+12.8} & \textcolor{talorgreen}{+0.5} & \textcolor{talorgreen}{+11.3} & \textcolor{talorgreen}{+5.2} \\
\midrule
SAR & 50.1 & 55.7 & 44.4 & 46.2 & 45.6 & 48.5 & 49.2 & 52.5 & 57.9 & 62.7 & 54.1 & 65.9 & 50.2 & 55.2 & 51.1 \\
+ TALOR & 53.3 & 79.4 & 50.9 & 87.2 & 52.6 & 92.5 & 54.0 & 91.6 & 59.5 & 73.8 & 56.3 & 81.1 & 54.4 & 84.3 & 65.4 \\
Improve & \textcolor{talorgreen}{+3.2} & \textcolor{talorgreen}{+23.7} & \textcolor{talorgreen}{+6.5} & \textcolor{talorgreen}{+41.0} & \textcolor{talorgreen}{+7.1} & \textcolor{talorgreen}{+44.1} & \textcolor{talorgreen}{+4.8} & \textcolor{talorgreen}{+39.2} & \textcolor{talorgreen}{+1.5} & \textcolor{talorgreen}{+11.0} & \textcolor{talorgreen}{+2.2} & \textcolor{talorgreen}{+15.2} & \textcolor{talorgreen}{+4.2} & \textcolor{talorgreen}{+29.0} & \textcolor{talorgreen}{+14.4} \\
\midrule
OSTTA & 57.4 & 59.8 & 56.1 & 55.7 & 51.1 & 66.6 & 55.7 & 58.7 & 59.0 & 62.7 & 57.5 & 67.3 & 56.1 & 61.8 & 58.4 \\
+ TALOR & 58.7 & 80.4 & 57.3 & 89.7 & 52.4 & 92.4 & 56.8 & 92.9 & 60.2 & 74.2 & 58.7 & 82.2 & 57.4 & 85.3 & 68.1 \\
Improve & \textcolor{talorgreen}{+1.3} & \textcolor{talorgreen}{+20.5} & \textcolor{talorgreen}{+1.2} & \textcolor{talorgreen}{+33.9} & \textcolor{talorgreen}{+1.3} & \textcolor{talorgreen}{+25.8} & \textcolor{talorgreen}{+1.1} & \textcolor{talorgreen}{+34.1} & \textcolor{talorgreen}{+1.3} & \textcolor{talorgreen}{+11.5} & \textcolor{talorgreen}{+1.2} & \textcolor{talorgreen}{+14.9} & \textcolor{talorgreen}{+1.2} & \textcolor{talorgreen}{+23.5} & \textcolor{talorgreen}{+9.6} \\
\midrule
ViDA & 50.4 & 65.3 & 50.3 & 68.7 & 50.3 & 76.9 & 50.3 & 68.5 & 50.6 & 56.4 & 50.4 & 64.2 & 50.4 & 66.7 & 56.4 \\
+ TALOR & 56.5 & 80.2 & 56.3 & 86.0 & 56.4 & 93.5 & 56.4 & 91.7 & 56.5 & 71.0 & 56.5 & 80.0 & 56.4 & 83.7 & 66.6 \\
Improve & \textcolor{talorgreen}{+6.1} & \textcolor{talorgreen}{+14.9} & \textcolor{talorgreen}{+6.0} & \textcolor{talorgreen}{+17.3} & \textcolor{talorgreen}{+6.1} & \textcolor{talorgreen}{+16.6} & \textcolor{talorgreen}{+6.1} & \textcolor{talorgreen}{+23.1} & \textcolor{talorgreen}{+5.9} & \textcolor{talorgreen}{+14.6} & \textcolor{talorgreen}{+6.1} & \textcolor{talorgreen}{+15.8} & \textcolor{talorgreen}{+6.0} & \textcolor{talorgreen}{+17.1} & \textcolor{talorgreen}{+10.1} \\
\midrule
UniEnt & 59.2 & 74.0 & 59.4 & 83.4 & 59.2 & 91.1 & 44.2 & 78.0 & 59.5 & 62.1 & 58.7 & 71.8 & 56.7 & 76.7 & 63.9 \\
+ TALOR & 60.1 & 81.3 & 60.2 & 90.6 & 59.9 & 96.0 & 44.7 & 81.5 & 60.4 & 72.7 & 59.6 & 82.6 & 57.5 & 84.1 & 67.3 \\
Improve & \textcolor{talorgreen}{+0.9} & \textcolor{talorgreen}{+7.2} & \textcolor{talorgreen}{+0.8} & \textcolor{talorgreen}{+7.2} & \textcolor{talorgreen}{+0.7} & \textcolor{talorgreen}{+4.9} & \textcolor{talorgreen}{+0.5} & \textcolor{talorgreen}{+3.5} & \textcolor{talorgreen}{+0.9} & \textcolor{talorgreen}{+10.6} & \textcolor{talorgreen}{+0.9} & \textcolor{talorgreen}{+10.8} & \textcolor{talorgreen}{+0.8} & \textcolor{talorgreen}{+7.4} & \textcolor{talorgreen}{+3.4} \\
\midrule
STAMP & 51.5 & 71.0 & 51.9 & 78.4 & 51.9 & 85.7 & 52.0 & 78.9 & 52.0 & 60.9 & 51.9 & 70.2 & 51.9 & 74.2 & 60.2 \\
+ TALOR & 54.0 & 82.3 & 54.2 & 90.3 & 54.3 & 96.1 & 54.5 & 94.1 & 54.5 & 71.7 & 54.4 & 81.8 & 54.3 & 86.0 & 65.9 \\
Improve & \textcolor{talorgreen}{+2.5} & \textcolor{talorgreen}{+11.2} & \textcolor{talorgreen}{+2.3} & \textcolor{talorgreen}{+11.9} & \textcolor{talorgreen}{+2.4} & \textcolor{talorgreen}{+10.4} & \textcolor{talorgreen}{+2.4} & \textcolor{talorgreen}{+15.2} & \textcolor{talorgreen}{+2.5} & \textcolor{talorgreen}{+10.8} & \textcolor{talorgreen}{+2.4} & \textcolor{talorgreen}{+11.6} & \textcolor{talorgreen}{+2.4} & \textcolor{talorgreen}{+11.9} & \textcolor{talorgreen}{+5.7} \\
\midrule
COME & 58.4 & 76.0 & 58.0 & 78.7 & 51.4 & 81.5 & 59.0 & 82.0 & 59.0 & 58.4 & 56.3 & 71.1 & 57.0 & 74.6 & 63.7 \\
\rowcolor{blue!20}
+ TALOR & 59.2 & 84.4 & 58.6 & 90.9 & 52.1 & 91.5 & 59.6 & 94.2 & 59.6 & 75.3 & 57.1 & 82.5 & 57.7 & 86.5 & 68.4 \\
Improve & \textcolor{talorgreen}{+0.8} & \textcolor{talorgreen}{+8.4} & \textcolor{talorgreen}{+0.6} & \textcolor{talorgreen}{+12.3} & \textcolor{talorgreen}{+0.7} & \textcolor{talorgreen}{+10.0} & \textcolor{talorgreen}{+0.6} & \textcolor{talorgreen}{+12.1} & \textcolor{talorgreen}{+0.6} & \textcolor{talorgreen}{+16.9} & \textcolor{talorgreen}{+0.8} & \textcolor{talorgreen}{+11.4} & \textcolor{talorgreen}{+0.7} & \textcolor{talorgreen}{+11.9} & \textcolor{talorgreen}{+4.7} \\
\midrule
DPCore & 56.5 & 75.5 & 53.8 & 78.3 & 52.1 & 84.0 & 52.4 & 81.7 & 58.3 & 65.2 & 55.6 & 74.9 & 54.8 & 76.6 & 63.5 \\
+ TALOR & 56.7 & 82.9 & 54.2 & 88.1 & 52.5 & 92.0 & 52.7 & 92.0 & 58.5 & 76.2 & 56.0 & 82.5 & 55.1 & 85.6 & 66.6 \\
Improve & \textcolor{talorgreen}{+0.2} & \textcolor{talorgreen}{+7.4} & \textcolor{talorgreen}{+0.4} & \textcolor{talorgreen}{+9.9} & \textcolor{talorgreen}{+0.4} & \textcolor{talorgreen}{+7.9} & \textcolor{talorgreen}{+0.3} & \textcolor{talorgreen}{+10.3} & \textcolor{talorgreen}{+0.2} & \textcolor{talorgreen}{+11.0} & \textcolor{talorgreen}{+0.4} & \textcolor{talorgreen}{+7.6} & \textcolor{talorgreen}{+0.3} & \textcolor{talorgreen}{+9.0} & \textcolor{talorgreen}{+3.1} \\
\bottomrule
\end{tabular}
}
\vspace{-10pt}

\end{table*}

\vspace{-5pt}
\subsection{Experimental Setup}

\textbf{Benchmarks and Baselines.}
To evaluate performance under non-stationary open-set streams, we construct protocols combining covariate-shifted in-distribution and out-of-distribution samples.
For the \textbf{csID} component, we employ the standard ImageNet-C \citep{hendrycks2019robustness} benchmark. To assess generalization to additional domain shifts, we also use the highly nonlinear ImageNet-R \citep{hendrycks2021many}, ImageNet-Sketch \citep{wang2019learning} and the recently released LAION-C \citep{li2025laion} benchmark.
For the \textbf{csOOD} component, we follow established practices by applying the corresponding corruptions to a diverse set of OOD datasets, including Places365 \citep{zhou2017places}, Textures \citep{cimpoi2014describing}, NINCO \citep{bitterwolf2023or}, iNaturalist \citep{van2018inaturalist}, SSB-Hard \citep{vaze2022openset}, and SUN \citep{xiao2010sun}.
We compare our method against a comprehensive suite of state-of-the-art baselines, including strong continual TTA methods (CoTTA \citep{wang2022continual}, SAR \citep{niu2023towards}, EATA \citep{niu2022efficient}, ViDA \citep{liu2023vida}, DPCore \citep{zhang2024dpcore}) and specialized Open-Set TTA frameworks (STAMP \citep{yu2024stamp}, OSTTA \citep{lee2023towards}, UniEnt \citep{gao2024unified}, COME \citep{zhang2024come}).

\textbf{Evaluation Protocol.}
We employ three quantitative metrics to provide a holistic assessment of OSTTA performance.
First, we report the ID classification accuracy (\textbf{Acc}) to measure the model's domain generalization capability on known classes.
Second, we use the area under the ROC curve (\textbf{AUC}) to evaluate the effectiveness of outlier detection in a threshold-free manner.
Finally, to balance these two objectives, we report the \textbf{H-score}, defined as the harmonic mean of Acc and AUC, which serves as the primary indicator of overall performance in open-set adaptation scenarios \footnote{Acc, AUC, H-score are first computed per dataset separately and then averaged.}.

\textbf{Implementation Details.}
To ensure consistency and reproducibility, our experimental framework is built upon the codebase of DPCore\citep{zhang2024dpcore}.
The primary results utilize a ViT-B/16 backbone pretrained on ImageNet-1K from the \texttt{timm} library, additional results using ResNet-50 are detailed in App.~\ref{app:resnet50}.
% We implement TALOR as a lightweight post-hoc logits rectification plugin registered via a forward-pre-hook, which receives features to apply gradient-free correction within the baseline's standard forward pass with negligible extra inference overhead. 
% Backbone-specific subspace construction details are given in App.~\ref{app:implementation-details}.
For OOD detection, we employ the energy score \citep{liu2020energy} as the scoring function to compute AUC.
We fix the batch size to 64 and the global csID:csOOD ratio to 1:1, the common OSTTA protocol~\citep{gao2024unified,lee2023towards,zhang2024come,li2023robustness}; samples are mixed at the dataset level, so the per-batch composition fluctuates naturally (Experiments with different ratios are provided in App.~\ref{app:oodratios}).

To simulate realistic deployment conditions and ensure a fair blind evaluation, we follow the protocol in \citep{yu2024stamp}: hyperparameters for all competing methods are tuned exclusively on the first dataset-domain combination. These settings are subsequently frozen and applied directly to all unseen datasets to rigorously evaluate generalization robustness.
To ensure consistency and avoid dataset-specific tuning, we fix $\tau{=}0.8$, $\beta{=}0.9$, and $\lambda{=}0.01$ for all main results before any exposure to target data to fairly quantify the incremental performance gains attributable to the TALOR mechanism. Regarding specific baselines, we select the strongest configuration for UniEnt by integrating EATA with the UniEnt+ variant. For COME, we adopt its optimal version reported for both continual and open-world adaptation scenarios.

% Requires: \usepackage{booktabs,xcolor}
% Softer colors (deep & less neon)
\definecolor{talorgreen}{HTML}{1B5E20}
\definecolor{talorred}{HTML}{8B1A1A}
\begin{table}[t]
\caption{Results ($\%$) on LAION-C benchmark. Dataset and mean columns report H-score.}
\label{tab2:imgl}
\centering
\setlength{\tabcolsep}{3.5pt}
\renewcommand{\arraystretch}{0.2}
{\small
\begin{tabular}{@{}cccccccc@{}}
\toprule
Method & Places365-C & Textures-C & iNaturalist-C & SUN-C & SSB-Hard-C & NINCO-C & Mean H-s \\
\midrule
Source & 55.4 (\textcolor{talorgreen}{+10.5}) & 57.5 (\textcolor{talorgreen}{+10.7}) & 60.5 (\textcolor{talorgreen}{+9.7}) & 58.4 (\textcolor{talorgreen}{+11.1}) & 54.8 (\textcolor{talorgreen}{+6.9}) & 56.8 (\textcolor{talorgreen}{+8.7}) & 57.2 (\textcolor{talorgreen}{+9.6}) \\
CoTTA & 54.4 (\textcolor{talorgreen}{+11.5}) & 55.8 (\textcolor{talorgreen}{+12.4}) & 58.5 (\textcolor{talorgreen}{+11.7}) & 57.0 (\textcolor{talorgreen}{+12.6}) & 54.6 (\textcolor{talorgreen}{+7.1}) & 56.0 (\textcolor{talorgreen}{+9.6}) & 56.1 (\textcolor{talorgreen}{+10.8}) \\
EATA & 62.0 (\textcolor{talorgreen}{+7.2}) & 65.3 (\textcolor{talorgreen}{+7.8}) & 66.2 (\textcolor{talorgreen}{+7.0}) & 65.3 (\textcolor{talorgreen}{+8.0}) & 62.1 (\textcolor{talorgreen}{+6.0}) & 62.2 (\textcolor{talorgreen}{+8.2}) & 63.8 (\textcolor{talorgreen}{+7.4}) \\
SAR & 38.2 (\textcolor{talorgreen}{+18.4}) & 39.7 (\textcolor{talorgreen}{+15.2}) & 38.5 (\textcolor{talorgreen}{+21.2}) & 40.2 (\textcolor{talorgreen}{+18.6}) & 50.3 (\textcolor{talorgreen}{+9.6}) & 45.1 (\textcolor{talorgreen}{+14.4}) & 42.0 (\textcolor{talorgreen}{+16.2}) \\
OSTTA & 57.1 (\textcolor{talorgreen}{+11.3}) & 54.1 (\textcolor{talorgreen}{+15.3}) & 55.5 (\textcolor{talorgreen}{+14.6}) & 56.6 (\textcolor{talorgreen}{+13.8}) & 60.5 (\textcolor{talorgreen}{+5.4}) & 60.6 (\textcolor{talorgreen}{+8.3}) & 57.4 (\textcolor{talorgreen}{+11.5}) \\
ViDA & 54.4 (\textcolor{talorgreen}{+11.6}) & 56.6 (\textcolor{talorgreen}{+11.7}) & 58.0 (\textcolor{talorgreen}{+11.8}) & 56.9 (\textcolor{talorgreen}{+12.5}) & 54.6 (\textcolor{talorgreen}{+7.0}) & 56.1 (\textcolor{talorgreen}{+9.4}) & 56.1 (\textcolor{talorgreen}{+10.7}) \\
UniEnt & 65.6 (\textcolor{talorgreen}{+5.4}) & 70.3 (\textcolor{talorgreen}{+3.9}) & 72.4 (\textcolor{talorgreen}{+3.2}) & 70.8 (\textcolor{talorgreen}{+3.8}) & 62.8 (\textcolor{talorgreen}{+5.1}) & 64.8 (\textcolor{talorgreen}{+6.6}) & 67.8 (\textcolor{talorgreen}{+4.7}) \\
STAMP & 56.6 (\textcolor{talorgreen}{+9.6}) & 59.0 (\textcolor{talorgreen}{+9.3}) & 61.7 (\textcolor{talorgreen}{+7.9}) & 60.2 (\textcolor{talorgreen}{+9.0}) & 55.6 (\textcolor{talorgreen}{+6.1}) & 57.5 (\textcolor{talorgreen}{+8.1}) & 58.4 (\textcolor{talorgreen}{+8.3}) \\
COME & 63.3 (\textcolor{talorgreen}{+4.9}) & 67.0 (\textcolor{talorgreen}{+4.9}) & 70.8 (\textcolor{talorgreen}{+4.0}) & 68.1 (\textcolor{talorgreen}{+5.0}) & 60.2 (\textcolor{talorgreen}{+7.8}) & 64.9 (\textcolor{talorgreen}{+6.5}) & 65.7 (\textcolor{talorgreen}{+5.5}) \\
DPCore & 69.7 (\textcolor{talorgreen}{+2.9}) & 69.2 (\textcolor{talorgreen}{+3.1}) & 62.9 (\textcolor{talorgreen}{+2.0}) & 72.6 (\textcolor{talorgreen}{+2.7}) & 65.6 (\textcolor{talorgreen}{+5.4}) & 68.8 (\textcolor{talorgreen}{+4.0}) & 68.2 (\textcolor{talorgreen}{+3.4}) \\
\bottomrule
\end{tabular}
}
\vspace{-15pt}

\end{table}

\vspace{-5pt}
\subsection{Main Results}
\vspace{-5pt}
\textbf{ImageNet-C Benchmark.}
Tab.~\ref{tab1:imgc} presents the comprehensive performance comparison on the ImageNet-C benchmark with 15 corruptions at highest severity level.
TALOR delivers consistent gains across all baselines, improving the mean H-score by a clear margin.
These consistent gains suggest that head-induced principal correction is complementary to current TTA objectives and provides a useful post-correction signal beyond their standard adaptation losses.
Even when deployed as a standalone plug-in on the source backbone, TALOR achieves a mean H-score of $\bm{66.1\%}$, surpassing the strongest original baseline UniEnt ($\bm{63.9\%}$) while maintaining near-inference speed.
Notably, TALOR mitigates the severe degradation of SAR under continual domain and semantic shifts, boosting its mean H-score from $\bm{51.1\%}$ to $\bm{65.4\%}$, a substantial improvement of $\bm{+14.4\%}$.
Furthermore, augmenting COME with TALOR attains the state-of-the-art mean H-score of $\bm{68.4\%}$ with a $\bm{+4.7\%}$ gain, establishing it as the most competitive solution under this protocol.
The gains persist under imbalance: standalone TALOR improves the Source H-score by $\bm{+8.2\%}$ at csID:csOOD$=$8:2 and $\bm{+10.2\%}$ at 3:7 (App.~\ref{app:oodratios}).
These gains come with negligible additional time and GPU memory overhead, indicating that TALOR improves robustness with minimal cost.
The efficiency comparison is provided in App.~\ref{app:time_and_mem}.

\textbf{LAION-C Benchmark.}
As shown in Tab.~\ref{tab2:imgl}, TALOR also generalizes to the more recent LAION-C benchmark with six challenging domain shifts, most of which are \textit{nonlinear shifts}.
Even without parameter updates, the standalone plug-in yields a $\bm{+9.6\%}$ gain on the frozen source model.
Substantial improvements are observed across representative adaptation methods, boosting CoTTA by $\bm{+10.8\%}$, EATA by $\bm{+7.4\%}$, OSTTA by $\bm{+11.5\%}$, and ViDA by $\bm{+10.7\%}$ in mean H-score.
Furthermore, it advances the strongest baselines, upgrading UniEnt to about $\bm{72.5\%}$ and COME to $\bm{71.2\%}$, indicating that TALOR provides a complementary correction across diverse shift types.

% Required packages in preamble:
% \usepackage{booktabs,multirow,xcolor,graphicx}

\definecolor{talorgreen}{HTML}{1B5E20}
\definecolor{talorred}{HTML}{8B1A1A}

\begin{table}[t]
\centering
\small
\caption{Closed-set accuracy (\%) results under five evaluation settings. Each entry reports the original accuracy with the TALOR improvement in parentheses.}
\label{tab:closed_combined}
\setlength{\tabcolsep}{5pt}
\renewcommand{\arraystretch}{0.2}
{\small
\begin{tabular}{cccccc}
\toprule
\textbf{Method} 
& \textbf{LAION-C} 
& \textbf{ImgNet-R} 
& \textbf{ImgNet-Sketch} 
& \textbf{Structured (CSC)} 
& \textbf{Dynamic (CDC)} \\
\midrule
Source 
& 51.4 (\textcolor{talorgreen}{+6.1}) 
& 42.9 (\textcolor{talorgreen}{+2.8}) 
& 45.8 (\textcolor{talorgreen}{+3.2}) 
& 49.8 (\textcolor{talorgreen}{+6.1}) 
& 49.8 (\textcolor{talorgreen}{+5.5}) \\

CoTTA 
& 51.4 (\textcolor{talorgreen}{+6.2}) 
& 43.2 (\textcolor{talorgreen}{+2.7}) 
& 46.1 (\textcolor{talorgreen}{+3.2}) 
& 50.0 (\textcolor{talorgreen}{+6.0}) 
& 49.4 (\textcolor{talorgreen}{+6.6}) \\

EATA 
& 61.1 (\textcolor{talorgreen}{+1.6}) 
& 50.9 (\textcolor{talorgreen}{+0.2}) 
& 52.6 (\textcolor{talorgreen}{+1.1}) 
& 59.5 (\textcolor{talorgreen}{+0.8}) 
& 58.5 (\textcolor{talorgreen}{+0.8}) \\

SAR 
& 49.1 (\textcolor{talorgreen}{+8.0}) 
& 45.9 (\textcolor{talorgreen}{+3.1}) 
& 50.0 (\textcolor{talorgreen}{+2.0}) 
& 58.3 (\textcolor{talorgreen}{+1.5}) 
& 48.8 (\textcolor{talorgreen}{+4.2}) \\

OSTTA 
& 56.4 (\textcolor{talorgreen}{+3.4}) 
& 46.2 (\textcolor{talorgreen}{+2.5}) 
& 49.1 (\textcolor{talorgreen}{+1.9}) 
& 58.4 (\textcolor{talorgreen}{+1.5}) 
& 58.2 (\textcolor{talorgreen}{+1.3}) \\

ViDA 
& 51.1 (\textcolor{talorgreen}{+6.2}) 
& 43.1 (\textcolor{talorgreen}{+2.7}) 
& 46.2 (\textcolor{talorgreen}{+3.1}) 
& 50.5 (\textcolor{talorgreen}{+5.7}) 
& 50.8 (\textcolor{talorgreen}{+5.1}) \\

UniEnt 
& 60.1 (\textcolor{talorgreen}{+1.8}) 
& 47.6 (\textcolor{talorgreen}{+1.1}) 
& 49.5 (\textcolor{talorgreen}{+1.8}) 
& 59.4 (\textcolor{talorgreen}{+1.3}) 
& 58.6 (\textcolor{talorgreen}{+1.0}) \\

STAMP 
& 51.4 (\textcolor{talorgreen}{+5.3}) 
& 43.5 (\textcolor{talorgreen}{+2.7}) 
& 46.3 (\textcolor{talorgreen}{+3.0}) 
& 51.5 (\textcolor{talorgreen}{+4.0}) 
& 51.8 (\textcolor{talorgreen}{+3.1}) \\

COME 
& 59.6 (\textcolor{talorgreen}{+1.5}) 
& 49.6 (\textcolor{talorgreen}{+0.2}) 
& 51.6 (\textcolor{talorgreen}{+1.1}) 
& 55.4 (\textcolor{talorgreen}{+1.1}) 
& 55.1 (\textcolor{talorgreen}{+0.9}) \\

DPCore 
& 65.3 (\textcolor{talorgreen}{+0.1}) 
& 47.6 (\textcolor{talorgreen}{+0.1}) 
& 48.3 (\textcolor{talorgreen}{+0.8}) 
& 60.5 (\textcolor{talorgreen}{+0.4}) 
& 57.4 (\textcolor{talorgreen}{+0.6}) \\
\bottomrule
\end{tabular}

}
\vspace{-15pt}
\end{table}

\begin{figure*}[t]
    \centering
    \begin{subfigure}[b]{0.22\textwidth}
        \centering
        \includegraphics[width=\linewidth]{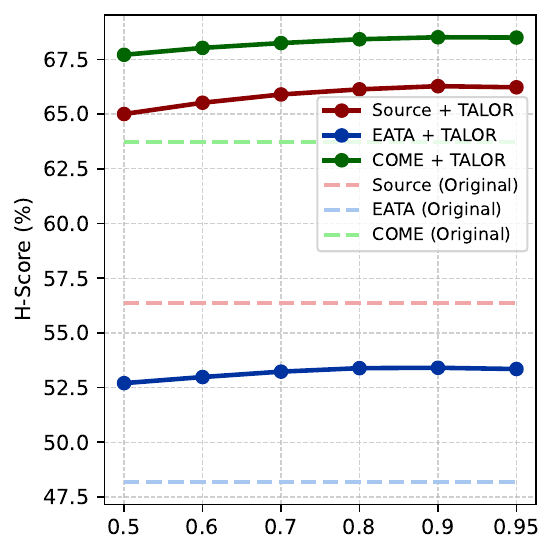}
        \caption{Energy threshold $\tau$}
        \label{fig:energy_thr}
    \end{subfigure}
    \hfill
    \begin{subfigure}[b]{0.22\textwidth}
        \centering
        \includegraphics[width=\linewidth]{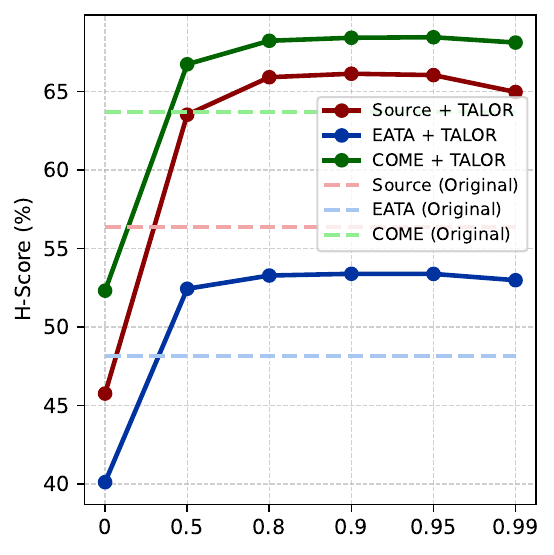}
        \caption{EMA momentum $\beta$}
        \label{fig:ema}
    \end{subfigure}
    \hfill
    \begin{subfigure}[b]{0.22\textwidth}
        \centering
        \includegraphics[width=\linewidth]{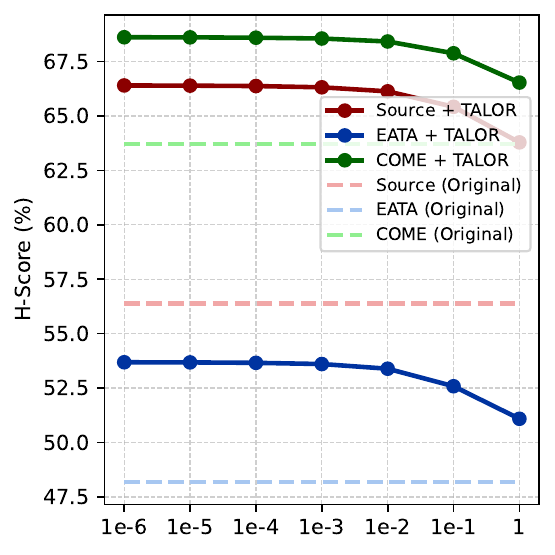}
        \caption{Ridge coefficient $\lambda$}
        \label{fig:lambda}
    \end{subfigure}
    \hfill
    \begin{subfigure}[b]{0.26\textwidth}
        \centering
        \includegraphics[width=\linewidth]{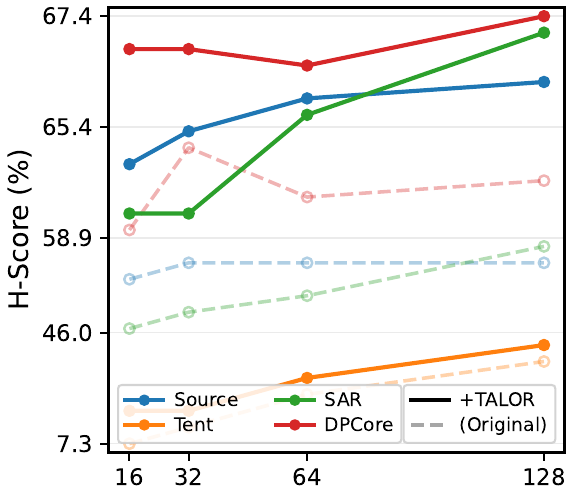}
        \caption{Batch Size}
        \label{fig:batch_size}
    \end{subfigure}
    
    \caption{Ablation studies on ImageNet-C benchmark averaged on six csOOD datasets. }
    \label{fig:Ablation}
    \vspace{-10pt}
\end{figure*}

\textbf{Robustness across Closed-Set and Continual Protocols.}
To test whether the correction mechanism transfers beyond open-set streams, we evaluate TALOR in the closed-set setting where only csID samples are present (Tab.~\ref{tab:closed_combined}).
Since the test stream contains no unknowns, both soft-routed branches of TALOR are estimated from csID samples only.
Despite closed-set TTA being a thoroughly explored field, TALOR consistently yields stable gains across all baselines, suggesting that head-induced principal correction can remain useful even when the label space is fixed. The first three columns highlight TALOR's broad generalization across diverse domain shifts. On benchmarks ranging from the synthetic corruptions of LAION-C to the natural shifts of ImageNet-R and Sketch, TALOR delivers uniform improvements.
Furthermore, we adopt the CSC and CDC protocols from \citep{zhang2024dpcore} to assess stability under complex temporal dynamics. As shown in last two columns, TALOR proves effective not only in the structured transitions of CSC but also in the chaotic, irregular recurrence of CDC.
Collectively, these results indicate that TALOR can be effective in both open- and closed-set scenarios and remains stable across static and dynamic data streams.

\vspace{-5pt}
\subsection{Ablation Studies}
\label{sec:ablation}
\vspace{-5pt}

\textbf{Robustness to hyperparameters and batch size.}
We study the sensitivity of TALOR on the ImageNet-C benchmark. 
\textbf{SVD energy threshold.} $\tau$ controls the split point $k$ between principal and tail subspaces; TALOR exhibits strong stability across different values, with H-score remaining near-plateau over a wide range (Fig.~\ref{fig:energy_thr}). 
\textbf{EMA momentum.} $\beta$ smooths the online updates of branch statistics and drift estimates; performance is similarly stable for moderate-to-large $\beta$, while removing EMA (very small $\beta$) is more prone to noisy batch-wise estimation (Fig.~\ref{fig:ema}). 
\textbf{Ridge coefficient.} $\lambda$ regularizes the weighted ridge regressor for the tail-to-principal slope term; TALOR remains robust for small-to-moderate $\lambda$ and only degrades slightly when over-regularized (Fig.~\ref{fig:lambda}). 
\textbf{Batch size.} While TTA methods typically benefit from the more stable statistics provided by larger batch sizes (Fig.~\ref{fig:batch_size}), TALOR demonstrates remarkable robustness to batch size variations.
Notably, even under small batches ($B{=}16$), TALOR delivers significant improvements over the baselines, achieving high performance comparable to or exceeding that of baselines with larger batches.
% This stability under small batch sizes is crucial for practical streaming deployment where data often arrives in small chunks.

\begin{table*}[t]
\centering
\caption{Module ablation study.}
\label{tab:moduleabla}
\setlength{\tabcolsep}{2.6pt}
\renewcommand{\arraystretch}{0.6}
\small
\resizebox{\textwidth}{!}{
\begin{tabular}{c ccc ccc ccc ccc ccc ccc ccc}
\toprule
\multirow{2}{*}{\textbf{Method}} 
& \multicolumn{3}{c}{\textbf{Source}} 
& \multicolumn{3}{c}{\textbf{EATA}} 
& \multicolumn{3}{c}{\textbf{DPCore}} 
& \multicolumn{3}{c}{\textbf{UniEnt}} 
& \multicolumn{3}{c}{\textbf{COME}} 
& \multicolumn{3}{c}{\textbf{OSTTA}} 
& \multicolumn{3}{c}{\textbf{SAR}} \\
\cmidrule(lr){2-4}
\cmidrule(lr){5-7}
\cmidrule(lr){8-10}
\cmidrule(lr){11-13}
\cmidrule(lr){14-16}
\cmidrule(lr){17-19}
\cmidrule(lr){20-22}
& \textbf{Acc} & \textbf{AUC} & \textbf{H-s}
& \textbf{Acc} & \textbf{AUC} & \textbf{H-s}
& \textbf{Acc} & \textbf{AUC} & \textbf{H-s}
& \textbf{Acc} & \textbf{AUC} & \textbf{H-s}
& \textbf{Acc} & \textbf{AUC} & \textbf{H-s}
& \textbf{Acc} & \textbf{AUC} & \textbf{H-s}
& \textbf{Acc} & \textbf{AUC} & \textbf{H-s} \\
\midrule
Origin 
& 49.8 & 68.0 & 56.4
& 45.5 & 62.6 & 48.2
& 54.8 & 76.6 & 63.5
& 56.7 & 76.7 & 63.9
& 57.0 & 74.6 & 63.7
& 56.1 & 61.8 & 58.4
& 50.2 & 55.2 & 51.1 \\

Mean-only
& 54.8 & 75.4 & 62.3
& 45.8 & 65.3 & 50.2
& 54.8 & 77.3 & 63.7
& 57.3 & 78.2 & 64.9
& 57.3 & 79.0 & 65.5
& 57.0 & 73.7 & 63.7
& 53.7 & 76.9 & 62.3 \\

Slope-only
& 51.0 & 80.8 & 61.4
& 45.6 & 73.1 & 52.4
& \textbf{55.1} & 85.3 & 66.5
& 56.8 & 83.6 & 66.6
& 57.3 & 85.2 & 67.8
& 56.5 & 82.1 & 66.5
& 49.7 & 73.6 & 58.4 \\

\rowcolor{blue!10}
Full 
& \textbf{55.9} & 83.5 & 66.1
& 46.0 & 73.9 & 53.4
& \textbf{55.1} & 85.6 & 66.6
& \textbf{57.5} & 84.1 & 67.3
& \textbf{57.7} & 86.5 & 68.4
& \textbf{57.4} & 85.3 & 68.1
& \textbf{54.4} & 84.3 & 65.4 \\

Single Branch 
& 55.6 & 82.4 & 65.4
& 45.8 & 73.4 & 53.0
& 54.9 & 84.9 & 66.1
& 57.3 & 83.3 & 66.8
& 57.3 & 86.1 & 68.0
& 57.2 & 84.7 & 67.6
& 54.3 & 83.5 & 64.9 \\

Oracle 
& 55.7 & \textbf{86.3} & \textbf{66.9}
& \textbf{47.1} & \textbf{76.0} & \textbf{55.0}
& 55.0 & \textbf{88.4} & \textbf{67.4}
& \textbf{57.5} & \textbf{86.9} & \textbf{68.2}
& 57.6 & \textbf{88.6} & \textbf{69.0}
& 57.2 & \textbf{88.4} & \textbf{69.0}
& 54.1 & \textbf{86.6} & \textbf{65.9} \\
\bottomrule
\end{tabular}
}
\vspace{-10pt}

\end{table*}

\textbf{Ablating affine terms and soft regime routing.}
Tab.~\ref{tab:moduleabla} ablates two design axes of TALOR: the affine form of the principal drift estimator and the use of soft ID/OOD regime routing.
(1) \textbf{Mean-only} uses only the branch principal mean, i.e., $\hat\Delta_i^{(q)}=\mu_{a,t}^{(q)}$, while \textbf{Slope-only} uses only the tail-to-principal slope, i.e., $\hat\Delta_i^{(q)}=R_t^{(q)\top}(u_i-\mu_{u,t}^{(q)})$.
The two terms are complementary: the mean term mainly improves ID accuracy by recentering the target principal coordinates, whereas the slope term mainly improves AUC by correcting sample-level tail-to-principal leakage, which is larger and better aligned with the paired drift for csOOD than for csID samples (App.~\ref{app:component-alignment}).
The \textbf{Full} model combines both terms and therefore gives a synergistic Acc--AUC gain across baselines.
(2) The routing variants isolate the value of regime separation.
\textbf{Single Branch} uses one affine principal rectifier for the whole mixed stream.
It remains competitive, consistent with the substantial shared leading-energy slope component in Fig.~\ref{fig:leakage-diagnostics-wrap}(b), while also providing a degenerate yet effective form of TALOR when extremely severe shifts make routing unreliable.
Nevertheless, the H-score and especially AUC generally follow the hierarchy \texttt{Single Branch} $<$ \texttt{Full TALOR} $<$ \texttt{Oracle} across baselines.
This ordering is consistent with soft ID/OOD routing refining residual regime-specific structure: the unsupervised GMM moves the estimator beyond a mixed regressor, while GT routing gives the upper bound when that residual structure is separated perfectly.

\vspace{-5pt}
\section{Conclusion}
\vspace{-5pt}
We presented TALOR, a lightweight plug-in for Open-Set TTA that uses the frozen source linear head as a geometric anchor for feature rectification.
TALOR innovatively estimates a removable affine principal drift through a principal mean and a tail-to-principal slope map, and subtracts this drift estimate, reconstructs the tailored feature and feeds it to the source head. Extensive experiments show that TALOR is effective as a standalone rectifier and as a stabilizing post-correction module across open-set and closed-set continual protocols.

\vspace{-5pt}
\section*{Acknowledgments and Disclosures}
\vspace{-5pt}
\textbf{Funding.} This work was supported in part by Guangdong S\&T Program (2024B0101050004), NSFC (W2611096, Overseas), ICFCRT (W2441020), Shenzhen Science and Technology Program (KJZD20240903100022028, KQTD20210811090044003), and Guangdong Provincial Key Laboratory of Visual Media and Multidimensional Intelligence (2026B1212010012).

\textbf{Competing Interests.} The authors declare no competing interests.

\bibliographystyle{plainnat}
\bibliography{nips2026}

%%%%%%%%%%%%%%%%%%%%%%%%%%%%%%%%%%%%%%%%%%%%%%%%%%%%%%%%%%%%

\newpage
\appendix

\vspace{-5pt}
\section{Related Work}
\vspace{-5pt}
\label{app:related_work}
\textbf{Continual Test-Time Adaptation.} Continual Test-Time Adaptation (CTTA) aims to cope with error accumulation and catastrophic forgetting under non-stationary test distributions \citep{wang2022continual,wang2023multiple,cao2023multi,lee2024becotta,chen2024each,hong2023mecta,zhang2024dpcore,ni2025maintaining,parkrethinking,liu2025test,Maharana2024PALMPA, maharana2026continual}. A common strategy is to stabilize online updates via regularization: EATA \citep{niu2022efficient} and SAR \citep{niu2023towards} constrain overly aggressive parameter changes, while FIND \citep{jiang2025feature} decouples features to normalize statistics robustly under dynamic shifts. Parameter-efficient methods further isolate domain-specific knowledge to reduce interference: VDP \citep{gan2023decorate} uses visual prompts to capture distribution shifts, ViDA \citep{liu2023vida} allocates low-/high-rank adapters homeostatically, and PAID \citep{wang2025paid} restricts updates to orthogonal transformations to preserve pre-trained geometry. 

\textbf{Open-Set Test-Time Adaptation.} 
OSTTA addresses covariate and semantic shifts by adapting to known classes while rejecting unknowns \citep{Yang_2026_CVPR, Zhao2026MitigatingTI, zou2026good}. 
STAMP \citep{yu2024stamp} uses class-balanced memory and self-weighted entropy to stabilize updates and outlier detection. 
OSTTA \citep{lee2023towards} mitigates open-set noise by filtering losses via confidence gaps, while UniEnt \citep{gao2024unified} employs distribution-aware filtering to minimize entropy for knowns and maximize it for unknowns. 
To curb overconfidence, COME \citep{zhang2024come} models predictive uncertainty for conservative entropy minimization. 
OWTTT \citep{li2023robustness} proposes open-world self-training with dynamic prototype expansion, progressively absorbing reliable novel clusters while using prototype-based rejection to separate unknowns. 
For foundation models, ZS-TTA \citep{cao2025noisy} couples zero-shot VLM classification with an online noise detector trained on frozen VLM features to gate entropy-based adaptation, while ROSITA \citep{sreenivas2024efficient} enables single-image OSTTA through contrastive separation and dynamic feature banks. A unified standard is currently missing regarding whether OSTTA should simultaneously accommodate continual adaptation without model resetting. 
While OSTTA and UniEnt adopt the non-resetting protocol, methods such as STAMP, COME, and OWTTT employ resetting between distributions. This work follows the more challenging non-stationary setting established by the former.

\textbf{Subspace-Aware Adaptation.}
A recent line of work recognizes that different subspaces of the feature representation contribute unequally to model outputs, motivating selective adaptation strategies.
SSA~\citep{adachi2025testtime} identifies the output-significant subspace via PCA of the source-feature covariance and aligns source and target marginal distributions only within this subspace, excluding complementary dimensions from adaptation. SSA focuses on closed-set classification and regression tasks and requires source feature statistics (mean, PCA basis, and variances) at test time; it does not address open-set scenarios or provide regime-conditioned mechanisms for unknown rejection.
Independently, SoMA~\citep{Yun_2025_CVPR} analyzes the weight matrices of vision foundation models via SVD and finds that singular components with large singular values extract general features spanning many classes, while minor singular components mainly capture fine-grained, texture-biased, and context-specific features.
Although SoMA studies backbone weights rather than a classifier head, this functional separation suggests that semantic and nuisance variation in TALOR's principal and tail subspaces are driven by largely distinct underlying factors, providing independent structural support for Assumption~\ref{asm:shared-nuisance}.

\vspace{-5pt}
\section{Implementation Details}
\vspace{-5pt}
\label{app:implementation-details}

\textbf{Backbone-dependent head geometry.}
TALOR constructs its coordinate system once from the frozen linear head before online inference.
For ImageNet-1K, the number of classes is $C=1000$, while the classifier-input feature dimension $D$ depends on the backbone.
Let $W\in\mathbb{R}^{C\times D}$ denote the head weight and $s=\min(C,D)$.
We compute the full singular value decomposition
\begin{equation}
W=U\Sigma V^{\top},
\qquad
U\in\mathbb{R}^{C\times C},
\quad
\Sigma\in\mathbb{R}^{C\times D},
\quad
V\in\mathbb{R}^{D\times D}.
\end{equation}
The decomposition provides $s$ singular values, including any zeros due to rank deficiency.
When $D>C$, the additional $D-s$ right-singular vectors span directions in the null space of $W$ and have zero singular values.
The split point $k$ is selected using the singular-value energy criterion in Eq.~\ref{eq:split-k} and clamped to $k\le s-1$, ensuring a non-empty tail subspace.
The principal and tail bases used in Sec.~\ref{sec:head-induced-coordinates} are defined as
\begin{equation}
B_p=(V_{:,1:k})^{\top}\in\mathbb{R}^{k\times D},
\qquad
B_t=(V_{:,k+1:D})^{\top}
\in\mathbb{R}^{(D-k)\times D}.
\end{equation}
Since $V$ is orthogonal, these bases form an orthogonal decomposition of the full feature space, satisfying
$B_pB_p^{\top}=I_k$,
$B_tB_t^{\top}=I_{D-k}$, and
$B_p^{\top}B_p+B_t^{\top}B_t=I_D$.
Thus, TALOR always operates in the $D$-dimensional feature space rather than the $C$-dimensional class space; the class dimension influences the decomposition only through the geometry of the frozen classifier head.

\textbf{ViT-B/16: $D<C$.}
For ViT-B/16, the classifier head has shape $W\in\mathbb{R}^{1000\times 768}$, so $D=768<C=1000$ and $s=768$.
The right-singular basis spans the entire 768-dimensional feature space.
After the energy split, TALOR obtains
\begin{equation}
B_p\in\mathbb{R}^{k\times 768},
\qquad
B_t\in\mathbb{R}^{(768-k)\times 768}.
\end{equation}
There is no additional classifier null space beyond the feature dimension; the tail consists of the lower-gain axes within the ViT feature space.

\textbf{ResNet-50: $D>C$.}
For ResNet-50, the classifier head has shape $W\in\mathbb{R}^{1000\times 2048}$, so $D=2048>C=1000$ and $s=1000$.
The full right-singular basis contains both classifier-supported axes and feature axes that lie in the null space of the head.
After selecting $k\le 999$, TALOR forms
\begin{equation}
B_p\in\mathbb{R}^{k\times 2048},
\qquad
B_t\in\mathbb{R}^{(2048-k)\times 2048}.
\end{equation}
The tail therefore contains the low-gain singular vectors $V_{:,k+1:1000}$ together with the head-null vectors $V_{:,1001:2048}$.
This is the same construction as in the ViT case: changing the backbone only changes the shape of $W$, and the SVD automatically supplies the appropriate tail geometry.
This economy is central to TALOR: a single head-induced decomposition covers both $D<C$ and $D>C$ without architecture-specific rules.

\textbf{Routing fail-safe.}
The two-component GMM of Sec.~\ref{sec:routing-estimation} is refitted in every batch on the scalar tail energies.
If the fit raises a numerical exception, TALOR assigns neutral responsibilities $w_i^{(\mathrm{ID})}=w_i^{(\mathrm{OOD})}=0.5$.
No label- or dataset-dependent confidence threshold is used, and only the affine branch states are smoothed across batches by the EMA in Eq.~\ref{eq:ema-update}.

\vspace{-5pt}
\section{Paired Corruption Diagnostics}
\vspace{-5pt}
\label{app:paired-diagnostics}

This appendix defines the paired diagnostics used to validate the rectification mechanism.
We use paired samples $(x_i, x_i^{\mathrm{corr}})$ where $x_i^{\mathrm{corr}}$ is a corruption version of $x_i$.
Projecting the normalized features onto the head-induced basis yields principal/tail coordinates
$(a_i,u_i)$ for $x_i$ and $(a_i^{\mathrm{corr}},u_i^{\mathrm{corr}})$ for $x_i^{\mathrm{corr}}$.
Define the coordinate drifts
\begin{align}
\Delta a_i &= a_i^{\mathrm{corr}}-a_i, &
\Delta u_i &= u_i^{\mathrm{corr}}-u_i.
\end{align}

\subsection{CCA coupling spectrum}
\label{app:cca-leakage-spectrum}

To test whether low-gain drift contains a linearly decodable high-gain drift component, we compute canonical correlation analysis (CCA) between $\Delta u$ and $\Delta a$.
Let $X$ and $Y$ be the centered row-stacked matrices of $\Delta u_i$ and $\Delta a_i$, respectively.
With covariance estimates
\begin{equation}
\Sigma_{uu}=\frac{1}{n-1}X^\top X,\qquad
\Sigma_{aa}=\frac{1}{n-1}Y^\top Y,\qquad
\Sigma_{ua}=\frac{1}{n-1}X^\top Y,
\end{equation}
and a small regularizer $\gamma$, we form
\begin{equation}
K=(\Sigma_{uu}+\gamma I)^{-1/2}\Sigma_{ua}(\Sigma_{aa}+\gamma I)^{-1/2}.
\end{equation}
The CCA spectrum is the sequence of singular values
\begin{equation}
\rho_j=\sigma_j(K),\qquad j=1,2,\ldots,
\end{equation}
where a large $\rho_j$ means that one linear combination of tail/null drift and one linear combination of principal drift are strongly correlated.
For the permutation baseline, we randomly shuffle the rows of $\Delta u$, recompute the spectrum, and plot the 95th percentile over permutations.
The ViT-B/16 diagnostics are shown in Fig.~\ref{fig:shift-spectrum}; the ResNet-50 counterpart in Fig.~\ref{fig:r50-shift-spectrum} additionally reflects the classifier-null axes available when $D>C$.
App.~\ref{app:drift-checks} complements this spectrum with a regression-based test of drift predictability against a permutation null.

\begin{figure*}[t]
    \centering
    \begin{subfigure}[t]{0.31\textwidth}
        \centering
        \includegraphics[width=\linewidth]{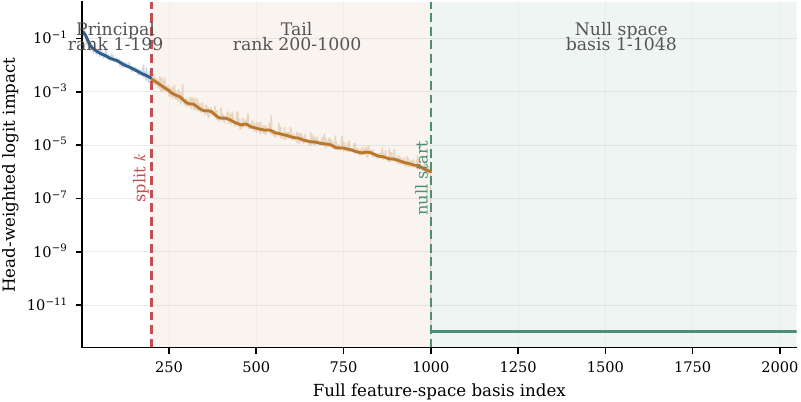}
        \caption{Logit impact.}
        \label{fig:r50-shift-spectrum-logit}
    \end{subfigure}
    \hfill
    \begin{subfigure}[t]{0.31\textwidth}
        \centering
        \includegraphics[width=\linewidth]{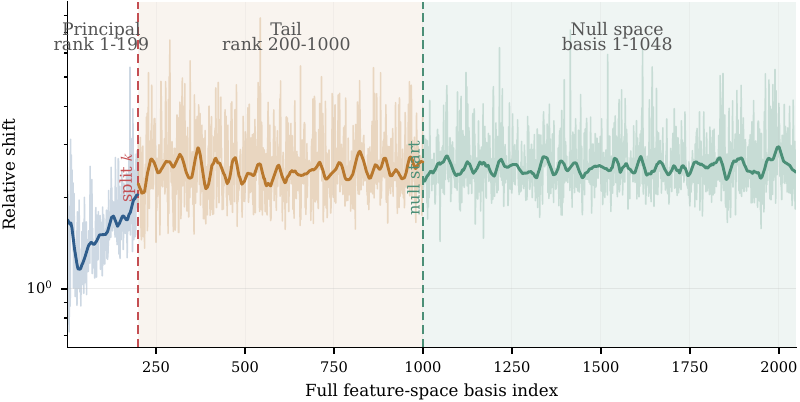}
        \caption{Relative shift.}
        \label{fig:r50-shift-spectrum-relshift}
    \end{subfigure}
    \hfill
    \begin{subfigure}[t]{0.31\textwidth}
        \centering
        \includegraphics[width=\linewidth]{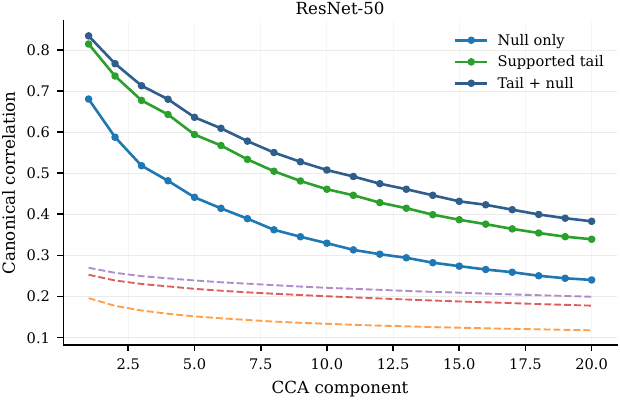}
        \caption{Tail/null-to-principal CCA.}
        \label{fig:r50-cca-leakage}
    \end{subfigure}
    \caption{ResNet-50 paired motivation diagnostics in the head-induced basis. The same sensor-target pattern holds beyond ViT-B/16; the CCA panel additionally shows that classifier-null coordinates, despite having no direct logit gain, contain linearly decodable information about principal drift.}
    \label{fig:r50-shift-spectrum}
\end{figure*}

\subsection{Oracle target ablation}
\label{app:oracle-target-ablation}

The CCA diagnostic tests whether tail/null drift can serve as a linear sensor. Furthermore, we separately test which subspace is the effective correction target by an oracle ablation.
For a target subspace $S$, the oracle replaces only that corrupted coordinate block with its clean-paired value,
\begin{equation}
\tilde c^{\mathrm{corr}}_{i,S}
=c^{\mathrm{corr}}_{i,S}
-\left(c^{\mathrm{corr}}_{i,S}-c_{i,S}\right)
=c_{i,S},
\end{equation}
while leaving all other coordinates unchanged.
We report the accuracy gain
\begin{equation}
\mathrm{Gain}(S)=\mathrm{Acc}_{\mathrm{oracle}}(S)-\mathrm{Acc}_{\mathrm{corr}},
\end{equation}
and the fraction of the clean--corrupt gap closed,
\begin{equation}
\mathrm{GapClosed}(S)=
\frac{\mathrm{Acc}_{\mathrm{oracle}}(S)-\mathrm{Acc}_{\mathrm{corr}}}
{\mathrm{Acc}_{\mathrm{clean}}-\mathrm{Acc}_{\mathrm{corr}}}.
\end{equation}
Table~\ref{tab:oracle-target-ablation} shows that principal coordinates are the dominant correction target, whereas tail/null coordinates are primarily useful as sensors.

\begin{table}[t]
\centering
\small
\caption{Oracle target correction on paired ImageNet-C diagnostics. Accuracy gains are reported in percentage points; gap closed is the percentage of the clean--corrupt accuracy gap recovered by restoring only the indicated subspace.}
\label{tab:oracle-target-ablation}
\setlength{\tabcolsep}{5.0pt}
\renewcommand{\arraystretch}{1.08}
\begin{tabular}{@{}llccccc@{}}
\toprule
Backbone & Metric & Principal & Supported tail & Null only & Tail+null & All \\
\midrule
\multirow{2}{*}{ViT-B/16}
& Acc. gain (pp) & 40.9 & 15.9 & -- & 15.9 & 42.1 \\
& Gap closed (\%) & 97.0 & 38.9 & -- & 38.9 & 100.0 \\
\midrule
\multirow{2}{*}{ResNet-50}
& Acc. gain (pp) & 56.9 & 7.2 & 0.0 & 7.2 & 58.7 \\
& Gap closed (\%) & 96.7 & 13.2 & 0.0 & 13.2 & 100.0 \\
\bottomrule
\end{tabular}
\end{table}

\subsection{Per-batch correction curves}

\textbf{Principal drift ratio (mean \& std).}
For each sample we measure how much of the paired drift lies in the principal coordinates:
\begin{equation}
\ell_i=\frac{\|\Delta a_i\|_2}{\|\Delta a_i\|_2+\|\Delta u_i\|_2}.
\end{equation}
For a mini-batch $\mathcal{B}_t$, we plot the mean and standard deviation

\vspace{-10pt}
\begin{align}
\mathrm{LR}_t &= \frac{1}{|\mathcal{B}_t|}\sum_{i\in\mathcal{B}_t}\ell_i, &
\mathrm{Std}(\mathrm{LR})_t &= \sqrt{\frac{1}{|\mathcal{B}_t|}\sum_{i\in\mathcal{B}_t}\ell_i^2-\mathrm{LR}_t^2}.
\end{align}

\textbf{Cosine alignment to the predicted correction (mean \& std).}
TALOR predicts an affine principal correction $\hat\Delta_i$ from branch statistics and tail-coordinate deviations, and rectifies the principal coordinate by
$\tilde a_i = a_i-\hat\Delta_i$ (see Eqns.~\ref{eq:affine-correction-prediction} and Eqns.~\ref{eq:rectified-inference} in the main text).
For paired diagnostics, we compare $\hat\Delta_i$ against the observed principal drift $\Delta a_i$ using cosine similarity:
\begin{equation}
c_i=\frac{\langle \hat\Delta_i,\Delta a_i\rangle}{\|\hat\Delta_i\|_2\,\|\Delta a_i\|_2+\epsilon}.
\end{equation}
We plot batch mean/std as

\begin{align}
\mathrm{Cos}_t &= \frac{1}{|\mathcal{B}_t|}\sum_{i\in\mathcal{B}_t} c_i,&
\mathrm{Std}(\mathrm{Cos})_t &= \sqrt{\frac{1}{|\mathcal{B}_t|}\sum_{i\in\mathcal{B}_t}c_i^2-\mathrm{Cos}_t^2}.
\end{align}

\textbf{Principal-drift reduction.}
We quantify how much rectification reduces the paired principal drift magnitude:

\begin{align}
d_i^{\mathrm{before}} &= \|\Delta a_i\|_2 = \|a_i^{\mathrm{corr}}-a_i\|_2,&
d_i^{\mathrm{after}}  &= \|\tilde a_i^{\mathrm{corr}}-a_i\|_2, \quad
\tilde a_i^{\mathrm{corr}} = a_i^{\mathrm{corr}}-\hat\Delta_i,
\end{align}
and define the normalized reduction
\begin{equation}
r_i=\frac{d_i^{\mathrm{before}}-d_i^{\mathrm{after}}}{d_i^{\mathrm{before}}+\epsilon},
\qquad
\mathrm{DR}_t=\frac{1}{|\mathcal{B}_t|}\sum_{i\in\mathcal{B}_t} r_i.
\end{equation}
\vspace{-3pt}
\textbf{Accuracy curves.}
Top-1 accuracy on corrupted samples and rectified logits are computed with ground-truth labels,
\vspace{-10pt}

\begin{align}
\mathrm{Acc}^{\mathrm{corr}}_t &= \frac{1}{|\mathcal{B}_t|}\sum_{i\in\mathcal{B}_t}\mathbb{I}\!\left[\arg\max_c z^{\mathrm{corr}}_{i,c}=y_i\right],&
\mathrm{Acc}^{\mathrm{rect}}_t &= \frac{1}{|\mathcal{B}_t|}\sum_{i\in\mathcal{B}_t}\mathbb{I}\!\left[\arg\max_c \tilde z_{i,c}=y_i\right],
\end{align}

where $\tilde z_i$ is obtained by recomputing logits from the rectified feature after rescaling by the original feature norm and using the frozen source head.
These accuracies are not required at test time; they serve only to validate the proposed mechanism in Fig.~\ref{fig:leakage_and_correct}.

\subsection{Component-wise alignment with the paired drift}
\label{app:component-alignment}

Fig.~\ref{fig:leakage-diagnostics-wrap}(a) evaluates the full correction.
To separate the two affine terms, we decompose $\hat\Delta_i=\hat\Delta_i^{\mathrm{mean}}+\hat\Delta_i^{\mathrm{slope}}$, with the routed mean $\hat\Delta_i^{\mathrm{mean}}=\sum_q w_i^{(q)}\mu_{a,t}^{(q)}$ and the routed slope $\hat\Delta_i^{\mathrm{slope}}=\sum_q w_i^{(q)}R_t^{(q)\top}(u_i-\mu_{u,t}^{(q)})$, along one online trajectory, so that all readouts share sample order, soft responsibilities, and ridge/EMA states.
We report the cosine alignment with the oracle drift $\Delta a_i$ and the PDR of Sec.~\ref{sec:mechanism-analysis} (Tab.~\ref{tab:component-alignment}).
The main open-set protocol replays Source+TALOR on the benchmark stream of Sec.~\ref{sec:experiments}; the paired clean views of both csID and csOOD samples are used only for analysis.
Each term aligns with the oracle drift and reduces it, and their sum performs best; the full correction has the highest cosine in all 90 cells.
Splitting the open-set cells by regime, the slope correction is larger and better aligned for csOOD than for csID samples (correction norm $0.299$ vs.\ $0.246$, cosine $0.362$ vs.\ $0.336$, PDR $0.057$ vs.\ $0.048$), whereas the mean term behaves similarly in both regimes (cosine $0.337$ vs.\ $0.325$).
This is consistent with the slope term contributing mainly to ID/OOD separation in Tab.~\ref{tab:moduleabla}.

\begin{table}[h]
\centering
\small
\caption{Component-wise alignment of TALOR's correction with the paired oracle principal drift in the main open-set benchmark stream. Mean $\pm$ std over the 90 cells (6 csOOD datasets $\times$ 15 corruptions).}
\label{tab:component-alignment}
\setlength{\tabcolsep}{5pt}
\begin{tabular}{@{}lccc@{}}
\toprule
Metric & Mean only & Slope only & Full \\
\midrule
Cosine & $0.331\pm0.135$ & $0.349\pm0.050$ & $0.483\pm0.105$ \\
PDR & $0.051\pm0.076$ & $0.053\pm0.026$ & $0.113\pm0.082$ \\
\bottomrule
\vspace{-5pt}
\end{tabular}
\vspace{-5pt}
\end{table}

\vspace{-5pt}
\section{Shared-Nuisance Condition and Drift Identification}
\vspace{-5pt}
\label{app:theory-justification}

This appendix proves Proposition~\ref{prop:nuisance-regression} and states the scope of Assumption~\ref{asm:shared-nuisance}, reports empirical checks of the condition and of the deployed online correction, and discusses the mean term.

\subsection{Proof and scope}
\label{app:proofs}

We use the notation of Sec.~\ref{sec:affine-leakage}, with $\Sigma_{\xi u}=\mathbb{E}_q[\xi_iu_i^{\circ\top}]$ and $\Sigma_{uu}=\mathbb{E}_q[u_i^{\circ}u_i^{\circ\top}]$.

\begin{proof}[Proof of Proposition~\ref{prop:nuisance-regression}]
Let $\Sigma_{au}=\mathbb{E}_q[a_i^{\circ}u_i^{\circ\top}]$.
The objective is a convex quadratic in $R$ with normal equations $(\Sigma_{uu}+\lambda I_d)R=\Sigma_{au}^{\top}$; its minimum-norm minimizer is $R_{q,\lambda}^{\ast}=(\Sigma_{uu}+\lambda I_d)^{\dagger}\Sigma_{au}^{\top}$, which is the unique minimizer for $\lambda>0$.
By Assumption~\ref{asm:shared-nuisance},
\[
\Sigma_{au}
= L_q\Sigma_{\xi u}
+ \mathbb{E}_q[s_i u_i^{\circ\top}]
+ \mathbb{E}_q[\varepsilon_i^a u_i^{\circ\top}]
= L_q\Sigma_{\xi u}.
\]
To clarify the statistical meaning of the resulting prediction,
consider the auxiliary population ridge regression of the latent
nuisance $\xi_i$ on the tail coordinates:
\[
C_{q,\lambda}^{\ast}
= \arg\min_C
\mathbb{E}_q\|\xi_i-C^\top u_i^\circ\|_2^2
+\lambda\|C\|_F^2
= (\Sigma_{uu}+\lambda I_d)^\dagger\Sigma_{\xi u}^\top,
\]
where the minimum-norm solution is taken at $\lambda=0$.
Since $\Sigma_{au}=L_q\Sigma_{\xi u}$, the two population
ridge solutions satisfy
\[
R_{q,\lambda}^{\ast}=C_{q,\lambda}^{\ast}L_q^\top.
\]
Consequently,
\[
R_{q,\lambda}^{\ast\top}u_i^\circ
=L_q C_{q,\lambda}^{\ast\top}u_i^\circ
=L_q\hat\xi_{i,\lambda},
\qquad
\hat\xi_{i,\lambda}
:=\Sigma_{\xi u}(\Sigma_{uu}+\lambda I_d)^\dagger u_i^\circ.
\]
Thus, although the fitted regression targets $a_i^\circ$,
its population prediction is exactly the linear image of
a ridge estimate of the shared nuisance $\xi_i$, and
therefore lies in $\mathrm{col}(L_q)$.
For any fixed matrix $M$, $\mathbb{E}_q[(M^{\top}u_i^{\circ})s_i^{\top}]=M^{\top}\mathbb{E}_q[u_i^{\circ}s_i^{\top}]=0$; taking $M=R_{q,\lambda}^{\ast}$ shows that the prediction is uncorrelated with $s_i$, and the residual $a_i^{\circ}-R_{q,\lambda}^{\ast\top}u_i^{\circ}=L_q(\xi_i-\hat\xi_{i,\lambda})+s_i+\varepsilon_i^{a}$ contains $s_i$ unchanged.
At $\lambda=0$, $\hat\xi_{i,0}=\Sigma_{\xi u}\Sigma_{uu}^{\dagger}u_i^{\circ}$ is the minimum-MSE linear estimate of the zero-mean factor $\xi_i$ from $u_i^{\circ}$, and by linearity $L_q\hat\xi_{i,0}$ is that of $L_q\xi_i$.
Finally, let $\varepsilon_i^{u}=0$, $\Sigma_{\xi}=\mathrm{Cov}_q(\xi_i)\succ0$, and let $T_q$ have full column rank.
Then $M_q=T_q\Sigma_{\xi}^{1/2}$ also has full column rank, and $u_i^{\circ}=M_q\Sigma_{\xi}^{-1/2}\xi_i$, $\Sigma_{uu}=M_qM_q^{\top}$, and $\Sigma_{\xi u}=\Sigma_{\xi}^{1/2}M_q^{\top}$.
Using $M_q^{\top}(M_qM_q^{\top})^{\dagger}=M_q^{+}$ and $M_q^{+}M_q=I$, we obtain $\hat\xi_{i,0}=\Sigma_{\xi}^{1/2}M_q^{+}M_q\Sigma_{\xi}^{-1/2}\xi_i=\xi_i$.
\end{proof}

\textbf{Scope.}
Proposition~\ref{prop:nuisance-regression} characterizes the population target of the slope term; TALOR estimates it from finite, soft-weighted batches, where ridge regularization and EMA smoothing stabilize the estimate.
Assumption~\ref{asm:shared-nuisance} idealizes the representation and is not implied by the SVD of the head, whose singular values only establish which directions are amplified.
When part of the principal semantic variation is linearly correlated with the tail cue, i.e., $\mathbb{E}_q[u_i^{\circ}s_i^{\top}]\neq0$, the slope term also removes that tail-predictable part.
To assess departures from the idealized assumption, we examine the relative strength of the two cross-covariance terms:
\begin{equation}
\eta_C^{(q)}
=
\frac{
\big\|\mathbb{E}_q[(s_i+\varepsilon_i^{a})u_i^{\circ\top}]\big\|_2
}{
\big\|\mathbb{E}_q[(L_q\xi_i)u_i^{\circ\top}]\big\|_2
}.
\label{eq:dominance}
\end{equation}
Smaller values indicate stronger drift dominance in cross-covariance.
App.~\ref{app:drift-checks} evaluates this diagnostic alongside the drift reduction achieved by the deployed estimator; the mean term is analyzed separately in App.~\ref{app:mean-safety}.

\subsection{Empirical checks of the condition and of the online correction}
\label{app:drift-checks}
Let $(a^{\mathrm{c}},u^{\mathrm{c}})$ and $(a^{\mathrm{t}},u^{\mathrm{t}})$ denote the clean and corrupted coordinates and $\delta a=a^{\mathrm{t}}-a^{\mathrm{c}}$.
Paired views identify the two principal components of the corrupted stream in Assumption~\ref{asm:shared-nuisance}: the corruption-induced displacement $\delta a$ plays the role of $L_q\xi_i$ and the clean principal coordinates $a^{\mathrm{c}}$ that of $s_i+\varepsilon_i^{a}$, while the regressor input is the observed corrupted tail $u^{\mathrm{t}}$, exactly as online.
Tab.~\ref{tab:drift-checks} reports three checks per corruption.

\textbf{(i) Drift predictability beyond high-dimensional fitting.}
We fit a pooled, centered ridge map from $u^{\mathrm{t}}$ to $\delta a$ on the $5{,}000$ pairs and report its in-sample $R^2$, i.e., the fraction of centered drift energy that it explains.
Because in-sample ridge fitting can explain part of an unrelated target through finite-sample correlations, we compare it with 200 within-corruption permutations of the rows of $\delta a$, which break the sample pairing while keeping both marginals.
The real $R^2$ is $0.312$ on average ($0.248$--$0.382$), whereas the permutation null is $0.083\pm0.001$ for every corruption; none of the $3{,}000$ permutations reaches the real value (z-scores $325$--$608$).
This regression-based test complements the CCA spectrum of Fig.~\ref{fig:vit-cca-leakage}: the observed tail linearly predicts a substantial part of the principal drift.

\textbf{(ii) Relative strength of the cross-covariance terms.}
The cross-covariance seen by the online regression splits exactly by prediction target,
$\mathrm{Cov}(a^{\mathrm{t}},u^{\mathrm{t}})=\mathrm{Cov}(\delta a,u^{\mathrm{t}})+\mathrm{Cov}(a^{\mathrm{c}},u^{\mathrm{t}})$,
whose two terms are the empirical counterparts of the denominator and numerator of Eq.~\ref{eq:dominance}.
Their spectral-norm ratio $\eta_C=\|\mathrm{Cov}(a^{\mathrm{c}},u^{\mathrm{t}})\|_2/\|\mathrm{Cov}(\delta a,u^{\mathrm{t}})\|_2$ is below one for all 15 corruptions, with mean $0.234$.
The margin is largest for global corruptions such as fog, frost, and contrast ($0.08$--$0.13$) and smallest for mild ones such as brightness, pixelate, and JPEG compression ($0.38$--$0.54$), indicating stronger drift dominance in cross-covariance for the former.

\textbf{(iii) The deployed online estimator.}
We run the TALOR module itself (GMM routing, weighted ridge, and EMA; batch size 64; state reset per corruption) on the corrupted stream and measure the reduction of squared principal drift,
$\mathrm{DR}=1-\sum_i\|\delta a_i-\hat\Delta_i\|_2^2/\sum_i\|\delta a_i\|_2^2$, for the full correction and for its slope and mean terms.
We also report $\mathrm{DR}_{\sigma}$, the corresponding reduction for the full correction after mapping principal directional drift to logits with each sample's corrupted feature norm held fixed, using weights $\rho_i\sigma_j$ (Eq.~\ref{eq:logit-decomposition}).
The full correction and the slope term reduce the drift for all 15 corruptions and the mean term for 14, and $\mathrm{DR}_{\sigma}$ exceeds $\mathrm{DR}$ for every corruption ($0.355$ vs.\ $0.276$ on average): the correction is most effective along the directions that the head amplifies.

\begin{table}[h]
\centering
\caption{Per-corruption checks of the shared-nuisance condition and of the deployed online correction (ViT-B/16, ImageNet-C at severity 5, 5{,}000 pairs per corruption). $R^2$: in-sample ridge fit of the centered drift from the corrupted tail (permutation null $0.083\pm0.001$ for every corruption). $\eta_C$: clean-target to drift-target cross-covariance ratio (smaller values indicate stronger drift dominance in cross-covariance). DR: reduction of squared principal drift by the online TALOR correction and its terms; $\mathrm{DR}_\sigma$: logit-space DR of the full correction.}
\label{tab:drift-checks}
\setlength{\tabcolsep}{5pt}
\small
\begin{tabular}{@{}lcccccc@{}}
\toprule
& \multicolumn{2}{c}{Condition} & \multicolumn{4}{c}{Online correction} \\
\cmidrule(lr){2-3}\cmidrule(lr){4-7}
Corruption & $R^2$ & $\eta_C$ & DR (full) & DR (slope) & DR (mean) & $\mathrm{DR}_\sigma$ \\
\midrule
Gaussian noise & 0.317 & 0.136 & 0.304 & 0.162 & 0.145 & 0.408 \\
Shot noise & 0.303 & 0.171 & 0.243 & 0.164 & 0.077 & 0.330 \\
Impulse noise & 0.314 & 0.159 & 0.265 & 0.163 & 0.099 & 0.361 \\
Defocus blur & 0.335 & 0.184 & 0.340 & 0.160 & 0.189 & 0.470 \\
Glass blur & 0.363 & 0.206 & 0.301 & 0.192 & 0.111 & 0.397 \\
Motion blur & 0.300 & 0.221 & 0.250 & 0.153 & 0.100 & 0.348 \\
Zoom blur & 0.324 & 0.215 & 0.270 & 0.164 & 0.109 & 0.335 \\
Snow & 0.288 & 0.331 & 0.236 & 0.128 & 0.111 & 0.271 \\
Frost & 0.382 & 0.106 & 0.318 & 0.210 & 0.110 & 0.405 \\
Fog & 0.375 & 0.075 & 0.534 & 0.085 & 0.451 & 0.656 \\
Brightness & 0.249 & 0.537 & 0.058 & 0.085 & $-0.025$ & 0.085 \\
Contrast & 0.296 & 0.131 & 0.458 & 0.069 & 0.394 & 0.605 \\
Elastic transform & 0.334 & 0.269 & 0.280 & 0.152 & 0.131 & 0.305 \\
Pixelate & 0.248 & 0.397 & 0.118 & 0.101 & 0.018 & 0.154 \\
JPEG compression & 0.255 & 0.375 & 0.158 & 0.094 & 0.063 & 0.202 \\
\midrule
Mean & 0.312 & 0.234 & 0.276 & 0.139 & 0.139 & 0.355 \\
\bottomrule
\end{tabular}
\vspace{-5pt}
\end{table}

\subsection{Mean term safety}
\label{app:mean-safety}

The mean term $\mu_a^{(q)}$ is the affine intercept of Eq.~\ref{eq:leakage-model}.
Because the slope is fitted on centered coordinates, it cannot represent a principal offset shared within a branch, so the weighted target mean supplies this intercept; an explicit source feature mean is not required for this role.
The observed target mean can contain both a nuisance component (reflecting domain shift) and a semantic component (reflecting the class composition of the current stream), and TALOR subtracts its smoothed estimate without decomposing it.
Two properties limit the risk of removing useful semantics.

First, consider the mean term in isolation at a fixed online
state $t$. Define the mean-only corrected coordinates as
\[
a_i^{\mathrm{mean}}
:= a_i-\hat{\Delta}_i^{\mathrm{mean}},
\qquad
\hat{\Delta}_i^{\mathrm{mean}}
:= \sum_{q\in\mathcal Q}
w_i^{(q)}\mu_{a,t}^{(q)}.
\]
For any two samples,
\begin{equation}
a_i^{\mathrm{mean}}-a_j^{\mathrm{mean}}
=
(a_i-a_j)
-
\sum_{q\in\mathcal Q}
\bigl(w_i^{(q)}-w_j^{(q)}\bigr)\mu_{a,t}^{(q)}.
\label{eq:mean-only-pairwise}
\end{equation}
Thus, samples with identical routing weights receive the same
translation, and their pairwise differences are preserved:
\[
a_i^{\mathrm{mean}}-a_j^{\mathrm{mean}}=a_i-a_j.
\]
This includes hard routing to the same branch.
Within any group receiving the same translation, centered
covariance is also unchanged. Under soft routing, differences
in the weighted means can change pairwise differences.

Second, if the removed mean were predominantly discriminative, ID accuracy would drop; instead, Mean-only improves or preserves ID accuracy for all seven hosts in Tab.~\ref{tab:moduleabla} (e.g., from $49.8\%$ to $54.8\%$ on Source).
Without covariate shift, TALOR changes accuracy by only $+0.12$ pp on an unshifted clean stream and by at most $-0.46$ pp under pure label shift, which alters the target principal mean through the class prior alone (App.~\ref{app:negative-controls}).
Hence the mean term corrects shift-induced bias without materially distorting the semantic geometry.

\section{Negative Controls: Unshifted and Pure Label-Shift Streams}
\label{app:negative-controls}

To verify that TALOR does not over-correct when no covariate shift is present, we conduct two negative-control experiments on ViT-B/16: an \emph{unshifted clean stream} and a \emph{pure label-shift stream}.
In both settings, $p_t(x\mid y)\approx p_s(x\mid y)$, so there is no domain shift to remove; the relevant criterion is whether TALOR preserves the Source accuracy.
Throughout this section, \textbf{TALOR} denotes standalone \texttt{Source\,+\,TALOR}; neither stream contains OOD samples, so we report ID accuracy only.

\paragraph{Unshifted clean stream.}
All methods use the same balanced stream of 5{,}000 clean ImageNet-val samples (5 per class), with no corruption or label-space change.

\begin{table}[h]
\centering
\small
\caption{Unshifted clean stream (ViT-B/16). Top-1 accuracy and change relative to Source (pp).}
\label{tab:neg-ctrl-clean}
\setlength{\tabcolsep}{2.5pt}
\renewcommand{\arraystretch}{1.05}
\begin{tabular}{@{}lccccccccccc@{}}
\toprule
Metric & Source & TALOR & CoTTA & EATA & UniEnt & SAR & OSTTA & ViDA & STAMP & COME & DPCore \\
\midrule
Top-1 Acc.\,(\%) & 84.48 & 84.60 & 84.52 & 84.60 & 84.84 & 84.56 & 84.54 & 84.22 & 84.72 & 84.16 & 84.62 \\
$\Delta$ Acc (pp)  & 0.00  & +0.12 & +0.04 & +0.12 & +0.36 & +0.08 & +0.06 & $-$0.26 & +0.24 & $-$0.32 & +0.14 \\
\bottomrule
\end{tabular}
\end{table}

TALOR changes accuracy by only $+0.12$\,pp, indicating near-identity behavior at the decision level on clean data.

\paragraph{Pure label shift.}
We keep $p_t(x\mid y)$ unchanged and alter only $p_t(y)$ using a long-tailed prior $q_r \propto \mathrm{IF}^{-r/(C-1)}$, where $r$ is the class rank and $C=1000$.
Each stream contains 5{,}000 clean samples, with $\mathrm{IF}\in\{1,10,100\}$ and a shared manifest.
Source accuracies are 85.32\%, 86.16\%, and 84.86\%, respectively.

\begin{table}[h]
\centering
\small
\caption{Pure label-shift streams (ViT-B/16). $\Delta$\,pp in accuracy relative to Source under varying imbalance factors.}
\label{tab:neg-ctrl-labelshift}
\setlength{\tabcolsep}{5pt}
\renewcommand{\arraystretch}{1.05}
\begin{tabular}{@{}lccccccc@{}}
\toprule
IF & TALOR & CoTTA & EATA & UniEnt & SAR & OSTTA & COME \\
\midrule
1   & $-$0.06 & +0.14 & $-$0.14 & $-$0.12 & 0.00 & $-$0.08 & +0.26 \\
10  & $-$0.16 & 0.00 & $-$0.02 & $-$0.12 & $-$0.28 & $-$0.24 & $-$0.50 \\
100 & $-$0.46 & $-$0.02 & $-$0.04 & $-$0.18 & $-$0.08 & +0.06 & $-$0.48 \\
\midrule
Mean & $-$0.23 & +0.04 & $-$0.07 & $-$0.14 & $-$0.12 & $-$0.09 & $-$0.24 \\
\bottomrule
\end{tabular}
\end{table}

TALOR's accuracy changes are $-0.06$, $-0.16$, and $-0.46$\,pp for $\mathrm{IF}=1,10,100$, respectively, with a mean change of $-0.23$\,pp.
Even under the extreme $\mathrm{IF}=100$ prior shift, the change remains below $0.5$\,pp, with no instability or accuracy collapse.
This is comparable to established TTA methods such as COME ($-0.24$\,pp), UniEnt ($-0.14$\,pp), and SAR ($-0.12$\,pp).

This experiment directly stresses the case where the target principal mean changes because of class-prior variation while no covariate drift is present.
TALOR nevertheless remains close to Source across all imbalance factors.
Together with the consistent Mean-only gains under covariate shift (Tab.~\ref{tab:moduleabla}), these results show that the principal-mean term does not substantially erase normal semantic structure; its main benefit appears when a shared high-gain domain response is present.

\vspace{-5pt}
\section{Measuring ID/OOD Leading-Energy Slope Subspaces}
\vspace{-5pt}
\label{app:leakage-modes}

\textbf{Energy-truncated slope subspaces.}
Since each slope map is numerically full column rank with shape \(415\times353\), their full column spaces necessarily intersect in at least \(291\) dimensions, yielding the same number of trivial zero principal angles regardless of empirical alignment. To avoid this dimension-induced degeneracy, we compare their leading-energy subspaces, using a 90\% Frobenius-energy target subject to the rank cap \(2r\leq D-k\). For each regime, write the compact SVD as
\begin{equation}
R_t^{(q)}=U_q\,\mathrm{diag}(s_{q,1},\ldots,s_{q,m_q})\,V_q^\top,
\qquad s_{q,1}\geq\cdots\geq s_{q,m_q}>0.
\end{equation}
Let $r_q$ be the smallest rank retaining 90\% of the slope-map Frobenius energy,
\begin{equation}
r_q=\min\left\{\ell:\frac{\sum_{j=1}^{\ell}s_{q,j}^2}{\sum_{j=1}^{m_q}s_{q,j}^2}\geq0.90\right\},
\qquad
r=\min\!\left(\max(r_{\mathrm{ID}},r_{\mathrm{OOD}}),\left\lfloor\frac{D-k}{2}\right\rfloor\right).
\end{equation}
We compare the leading-energy tail subspaces
$\mathcal{S}^{(q)}_r:=\mathrm{col}(U_q[:,1\!:\!r])\subset\mathbb{R}^{D-k}$.
The cap $2r\leq D-k$ excludes dimension-forced intersections; for the displayed cell (contrast, iNaturalist-C), $r=38$.

\textbf{Principal-angle calibration.}
With $U_{q,r}=U_q[:,1\!:\!r]$, we compute
\begin{equation}
U_{\mathrm{ID},r}^\top U_{\mathrm{OOD},r}=P\,\mathrm{diag}(\sigma_1,\ldots,\sigma_r)\,T^\top,
\qquad
\theta_j=\arccos(\sigma_j)\in[0,\tfrac{\pi}{2}].
\end{equation}
We evaluate the spectrum using cross-halves of the ID and OOD streams, and calibrate it against within-regime split-half spectra and pairs of independent Haar-random $r$-dimensional subspaces in $\mathbb{R}^{D-k}$.
The split-half references quantify finite-stream estimation noise, whereas the random reference quantifies the overlap expected for unrelated subspaces of the same dimension.

In Fig.~\ref{fig:dual_leakage_mode}, the ID/OOD cross-half spectrum is well below the Haar-random reference and close to the ID split-half spectrum over the first ${\sim}10$ directions, and then departs toward the random reference.
The remaining gap to the within-regime spectra shows that the two maps are not identical and is consistent with residual regime-specific structure.
Accordingly, the diagnostic supports---but does not by itself prove---the routing interpretation in Tab.~\ref{tab:moduleabla}: a mixed rectifier can exploit the shared leading component, while ID/OOD routing can refine residual structure.

\vspace{-5pt}
\section{Result on ResNet-50 backbone}
\vspace{-5pt}
\label{app:resnet50}

% Requires: \usepackage{booktabs,multirow,xcolor}
% Softer colors (deep & less neon)
\definecolor{talorgreen}{HTML}{1B5E20}
\definecolor{talorred}{HTML}{8B1A1A}
\begin{table*}[h]
\centering
\caption{Results on the ResNet-50 backbone and ImageNet-to-ImageNet-C in OSTTA with TALOR plugin across six covariate-shifted OOD datasets. All the results are \textbf{averaged over 15} corruptions. }
\setlength{\tabcolsep}{1.0pt}
\renewcommand{\arraystretch}{0.4}
{\small
\begin{tabular}{ccc@{\hspace{8pt}}cc@{\hspace{8pt}}cc@{\hspace{8pt}}cc@{\hspace{8pt}}cc@{\hspace{8pt}}cc@{\hspace{8pt}}ccc}
\toprule
\multicolumn{1}{c}{\multirow{2}{*}{Method}} & \multicolumn{2}{c}{Places.-C} & \multicolumn{2}{c}{Textures-C} & \multicolumn{2}{c}{iNatur.-C} & \multicolumn{2}{c}{SUN-C} & \multicolumn{2}{c}{SSB-H.C} & \multicolumn{2}{c}{NINCO--C} & \multicolumn{3}{c}{Mean} \\
\cmidrule(lr){2-3} \cmidrule(lr){4-5} \cmidrule(lr){6-7} \cmidrule(lr){8-9} \cmidrule(lr){10-11} \cmidrule(lr){12-13} \cmidrule(lr){14-16}
\multicolumn{1}{c}{} & Acc & AUC & Acc & AUC & Acc & AUC & Acc & AUC & Acc & AUC & Acc & AUC & Acc & AUC & H-s \\
\midrule
Source & 17.7 & 64.7 & 17.7 & 60.1 & 17.7 & 75.1 & 17.7 & 70.2 & 17.7 & 54.0 & 17.7 & 60.6 & 17.7 & 64.1 & 25.9 \\
+ TALOR & 21.9 & 74.0 & 21.8 & 74.8 & 21.8 & 90.4 & 21.9 & 86.0 & 21.9 & 58.1 & 21.8 & 67.5 & 21.9 & 75.1 & 32.0 \\
Improve & \textcolor{talorgreen}{+4.2} & \textcolor{talorgreen}{+9.4} & \textcolor{talorgreen}{+4.2} & \textcolor{talorgreen}{+14.7} & \textcolor{talorgreen}{+4.1} & \textcolor{talorgreen}{+15.3} & \textcolor{talorgreen}{+4.3} & \textcolor{talorgreen}{+15.8} & \textcolor{talorgreen}{+4.2} & \textcolor{talorgreen}{+4.0} & \textcolor{talorgreen}{+4.2} & \textcolor{talorgreen}{+6.9} & \textcolor{talorgreen}{+4.2} & \textcolor{talorgreen}{+11.0} & \textcolor{talorgreen}{+6.1} \\
\midrule
EATA & 32.2 & 76.2 & 29.9 & 53.7 & 13.5 & 63.5 & 27.5 & 82.3 & 35.6 & 56.0 & 32.4 & 63.4 & 28.5 & 65.8 & 37.2 \\
+ TALOR & 32.6 & 79.4 & 30.8 & 64.8 & 14.6 & 76.1 & 28.2 & 87.6 & 36.1 & 57.9 & 33.0 & 67.2 & 29.2 & 72.2 & 39.0 \\
Improve & \textcolor{talorgreen}{+0.4} & \textcolor{talorgreen}{+3.2} & \textcolor{talorgreen}{+0.9} & \textcolor{talorgreen}{+11.2} & \textcolor{talorgreen}{+1.1} & \textcolor{talorgreen}{+12.6} & \textcolor{talorgreen}{+0.7} & \textcolor{talorgreen}{+5.3} & \textcolor{talorgreen}{+0.5} & \textcolor{talorgreen}{+1.9} & \textcolor{talorgreen}{+0.7} & \textcolor{talorgreen}{+3.9} & \textcolor{talorgreen}{+0.7} & \textcolor{talorgreen}{+6.3} & \textcolor{talorgreen}{+1.8} \\
\midrule
SAR & 13.7 & 65.3 & 8.2 & 36.3 & 13.4 & 62.4 & 14.2 & 66.2 & 11.3 & 48.4 & 7.8 & 47.3 & 11.4 & 54.3 & 17.8 \\
+ TALOR & 17.6 & 70.8 & 11.6 & 56.7 & 17.0 & 86.3 & 17.7 & 85.9 & 15.3 & 51.5 & 10.3 & 54.5 & 14.9 & 67.6 & 23.4 \\
Improve & \textcolor{talorgreen}{+3.9} & \textcolor{talorgreen}{+5.5} & \textcolor{talorgreen}{+3.4} & \textcolor{talorgreen}{+20.4} & \textcolor{talorgreen}{+3.6} & \textcolor{talorgreen}{+23.9} & \textcolor{talorgreen}{+3.6} & \textcolor{talorgreen}{+19.7} & \textcolor{talorgreen}{+4.0} & \textcolor{talorgreen}{+3.1} & \textcolor{talorgreen}{+2.5} & \textcolor{talorgreen}{+7.2} & \textcolor{talorgreen}{+3.5} & \textcolor{talorgreen}{+13.3} & \textcolor{talorgreen}{+5.7} \\
\midrule
OSTTA & 20.2 & 63.9 & 20.1 & 41.4 & 16.8 & 49.1 & 17.5 & 57.4 & 26.1 & 51.0 & 22.1 & 51.0 & 20.5 & 52.3 & 28.5 \\
+ TALOR & 22.5 & 70.6 & 22.3 & 59.3 & 19.5 & 78.6 & 20.1 & 77.6 & 28.0 & 53.8 & 24.5 & 61.3 & 22.8 & 66.9 & 33.0 \\
Improve & \textcolor{talorgreen}{+2.2} & \textcolor{talorgreen}{+6.7} & \textcolor{talorgreen}{+2.2} & \textcolor{talorgreen}{+17.8} & \textcolor{talorgreen}{+2.7} & \textcolor{talorgreen}{+29.5} & \textcolor{talorgreen}{+2.6} & \textcolor{talorgreen}{+20.2} & \textcolor{talorgreen}{+2.0} & \textcolor{talorgreen}{+2.8} & \textcolor{talorgreen}{+2.4} & \textcolor{talorgreen}{+10.3} & \textcolor{talorgreen}{+2.3} & \textcolor{talorgreen}{+14.6} & \textcolor{talorgreen}{+4.5} \\
\midrule
ViDA & 18.0 & 63.9 & 17.9 & 57.6 & 18.0 & 72.5 & 18.0 & 67.4 & 18.0 & 54.1 & 18.0 & 59.8 & 18.0 & 62.6 & 26.0 \\
+ TALOR & 22.3 & 74.3 & 22.2 & 73.9 & 22.2 & 90.2 & 22.3 & 85.9 & 22.3 & 58.3 & 22.2 & 67.4 & 22.2 & 75.0 & 32.4 \\
Improve & \textcolor{talorgreen}{+4.3} & \textcolor{talorgreen}{+10.4} & \textcolor{talorgreen}{+4.3} & \textcolor{talorgreen}{+16.3} & \textcolor{talorgreen}{+4.2} & \textcolor{talorgreen}{+17.7} & \textcolor{talorgreen}{+4.3} & \textcolor{talorgreen}{+18.5} & \textcolor{talorgreen}{+4.2} & \textcolor{talorgreen}{+4.2} & \textcolor{talorgreen}{+4.2} & \textcolor{talorgreen}{+7.6} & \textcolor{talorgreen}{+4.3} & \textcolor{talorgreen}{+12.4} & \textcolor{talorgreen}{+6.4} \\
\midrule
STAMP & 28.3 & 70.2 & 27.3 & 63.1 & 26.5 & 81.9 & 27.1 & 78.6 & 29.2 & 59.0 & 29.0 & 62.4 & 27.9 & 69.2 & 37.9 \\
+ TALOR & 28.5 & 70.2 & 27.7 & 70.7 & 27.4 & 82.5 & 27.8 & 78.9 & 29.5 & 57.2 & 29.2 & 65.7 & 28.4 & 70.9 & 38.7 \\
Improve & \textcolor{talorgreen}{+0.2} & \textcolor{talorgreen}{+0.0} & \textcolor{talorgreen}{+0.4} & \textcolor{talorgreen}{+7.6} & \textcolor{talorgreen}{+0.8} & \textcolor{talorgreen}{+0.5} & \textcolor{talorgreen}{+0.7} & \textcolor{talorgreen}{+0.3} & \textcolor{talorgreen}{+0.3} & \textcolor{talorred}{-1.8} & \textcolor{talorgreen}{+0.3} & \textcolor{talorgreen}{+3.2} & \textcolor{talorgreen}{+0.4} & \textcolor{talorgreen}{+1.6} & \textcolor{talorgreen}{+0.9} \\
\bottomrule
\end{tabular}
}
\end{table*}

\vspace{-5pt}
\section{Results under Different OOD Ratios}
\vspace{-5pt}
\label{app:oodratios}

% Requires: \usepackage{booktabs,multirow,xcolor}
% Softer colors (deep & less neon)
\definecolor{talorgreen}{HTML}{1B5E20}
\definecolor{talorred}{HTML}{8B1A1A}
\begin{table*}[t]
\centering
\caption{Results ($\%$) for ImageNet-to-ImageNet-C in csID:csOOD=8:2 ratio with the TALOR plug-in across six covariate-shifted OOD datasets. \textcolor{talorgreen}{Improvement} and \textcolor{talorred}{degradation} compared to the baseline are highlighted in green or red respectively.  All the results are \textbf{averaged over 15} corruptions.}
\setlength{\tabcolsep}{1.0pt}
\renewcommand{\arraystretch}{0.3}
{\small
\begin{tabular}{ccc@{\hspace{8pt}}cc@{\hspace{8pt}}cc@{\hspace{8pt}}cc@{\hspace{8pt}}cc@{\hspace{8pt}}cc@{\hspace{8pt}}ccc}
\toprule
\multicolumn{1}{c}{\multirow{2}{*}{Method}} & \multicolumn{2}{c}{Places.-C} & \multicolumn{2}{c}{Textures-C} & \multicolumn{2}{c}{iNatur.-C} & \multicolumn{2}{c}{SUN-C} & \multicolumn{2}{c}{SSB-H.C} & \multicolumn{2}{c}{NINCO--C} & \multicolumn{3}{c}{Mean} \\
\cmidrule(lr){2-3} \cmidrule(lr){4-5} \cmidrule(lr){6-7} \cmidrule(lr){8-9} \cmidrule(lr){10-11} \cmidrule(lr){12-13} \cmidrule(lr){14-16}
\multicolumn{1}{c}{} & Acc & AUC & Acc & AUC & Acc & AUC & Acc & AUC & Acc & AUC & Acc & AUC & Acc & AUC & H-s \\
\midrule
Source & 49.8 & 66.8 & 49.8 & 70.2 & 49.8 & 78.9 & 49.8 & 71.4 & 49.8 & 55.9 & 49.8 & 64.8 & 49.8 & 68.0 & 56.3 \\
+ TALOR & 55.9 & 75.6 & 55.8 & 80.8 & 55.8 & 90.7 & 55.9 & 88.0 & 55.8 & 64.6 & 55.8 & 74.7 & 55.8 & 79.1 & 64.5 \\
Improve & \textcolor{talorgreen}{+6.1} & \textcolor{talorgreen}{+8.8} & \textcolor{talorgreen}{+6.0} & \textcolor{talorgreen}{+10.6} & \textcolor{talorgreen}{+6.1} & \textcolor{talorgreen}{+11.8} & \textcolor{talorgreen}{+6.1} & \textcolor{talorgreen}{+16.6} & \textcolor{talorgreen}{+6.1} & \textcolor{talorgreen}{+8.7} & \textcolor{talorgreen}{+6.1} & \textcolor{talorgreen}{+9.8} & \textcolor{talorgreen}{+6.1} & \textcolor{talorgreen}{+11.0} & \textcolor{talorgreen}{+8.2} \\
\midrule
CoTTA & 49.9 & 65.8 & 49.5 & 68.3 & 49.7 & 78.3 & 50.0 & 69.9 & 49.8 & 56.3 & 49.6 & 64.6 & 49.7 & 67.2 & 56.1 \\
+ TALOR & 56.0 & 75.7 & 55.8 & 81.2 & 55.8 & 90.2 & 56.0 & 88.3 & 55.9 & 64.0 & 55.9 & 74.7 & 55.9 & 79.0 & 64.5 \\
Improve & \textcolor{talorgreen}{+6.1} & \textcolor{talorgreen}{+9.9} & \textcolor{talorgreen}{+6.3} & \textcolor{talorgreen}{+12.9} & \textcolor{talorgreen}{+6.1} & \textcolor{talorgreen}{+11.8} & \textcolor{talorgreen}{+6.0} & \textcolor{talorgreen}{+18.4} & \textcolor{talorgreen}{+6.1} & \textcolor{talorgreen}{+7.7} & \textcolor{talorgreen}{+6.3} & \textcolor{talorgreen}{+10.1} & \textcolor{talorgreen}{+6.1} & \textcolor{talorgreen}{+11.8} & \textcolor{talorgreen}{+8.4} \\
\midrule
EATA & 59.2 & 65.0 & 56.8 & 69.8 & 54.8 & 73.5 & 57.9 & 68.8 & 59.4 & 61.1 & 56.9 & 66.2 & 57.5 & 67.4 & 61.3 \\
+ TALOR & 60.0 & 74.1 & 57.5 & 83.2 & 55.5 & 89.2 & 58.6 & 86.5 & 60.3 & 67.1 & 57.8 & 75.3 & 58.3 & 79.2 & 66.4 \\
Improve & \textcolor{talorgreen}{+0.8} & \textcolor{talorgreen}{+9.1} & \textcolor{talorgreen}{+0.7} & \textcolor{talorgreen}{+13.5} & \textcolor{talorgreen}{+0.7} & \textcolor{talorgreen}{+15.7} & \textcolor{talorgreen}{+0.7} & \textcolor{talorgreen}{+17.7} & \textcolor{talorgreen}{+0.9} & \textcolor{talorgreen}{+6.0} & \textcolor{talorgreen}{+0.9} & \textcolor{talorgreen}{+9.0} & \textcolor{talorgreen}{+0.8} & \textcolor{talorgreen}{+11.8} & \textcolor{talorgreen}{+5.1} \\
\midrule
SAR & 57.1 & 64.5 & 55.6 & 55.5 & 55.8 & 79.3 & 55.6 & 67.2 & 58.4 & 64.8 & 58.2 & 70.5 & 56.8 & 67.0 & 60.7 \\
+ TALOR & 58.7 & 77.4 & 57.5 & 87.7 & 57.5 & 93.2 & 57.6 & 90.6 & 59.8 & 67.7 & 59.8 & 78.4 & 58.5 & 82.5 & 67.9 \\
Improve & \textcolor{talorgreen}{+1.6} & \textcolor{talorgreen}{+12.9} & \textcolor{talorgreen}{+1.9} & \textcolor{talorgreen}{+32.1} & \textcolor{talorgreen}{+1.6} & \textcolor{talorgreen}{+13.9} & \textcolor{talorgreen}{+1.9} & \textcolor{talorgreen}{+23.4} & \textcolor{talorgreen}{+1.4} & \textcolor{talorgreen}{+2.9} & \textcolor{talorgreen}{+1.6} & \textcolor{talorgreen}{+7.8} & \textcolor{talorgreen}{+1.7} & \textcolor{talorgreen}{+15.5} & \textcolor{talorgreen}{+7.1} \\
\midrule
OSTTA & 58.5 & 64.6 & 58.0 & 62.1 & 57.9 & 77.5 & 58.2 & 68.1 & 58.8 & 63.7 & 58.5 & 68.7 & 58.3 & 67.4 & 62.1 \\
+ TALOR & 60.0 & 77.4 & 59.4 & 87.2 & 59.1 & 93.8 & 59.6 & 91.1 & 60.2 & 67.5 & 60.0 & 77.8 & 59.7 & 82.5 & 68.8 \\
Improve & \textcolor{talorgreen}{+1.4} & \textcolor{talorgreen}{+12.9} & \textcolor{talorgreen}{+1.4} & \textcolor{talorgreen}{+25.1} & \textcolor{talorgreen}{+1.3} & \textcolor{talorgreen}{+16.3} & \textcolor{talorgreen}{+1.4} & \textcolor{talorgreen}{+23.0} & \textcolor{talorgreen}{+1.4} & \textcolor{talorgreen}{+3.7} & \textcolor{talorgreen}{+1.5} & \textcolor{talorgreen}{+9.1} & \textcolor{talorgreen}{+1.4} & \textcolor{talorgreen}{+15.0} & \textcolor{talorgreen}{+6.7} \\
\midrule
ViDA & 50.5 & 66.5 & 50.4 & 70.5 & 50.5 & 79.1 & 50.4 & 71.1 & 50.5 & 56.4 & 50.5 & 65.2 & 50.5 & 68.2 & 57.0 \\
+ TALOR & 56.4 & 75.8 & 56.2 & 81.5 & 56.3 & 91.0 & 56.3 & 88.3 & 56.3 & 65.0 & 56.2 & 75.0 & 56.3 & 79.5 & 65.0 \\
Improve & \textcolor{talorgreen}{+5.9} & \textcolor{talorgreen}{+9.3} & \textcolor{talorgreen}{+5.8} & \textcolor{talorgreen}{+11.0} & \textcolor{talorgreen}{+5.8} & \textcolor{talorgreen}{+12.0} & \textcolor{talorgreen}{+5.8} & \textcolor{talorgreen}{+17.2} & \textcolor{talorgreen}{+5.7} & \textcolor{talorgreen}{+8.5} & \textcolor{talorgreen}{+5.8} & \textcolor{talorgreen}{+9.9} & \textcolor{talorgreen}{+5.8} & \textcolor{talorgreen}{+11.3} & \textcolor{talorgreen}{+7.9} \\
\midrule
UniEnt & 59.5 & 73.9 & 57.4 & 80.4 & 59.4 & 88.3 & 59.4 & 82.6 & 59.3 & 62.4 & 59.4 & 72.9 & 59.1 & 76.7 & 66.2 \\
+ TALOR & 60.5 & 79.1 & 58.4 & 86.9 & 60.4 & 95.2 & 60.5 & 91.2 & 60.6 & 67.4 & 60.4 & 79.2 & 60.1 & 83.2 & 69.3 \\
Improve & \textcolor{talorgreen}{+1.1} & \textcolor{talorgreen}{+5.2} & \textcolor{talorgreen}{+0.9} & \textcolor{talorgreen}{+6.5} & \textcolor{talorgreen}{+1.0} & \textcolor{talorgreen}{+6.9} & \textcolor{talorgreen}{+1.1} & \textcolor{talorgreen}{+8.6} & \textcolor{talorgreen}{+1.3} & \textcolor{talorgreen}{+5.0} & \textcolor{talorgreen}{+1.1} & \textcolor{talorgreen}{+6.4} & \textcolor{talorgreen}{+1.1} & \textcolor{talorgreen}{+6.4} & \textcolor{talorgreen}{+3.1} \\
\midrule
STAMP & 51.7 & 69.9 & 50.9 & 72.8 & 51.7 & 83.6 & 51.1 & 75.5 & 51.6 & 59.2 & 51.6 & 68.2 & 51.5 & 71.5 & 59.0 \\
+ TALOR & 55.4 & 77.4 & 54.5 & 83.4 & 55.5 & 92.2 & 54.9 & 89.0 & 55.4 & 65.2 & 55.4 & 75.8 & 55.2 & 80.5 & 64.7 \\
Improve & \textcolor{talorgreen}{+3.8} & \textcolor{talorgreen}{+7.6} & \textcolor{talorgreen}{+3.6} & \textcolor{talorgreen}{+10.5} & \textcolor{talorgreen}{+3.7} & \textcolor{talorgreen}{+8.6} & \textcolor{talorgreen}{+3.7} & \textcolor{talorgreen}{+13.4} & \textcolor{talorgreen}{+3.8} & \textcolor{talorgreen}{+6.0} & \textcolor{talorgreen}{+3.8} & \textcolor{talorgreen}{+7.6} & \textcolor{talorgreen}{+3.7} & \textcolor{talorgreen}{+9.0} & \textcolor{talorgreen}{+5.7} \\
\midrule
COME & 57.9 & 75.5 & 58.0 & 77.7 & 59.1 & 85.8 & 59.3 & 82.0 & 57.6 & 59.2 & 58.6 & 72.5 & 58.4 & 75.4 & 65.3 \\
+ TALOR & 58.7 & 80.7 & 58.7 & 87.1 & 59.7 & 94.6 & 59.9 & 92.3 & 58.3 & 67.1 & 59.4 & 78.5 & 59.1 & 83.4 & 68.7 \\
Improve & \textcolor{talorgreen}{+0.8} & \textcolor{talorgreen}{+5.1} & \textcolor{talorgreen}{+0.7} & \textcolor{talorgreen}{+9.5} & \textcolor{talorgreen}{+0.6} & \textcolor{talorgreen}{+8.9} & \textcolor{talorgreen}{+0.6} & \textcolor{talorgreen}{+10.3} & \textcolor{talorgreen}{+0.8} & \textcolor{talorgreen}{+7.9} & \textcolor{talorgreen}{+0.8} & \textcolor{talorgreen}{+6.0} & \textcolor{talorgreen}{+0.7} & \textcolor{talorgreen}{+8.0} & \textcolor{talorgreen}{+3.4} \\
\midrule
DPCore & 58.5 & 76.5 & 59.8 & 80.9 & 56.3 & 93.8 & 60.5 & 87.3 & 60.1 & 62.6 & 58.2 & 75.5 & 58.9 & 79.4 & 67.1 \\
+ TALOR & 58.8 & 80.5 & 60.1 & 87.1 & 56.6 & 95.7 & 60.6 & 93.7 & 60.3 & 68.2 & 58.6 & 79.5 & 59.2 & 84.1 & 69.0 \\
Improve & \textcolor{talorgreen}{+0.3} & \textcolor{talorgreen}{+4.1} & \textcolor{talorgreen}{+0.3} & \textcolor{talorgreen}{+6.2} & \textcolor{talorgreen}{+0.3} & \textcolor{talorgreen}{+1.9} & \textcolor{talorgreen}{+0.2} & \textcolor{talorgreen}{+6.4} & \textcolor{talorgreen}{+0.2} & \textcolor{talorgreen}{+5.6} & \textcolor{talorgreen}{+0.3} & \textcolor{talorgreen}{+4.0} & \textcolor{talorgreen}{+0.3} & \textcolor{talorgreen}{+4.7} & \textcolor{talorgreen}{+1.9} \\
\bottomrule
\end{tabular}
}
\vspace{-5pt}
\end{table*}

% Requires: \usepackage{booktabs,multirow,xcolor}
% Softer colors (deep & less neon)
\definecolor{talorgreen}{HTML}{1B5E20}
\definecolor{talorred}{HTML}{8B1A1A}
\begin{table*}[t]
\centering
\caption{Results ($\%$) for ImageNet-to-ImageNet-C in csID:csOOD=3:7 ratio with the TALOR plug-in across six covariate-shifted OOD datasets. \textcolor{talorgreen}{Improvement} and \textcolor{talorred}{degradation} compared to the baseline are highlighted in green or red respectively.  All the results are \textbf{averaged over 15} corruptions.}
\setlength{\tabcolsep}{1.0pt}
\renewcommand{\arraystretch}{0.4}
{\small
\begin{tabular}{ccc@{\hspace{8pt}}cc@{\hspace{8pt}}cc@{\hspace{8pt}}cc@{\hspace{8pt}}cc@{\hspace{8pt}}cc@{\hspace{8pt}}ccc}
\toprule
\multicolumn{1}{c}{\multirow{2}{*}{Method}} & \multicolumn{2}{c}{Places.-C} & \multicolumn{2}{c}{Textures-C} & \multicolumn{2}{c}{iNatur.-C} & \multicolumn{2}{c}{SUN-C} & \multicolumn{2}{c}{SSB-H.C} & \multicolumn{2}{c}{NINCO--C} & \multicolumn{3}{c}{Mean} \\
\cmidrule(lr){2-3} \cmidrule(lr){4-5} \cmidrule(lr){6-7} \cmidrule(lr){8-9} \cmidrule(lr){10-11} \cmidrule(lr){12-13} \cmidrule(lr){14-16}
\multicolumn{1}{c}{} & Acc & AUC & Acc & AUC & Acc & AUC & Acc & AUC & Acc & AUC & Acc & AUC & Acc & AUC & H-s \\
\midrule
Source & 49.5 & 67.0 & 49.5 & 70.7 & 49.5 & 78.7 & 49.5 & 71.5 & 49.5 & 56.2 & 49.5 & 64.6 & 49.5 & 68.1 & 56.2 \\
+ TALOR & 55.6 & 81.9 & 55.3 & 87.6 & 55.4 & 94.2 & 55.5 & 92.8 & 55.4 & 73.2 & 55.5 & 81.7 & 55.5 & 85.3 & 66.4 \\
Improve & \textcolor{talorgreen}{+6.1} & \textcolor{talorgreen}{+14.9} & \textcolor{talorgreen}{+5.8} & \textcolor{talorgreen}{+17.0} & \textcolor{talorgreen}{+5.9} & \textcolor{talorgreen}{+15.5} & \textcolor{talorgreen}{+6.0} & \textcolor{talorgreen}{+21.3} & \textcolor{talorgreen}{+5.9} & \textcolor{talorgreen}{+17.0} & \textcolor{talorgreen}{+6.0} & \textcolor{talorgreen}{+17.1} & \textcolor{talorgreen}{+6.0} & \textcolor{talorgreen}{+17.1} & \textcolor{talorgreen}{+10.2} \\
\midrule
CoTTA & 49.4 & 63.2 & 48.7 & 65.8 & 49.0 & 75.6 & 49.2 & 66.0 & 49.5 & 54.7 & 48.9 & 61.8 & 49.1 & 64.5 & 54.5 \\
+ TALOR & 55.4 & 81.1 & 55.1 & 87.4 & 55.2 & 93.5 & 55.5 & 92.8 & 55.4 & 71.9 & 55.2 & 80.6 & 55.3 & 84.5 & 66.0 \\
Improve & \textcolor{talorgreen}{+6.1} & \textcolor{talorgreen}{+18.0} & \textcolor{talorgreen}{+6.4} & \textcolor{talorgreen}{+21.6} & \textcolor{talorgreen}{+6.2} & \textcolor{talorgreen}{+17.9} & \textcolor{talorgreen}{+6.2} & \textcolor{talorgreen}{+26.8} & \textcolor{talorgreen}{+5.9} & \textcolor{talorgreen}{+17.2} & \textcolor{talorgreen}{+6.3} & \textcolor{talorgreen}{+18.8} & \textcolor{talorgreen}{+6.2} & \textcolor{talorgreen}{+20.0} & \textcolor{talorgreen}{+11.5} \\
\midrule
EATA & 50.2 & 58.9 & 22.6 & 56.8 & 37.5 & 63.1 & 45.5 & 65.6 & 40.7 & 55.1 & 48.5 & 61.5 & 40.8 & 60.2 & 44.8 \\
+ TALOR & 50.7 & 74.1 & 22.9 & 63.9 & 38.2 & 79.8 & 46.3 & 81.1 & 41.2 & 67.0 & 49.2 & 76.0 & 41.4 & 73.6 & 50.0 \\
Improve & \textcolor{talorgreen}{+0.5} & \textcolor{talorgreen}{+15.2} & \textcolor{talorgreen}{+0.4} & \textcolor{talorgreen}{+7.1} & \textcolor{talorgreen}{+0.7} & \textcolor{talorgreen}{+16.7} & \textcolor{talorgreen}{+0.8} & \textcolor{talorgreen}{+15.4} & \textcolor{talorgreen}{+0.6} & \textcolor{talorgreen}{+11.9} & \textcolor{talorgreen}{+0.7} & \textcolor{talorgreen}{+14.4} & \textcolor{talorgreen}{+0.6} & \textcolor{talorgreen}{+13.5} & \textcolor{talorgreen}{+5.1} \\
\midrule
SAR & 48.6 & 56.0 & 43.6 & 49.6 & 43.2 & 49.9 & 41.8 & 44.1 & 55.1 & 59.7 & 50.2 & 62.3 & 47.1 & 53.6 & 48.2 \\
+ TALOR & 52.7 & 80.6 & 50.5 & 87.6 & 47.2 & 88.7 & 46.4 & 90.5 & 56.8 & 74.6 & 53.2 & 81.6 & 51.1 & 83.9 & 62.3 \\
Improve & \textcolor{talorgreen}{+4.1} & \textcolor{talorgreen}{+24.7} & \textcolor{talorgreen}{+6.9} & \textcolor{talorgreen}{+37.9} & \textcolor{talorgreen}{+4.0} & \textcolor{talorgreen}{+38.8} & \textcolor{talorgreen}{+4.7} & \textcolor{talorgreen}{+46.5} & \textcolor{talorgreen}{+1.7} & \textcolor{talorgreen}{+14.8} & \textcolor{talorgreen}{+2.9} & \textcolor{talorgreen}{+19.3} & \textcolor{talorgreen}{+4.1} & \textcolor{talorgreen}{+30.3} & \textcolor{talorgreen}{+14.1} \\
\midrule
OSTTA & 55.9 & 58.2 & 52.4 & 52.2 & 45.6 & 60.3 & 53.4 & 57.2 & 57.6 & 60.3 & 55.5 & 65.3 & 53.4 & 58.9 & 55.5 \\
+ TALOR & 57.2 & 81.8 & 53.8 & 90.0 & 47.0 & 91.0 & 54.5 & 92.5 & 59.0 & 76.4 & 56.8 & 83.3 & 54.7 & 85.8 & 66.2 \\
Improve & \textcolor{talorgreen}{+1.3} & \textcolor{talorgreen}{+23.5} & \textcolor{talorgreen}{+1.5} & \textcolor{talorgreen}{+37.8} & \textcolor{talorgreen}{+1.4} & \textcolor{talorgreen}{+30.8} & \textcolor{talorgreen}{+1.1} & \textcolor{talorgreen}{+35.4} & \textcolor{talorgreen}{+1.4} & \textcolor{talorgreen}{+16.1} & \textcolor{talorgreen}{+1.2} & \textcolor{talorgreen}{+18.0} & \textcolor{talorgreen}{+1.3} & \textcolor{talorgreen}{+26.9} & \textcolor{talorgreen}{+10.7} \\
\midrule
ViDA & 49.8 & 65.2 & 49.7 & 68.3 & 49.6 & 76.6 & 49.7 & 68.4 & 50.1 & 56.1 & 49.8 & 64.0 & 49.8 & 66.5 & 55.9 \\
+ TALOR & 55.8 & 81.9 & 55.6 & 87.7 & 55.7 & 94.2 & 55.7 & 92.8 & 55.8 & 73.6 & 55.9 & 81.9 & 55.8 & 85.3 & 66.6 \\
Improve & \textcolor{talorgreen}{+6.1} & \textcolor{talorgreen}{+16.7} & \textcolor{talorgreen}{+5.9} & \textcolor{talorgreen}{+19.3} & \textcolor{talorgreen}{+6.1} & \textcolor{talorgreen}{+17.6} & \textcolor{talorgreen}{+6.1} & \textcolor{talorgreen}{+24.4} & \textcolor{talorgreen}{+5.8} & \textcolor{talorgreen}{+17.5} & \textcolor{talorgreen}{+6.1} & \textcolor{talorgreen}{+17.9} & \textcolor{talorgreen}{+6.0} & \textcolor{talorgreen}{+18.9} & \textcolor{talorgreen}{+10.8} \\
\midrule
UniEnt & 56.6 & 71.2 & 51.0 & 76.1 & 54.1 & 86.7 & 41.6 & 76.1 & 58.2 & 61.4 & 57.6 & 69.9 & 53.2 & 73.6 & 59.9 \\
+ TALOR & 57.2 & 80.8 & 51.9 & 85.1 & 54.8 & 93.0 & 42.2 & 80.5 & 59.1 & 75.2 & 58.5 & 83.9 & 53.9 & 83.1 & 64.1 \\
Improve & \textcolor{talorgreen}{+0.7} & \textcolor{talorgreen}{+9.6} & \textcolor{talorgreen}{+0.8} & \textcolor{talorgreen}{+9.0} & \textcolor{talorgreen}{+0.7} & \textcolor{talorgreen}{+6.2} & \textcolor{talorgreen}{+0.5} & \textcolor{talorgreen}{+4.4} & \textcolor{talorgreen}{+1.0} & \textcolor{talorgreen}{+13.8} & \textcolor{talorgreen}{+0.9} & \textcolor{talorgreen}{+14.0} & \textcolor{talorgreen}{+0.8} & \textcolor{talorgreen}{+9.5} & \textcolor{talorgreen}{+4.2} \\
\midrule
STAMP & 51.6 & 70.3 & 51.4 & 76.1 & 51.5 & 84.8 & 51.5 & 77.1 & 51.3 & 59.7 & 51.4 & 69.2 & 51.5 & 72.9 & 59.5 \\
+ TALOR & 54.4 & 83.9 & 53.8 & 91.0 & 54.0 & 96.4 & 54.1 & 94.5 & 53.8 & 73.9 & 53.8 & 83.5 & 54.0 & 87.2 & 65.9 \\
Improve & \textcolor{talorgreen}{+2.7} & \textcolor{talorgreen}{+13.6} & \textcolor{talorgreen}{+2.4} & \textcolor{talorgreen}{+14.9} & \textcolor{talorgreen}{+2.4} & \textcolor{talorgreen}{+11.6} & \textcolor{talorgreen}{+2.6} & \textcolor{talorgreen}{+17.4} & \textcolor{talorgreen}{+2.5} & \textcolor{talorgreen}{+14.2} & \textcolor{talorgreen}{+2.4} & \textcolor{talorgreen}{+14.3} & \textcolor{talorgreen}{+2.5} & \textcolor{talorgreen}{+14.3} & \textcolor{talorgreen}{+6.5} \\
\midrule
COME & 12.9 & 49.5 & 53.6 & 77.4 & 54.1 & 83.8 & 54.4 & 80.7 & 53.6 & 56.9 & 51.8 & 70.3 & 46.7 & 69.8 & 53.5 \\
+ TALOR & 13.2 & 60.1 & 54.6 & 91.2 & 54.8 & 95.2 & 55.4 & 94.1 & 54.4 & 75.9 & 52.7 & 82.3 & 47.5 & 83.1 & 58.1 \\
Improve & \textcolor{talorgreen}{+0.2} & \textcolor{talorgreen}{+10.7} & \textcolor{talorgreen}{+0.9} & \textcolor{talorgreen}{+13.7} & \textcolor{talorgreen}{+0.7} & \textcolor{talorgreen}{+11.4} & \textcolor{talorgreen}{+1.1} & \textcolor{talorgreen}{+13.4} & \textcolor{talorgreen}{+0.8} & \textcolor{talorgreen}{+19.0} & \textcolor{talorgreen}{+1.0} & \textcolor{talorgreen}{+12.0} & \textcolor{talorgreen}{+0.8} & \textcolor{talorgreen}{+13.4} & \textcolor{talorgreen}{+4.7} \\
\midrule
DPCore & 55.9 & 75.7 & 53.9 & 77.3 & 46.6 & 81.5 & 51.3 & 81.3 & 56.2 & 64.8 & 50.0 & 73.5 & 52.3 & 75.7 & 61.4 \\
+ TALOR & 56.0 & 83.7 & 54.4 & 88.8 & 47.2 & 89.9 & 51.6 & 90.9 & 56.5 & 78.6 & 50.5 & 81.7 & 52.7 & 85.6 & 64.8 \\
Improve & \textcolor{talorgreen}{+0.1} & \textcolor{talorgreen}{+8.0} & \textcolor{talorgreen}{+0.5} & \textcolor{talorgreen}{+11.5} & \textcolor{talorgreen}{+0.5} & \textcolor{talorgreen}{+8.4} & \textcolor{talorgreen}{+0.3} & \textcolor{talorgreen}{+9.6} & \textcolor{talorgreen}{+0.3} & \textcolor{talorgreen}{+13.8} & \textcolor{talorgreen}{+0.5} & \textcolor{talorgreen}{+8.2} & \textcolor{talorgreen}{+0.4} & \textcolor{talorgreen}{+9.9} & \textcolor{talorgreen}{+3.4} \\
\bottomrule
\end{tabular}
}
\end{table*}

\vspace{-5pt}
\section{Computational Complexity}
\vspace{-5pt}
\label{app:time_and_mem}

\definecolor{talorgreen}{HTML}{1B5E20}
\definecolor{talorred}{HTML}{8B1A1A}

\begin{table*}[h]
\centering
\caption{Comparisons w.r.t. computation complexity. Accuracy, AUC, and H-score are reported in \%. TALOR adds negligible computational cost. The wall-clock time (Min) and memory usage (MiB) are measured for processing 15 consecutive corruptions, where each domain is constructed by mixing 5,000 csID and 5,000 csOOD images, batch size = 64; all Time/Memory values are kept as integers. Mem\(\uparrow\) (\%) denotes the relative GPU-memory increment after adding TALOR, i.e., \(\frac{\text{Memory}_{+TALOR}-\text{Memory}_{Origin}}{\text{Memory}_{Origin}}\times 100\%\). All experiments are conducted on an NVIDIA Quadro P6000.}
\label{tab:complexity}
\setlength{\tabcolsep}{2.0pt}
\renewcommand{\arraystretch}{0.9}
\small
\begin{tabular}{l|ccc cc|ccc cc|ccc cc}
\toprule
\multirow{2}{*}{Method}
& \multicolumn{5}{c|}{Origin}
& \multicolumn{5}{c|}{+TALOR}
& \multicolumn{5}{c}{\(\Delta\) (+TALOR $-$ Origin)} \\
\cmidrule(lr){2-6}\cmidrule(lr){7-11}\cmidrule(lr){12-16}
& Acc & AUC & \multicolumn{1}{c|}{H-s} & Time & Mem
& Acc & AUC & \multicolumn{1}{c|}{H-s} & Time & Mem
& Acc & AUC & \multicolumn{1}{c|}{H-s} & Time & Mem(\%) \\
\midrule
Source  & 49.8 & 68.0 & \multicolumn{1}{c|}{56.4} & 29  & 1,565
        & 55.9 & 83.5 & \multicolumn{1}{c|}{66.1} & 30  & 1,717
        & \textcolor{talorgreen}{+6.1} & \textcolor{talorgreen}{+15.5} & \multicolumn{1}{c|}{\textcolor{talorgreen}{+9.7}}  & +1  & +10 \\
CoTTA   & 49.6 & 64.6 & \multicolumn{1}{c|}{54.9} & 212 & 21,064
        & 55.6 & 82.7 & \multicolumn{1}{c|}{65.6} & 214 & 21,197
        & \textcolor{talorgreen}{+6.0} & \textcolor{talorgreen}{+18.2} & \multicolumn{1}{c|}{\textcolor{talorgreen}{+10.8}} & +2  & +1 \\
EATA    & 45.5 & 62.6 & \multicolumn{1}{c|}{48.2} & 42  & 7,123
        & 46.0 & 73.9 & \multicolumn{1}{c|}{53.4} & 43  & 7,255
        & \textcolor{talorgreen}{+0.5} & \textcolor{talorgreen}{+11.3} & \multicolumn{1}{c|}{\textcolor{talorgreen}{+5.2}}  & +1  & +2 \\
SAR     & 50.2 & 55.2 & \multicolumn{1}{c|}{51.1} & 114 & 7,073
        & 54.4 & 84.3 & \multicolumn{1}{c|}{65.4} & 115 & 7,353
        & \textcolor{talorgreen}{+4.2} & \textcolor{talorgreen}{+29.0} & \multicolumn{1}{c|}{\textcolor{talorgreen}{+14.4}} & +1  & +4 \\
OSTTA   & 56.1 & 61.8 & \multicolumn{1}{c|}{58.4} & 101 & 12,891
        & 57.4 & 85.3 & \multicolumn{1}{c|}{68.1} & 103 & 13,209
        & \textcolor{talorgreen}{+1.2} & \textcolor{talorgreen}{+23.5} & \multicolumn{1}{c|}{\textcolor{talorgreen}{+9.6}}  & +2  & +2 \\
ViDA    & 50.4 & 66.7 & \multicolumn{1}{c|}{56.4} & 469 & 12,089
        & 56.4 & 83.7 & \multicolumn{1}{c|}{66.6} & 465 & 12,253
        & \textcolor{talorgreen}{+6.0} & \textcolor{talorgreen}{+17.1} & \multicolumn{1}{c|}{\textcolor{talorgreen}{+10.1}} & \textcolor{talorred}{-4} & +1 \\
UniEnt  & 56.7 & 76.7 & \multicolumn{1}{c|}{63.9} & 74  & 14,217
        & 57.5 & 84.1 & \multicolumn{1}{c|}{67.3} & 75  & 14,349
        & \textcolor{talorgreen}{+0.8} & \textcolor{talorgreen}{+7.4}  & \multicolumn{1}{c|}{\textcolor{talorgreen}{+3.4}}  & +1  & +1 \\
STAMP   & 51.9 & 74.2 & \multicolumn{1}{c|}{60.2} & 181 & 11,307
        & 54.3 & 86.0 & \multicolumn{1}{c|}{65.9} & 182 & 11,165
        & \textcolor{talorgreen}{+2.4} & \textcolor{talorgreen}{+11.9} & \multicolumn{1}{c|}{\textcolor{talorgreen}{+5.7}}  & +1  & \textcolor{talorred}{-1} \\
COME    & 57.0 & 74.6 & \multicolumn{1}{c|}{63.7} & 42  & 7,123
        & 57.7 & 86.5 & \multicolumn{1}{c|}{68.4} & 42  & 7,257
        & \textcolor{talorgreen}{+0.7} & \textcolor{talorgreen}{+11.9} & \multicolumn{1}{c|}{\textcolor{talorgreen}{+4.7}}  & 0   & +2 \\
DPCore  & 54.8 & 76.6 & \multicolumn{1}{c|}{63.5} & 116 & 11,007
        & 55.1 & 85.6 & \multicolumn{1}{c|}{66.6} & 115 & 11,141
        & \textcolor{talorgreen}{+0.3} & \textcolor{talorgreen}{+9.0}  & \multicolumn{1}{c|}{\textcolor{talorgreen}{+3.1}}  & \textcolor{talorred}{-1} & +1 \\
\bottomrule
\end{tabular}
\end{table*}

\vspace{-5pt}
\section{Limitations and Broader Impacts}
\vspace{-5pt}
\label{app:limitations_broader_impacts}

\textbf{Limitations.}
TALOR is built upon the geometry induced by a frozen linear classifier head, which provides a simple, task-aligned anchor for test-time feature rectification. Our extensive diagnostics and ablations show that, across different backbones, OOD ratios, open-set streams, closed-set settings, and dynamic continual protocols, harmful target-stream shifts often contain a removable component in the classifier-amplified principal coordinates, while tail coordinates provide useful cues for estimating this component. Extending the head-induced rectification principle to richer prediction tasks, multimodal models, and alternative classifier parameterizations is an interesting direction for future work. Since TALOR is backward-free, source-free, and modular, such extensions may be achieved by constructing task-aligned coordinate systems and lightweight online estimators without changing the underlying model architecture.

\textbf{Broader impacts.}
TALOR offers an innovative and lightweight route toward robust test-time adaptation. By using the frozen source head as a geometric anchor, TALOR performs feature rectification without backpropagation, costly augmentation loops, or large mutable memories. This makes it suitable for low-latency and resource-constrained deployment scenarios, such as edge recognition systems, real-time monitoring, and continually changing visual environments. In open-set streams, TALOR avoids directly updating model parameters from unlabeled target samples, reducing the risk that unknown samples contaminate the source decision boundary. More broadly, its plug-and-play design can make robust adaptation more accessible across deployed recognition systems, especially when computational budgets, privacy constraints, or deployment stability make conventional online optimization undesirable. As with other robustness methods, TALOR should be used together with standard monitoring and calibration practices when deployed in high-stakes applications.

%%%%%%%%%%%%%%%%%%%%%%%%%%%%%%%%%%%%%%%%%%%%%%%%%%%%%%%%%%%%

\newpage
\clearpage
\input{checklist.tex}

\end{document}

%% file: checklist.tex
\section*{NeurIPS Paper Checklist}

\begin{enumerate}

\item {\bf Claims}
    \item[] Question: Do the main claims made in the abstract and introduction accurately reflect the paper's contributions and scope?
    \item[] Answer: \answerYes{} % Replace by \answerYes{}, \answerNo{}, or \answerNA{}.
    \item[] Justification: The abstract and introduction state the scope of TALOR as a backward-free feature-rectification plug-in for open-set test-time adaptation. The main claims are supported by the method description, main OSTTA results, mechanism validation experiment, closed-set evaluations, ablations, and computational analysis.
    \item[] Guidelines:
    \begin{itemize}
        \item The answer \answerNA{} means that the abstract and introduction do not include the claims made in the paper.
        \item The abstract and/or introduction should clearly state the claims made, including the contributions made in the paper and important assumptions and limitations. A \answerNo{} or \answerNA{} answer to this question will not be perceived well by the reviewers. 
        \item The claims made should match theoretical and experimental results, and reflect how much the results can be expected to generalize to other settings. 
        \item It is fine to include aspirational goals as motivation as long as it is clear that these goals are not attained by the paper. 
    \end{itemize}

\item {\bf Limitations}
    \item[] Question: Does the paper discuss the limitations of the work performed by the authors?
    \item[] Answer: \answerYes{} % Replace by \answerYes{}, \answerNo{}, or \answerNA{}.
    \item[] Justification: Limitations are discussed in App.~\ref{app:limitations_broader_impacts}.
    \item[] Guidelines:
    \begin{itemize}
        \item The answer \answerNA{} means that the paper has no limitation while the answer \answerNo{} means that the paper has limitations, but those are not discussed in the paper. 
        \item The authors are encouraged to create a separate ``Limitations'' section in their paper.
        \item The paper should point out any strong assumptions and how robust the results are to violations of these assumptions (e.g., independence assumptions, noiseless settings, model well-specification, asymptotic approximations only holding locally). The authors should reflect on how these assumptions might be violated in practice and what the implications would be.
        \item The authors should reflect on the scope of the claims made, e.g., if the approach was only tested on a few datasets or with a few runs. In general, empirical results often depend on implicit assumptions, which should be articulated.
        \item The authors should reflect on the factors that influence the performance of the approach. For example, a facial recognition algorithm may perform poorly when image resolution is low or images are taken in low lighting. Or a speech-to-text system might not be used reliably to provide closed captions for online lectures because it fails to handle technical jargon.
        \item The authors should discuss the computational efficiency of the proposed algorithms and how they scale with dataset size.
        \item If applicable, the authors should discuss possible limitations of their approach to address problems of privacy and fairness.
        \item While the authors might fear that complete honesty about limitations might be used by reviewers as grounds for rejection, a worse outcome might be that reviewers discover limitations that aren't acknowledged in the paper. The authors should use their best judgment and recognize that individual actions in favor of transparency play an important role in developing norms that preserve the integrity of the community. Reviewers will be specifically instructed to not penalize honesty concerning limitations.
    \end{itemize}

\item {\bf Theory assumptions and proofs}
    \item[] Question: For each theoretical result, does the paper provide the full set of assumptions and a complete (and correct) proof?
    \item[] Answer: \answerYes{} % Replace by \answerYes{}, \answerNo{}, or \answerNA{}.
    \item[] Proposition 1 is formally stated under Assumption 1 in Sec. 4.2, with its complete proof provided in App. D.1. The result characterizes the population tail-conditioned ridge estimator, and its scope and limitations are explicitly discussed in App. D.
    \item[] Guidelines:
    \begin{itemize}
        \item The answer \answerNA{} means that the paper does not include theoretical results. 
        \item All the theorems, formulas, and proofs in the paper should be numbered and cross-referenced.
        \item All assumptions should be clearly stated or referenced in the statement of any theorems.
        \item The proofs can either appear in the main paper or the supplemental material, but if they appear in the supplemental material, the authors are encouraged to provide a short proof sketch to provide intuition. 
        \item Inversely, any informal proof provided in the core of the paper should be complemented by formal proofs provided in appendix or supplemental material.
        \item Theorems and Lemmas that the proof relies upon should be properly referenced. 
    \end{itemize}

    \item {\bf Experimental result reproducibility}
    \item[] Question: Does the paper fully disclose all the information needed to reproduce the main experimental results of the paper to the extent that it affects the main claims and/or conclusions of the paper (regardless of whether the code and data are provided or not)?
    \item[] Answer: \answerYes{} % Replace by \answerYes{}, \answerNo{}, or \answerNA{}.
    \item[] Justification: The paper specifies the TALOR algorithm, evaluation protocol, benchmarks, evaluation metrics, batch size, OOD scoring rule, and fixed hyperparameters used for the main experiments. Additional backbone, OOD-ratio, and computational details are provided in the appendices.
    \item[] Guidelines:
    \begin{itemize}
        \item The answer \answerNA{} means that the paper does not include experiments.
        \item If the paper includes experiments, a \answerNo{} answer to this question will not be perceived well by the reviewers: Making the paper reproducible is important, regardless of whether the code and data are provided or not.
        \item If the contribution is a dataset and\slash or model, the authors should describe the steps taken to make their results reproducible or verifiable. 
        \item Depending on the contribution, reproducibility can be accomplished in various ways. For example, if the contribution is a novel architecture, describing the architecture fully might suffice, or if the contribution is a specific model and empirical evaluation, it may be necessary to either make it possible for others to replicate the model with the same dataset, or provide access to the model. In general. releasing code and data is often one good way to accomplish this, but reproducibility can also be provided via detailed instructions for how to replicate the results, access to a hosted model (e.g., in the case of a large language model), releasing of a model checkpoint, or other means that are appropriate to the research performed.
        \item While NeurIPS does not require releasing code, the conference does require all submissions to provide some reasonable avenue for reproducibility, which may depend on the nature of the contribution. For example
        \begin{enumerate}
            \item If the contribution is primarily a new algorithm, the paper should make it clear how to reproduce that algorithm.
            \item If the contribution is primarily a new model architecture, the paper should describe the architecture clearly and fully.
            \item If the contribution is a new model (e.g., a large language model), then there should either be a way to access this model for reproducing the results or a way to reproduce the model (e.g., with an open-source dataset or instructions for how to construct the dataset).
            \item We recognize that reproducibility may be tricky in some cases, in which case authors are welcome to describe the particular way they provide for reproducibility. In the case of closed-source models, it may be that access to the model is limited in some way (e.g., to registered users), but it should be possible for other researchers to have some path to reproducing or verifying the results.
        \end{enumerate}
    \end{itemize}

\item {\bf Open access to data and code}
    \item[] Question: Does the paper provide open access to the data and code, with sufficient instructions to faithfully reproduce the main experimental results, as described in supplemental material?
    \item[] Answer: \answerYes{} % Replace by \answerYes{}, \answerNo{}, or \answerNA{}.
    \item[] Justification: The code will be released upon acceptance.
    \item[] Guidelines:
    \begin{itemize}
        \item The answer \answerNA{} means that paper does not include experiments requiring code.
        \item Please see the NeurIPS code and data submission guidelines (\url{https://neurips.cc/public/guides/CodeSubmissionPolicy}) for more details.
        \item While we encourage the release of code and data, we understand that this might not be possible, so \answerNo{} is an acceptable answer. Papers cannot be rejected simply for not including code, unless this is central to the contribution (e.g., for a new open-source benchmark).
        \item The instructions should contain the exact command and environment needed to run to reproduce the results. See the NeurIPS code and data submission guidelines (\url{https://neurips.cc/public/guides/CodeSubmissionPolicy}) for more details.
        \item The authors should provide instructions on data access and preparation, including how to access the raw data, preprocessed data, intermediate data, and generated data, etc.
        \item The authors should provide scripts to reproduce all experimental results for the new proposed method and baselines. If only a subset of experiments are reproducible, they should state which ones are omitted from the script and why.
        \item At submission time, to preserve anonymity, the authors should release anonymized versions (if applicable).
        \item Providing as much information as possible in supplemental material (appended to the paper) is recommended, but including URLs to data and code is permitted.
    \end{itemize}

\item {\bf Experimental setting/details}
    \item[] Question: Does the paper specify all the training and test details (e.g., data splits, hyperparameters, how they were chosen, type of optimizer) necessary to understand the results?
    \item[] Answer: \answerYes{} % Replace by \answerYes{}, \answerNo{}, or \answerNA{}.
    \item[] Justification: Already illustrate in Experimental result reproducibility part.
    \item[] Guidelines:
    \begin{itemize}
        \item The answer \answerNA{} means that the paper does not include experiments.
        \item The experimental setting should be presented in the core of the paper to a level of detail that is necessary to appreciate the results and make sense of them.
        \item The full details can be provided either with the code, in appendix, or as supplemental material.
    \end{itemize}

\item {\bf Experiment statistical significance}
    \item[] Question: Does the paper report error bars suitably and correctly defined or other appropriate information about the statistical significance of the experiments?
    \item[] Answer: \answerNo{} % Replace by \answerYes{}, \answerNo{}, or \answerNA{}.
    \item[] Justification: The paper reports averages over corruption types and multiple datasets. We focus on consistent improvements across datasets, baselines, ablations, and computational comparisons.
    \item[] Guidelines:
    \begin{itemize}
        \item The answer \answerNA{} means that the paper does not include experiments.
        \item The authors should answer \answerYes{} if the results are accompanied by error bars, confidence intervals, or statistical significance tests, at least for the experiments that support the main claims of the paper.
        \item The factors of variability that the error bars are capturing should be clearly stated (for example, train/test split, initialization, random drawing of some parameter, or overall run with given experimental conditions).
        \item The method for calculating the error bars should be explained (closed form formula, call to a library function, bootstrap, etc.)
        \item The assumptions made should be given (e.g., Normally distributed errors).
        \item It should be clear whether the error bar is the standard deviation or the standard error of the mean.
        \item It is OK to report 1-sigma error bars, but one should state it. The authors should preferably report a 2-sigma error bar than state that they have a 96\% CI, if the hypothesis of Normality of errors is not verified.
        \item For asymmetric distributions, the authors should be careful not to show in tables or figures symmetric error bars that would yield results that are out of range (e.g., negative error rates).
        \item If error bars are reported in tables or plots, the authors should explain in the text how they were calculated and reference the corresponding figures or tables in the text.
    \end{itemize}

\item {\bf Experiments compute resources}
    \item[] Question: For each experiment, does the paper provide sufficient information on the computer resources (type of compute workers, memory, time of execution) needed to reproduce the experiments?
    \item[] Answer: \answerYes{} % Replace by \answerYes{}, \answerNo{}, or \answerNA{}.
    \item[] Justification: App.~\ref{app:time_and_mem} reports wall-clock time and GPU memory usage for the main baselines with and without TALOR. The experiments are conducted on a single NVIDIA Quadro P6000, and the table specifies the processed stream size and batch size.
    \item[] Guidelines:
    \begin{itemize}
        \item The answer \answerNA{} means that the paper does not include experiments.
        \item The paper should indicate the type of compute workers CPU or GPU, internal cluster, or cloud provider, including relevant memory and storage.
        \item The paper should provide the amount of compute required for each of the individual experimental runs as well as estimate the total compute. 
        \item The paper should disclose whether the full research project required more compute than the experiments reported in the paper (e.g., preliminary or failed experiments that didn't make it into the paper). 
    \end{itemize}
    
\item {\bf Code of ethics}
    \item[] Question: Does the research conducted in the paper conform, in every respect, with the NeurIPS Code of Ethics \url{https://neurips.cc/public/EthicsGuidelines}?
    \item[] Answer: \answerYes{} % Replace by \answerYes{}, \answerNo{}, or \answerNA{}.
    \item[] Justification: The research uses public benchmarks and does not involve human-subject experiments, private data. We have reviewed the NeurIPS Code of Ethics and do not identify any deviation from it.
    \item[] Guidelines:
    \begin{itemize}
        \item The answer \answerNA{} means that the authors have not reviewed the NeurIPS Code of Ethics.
        \item If the authors answer \answerNo, they should explain the special circumstances that require a deviation from the Code of Ethics.
        \item The authors should make sure to preserve anonymity (e.g., if there is a special consideration due to laws or regulations in their jurisdiction).
    \end{itemize}

\item {\bf Broader impacts}
    \item[] Question: Does the paper discuss both potential positive societal impacts and negative societal impacts of the work performed?
    \item[] Answer: \answerYes{} % Replace by \answerYes{}, \answerNo{}, or \answerNA{}.
    \item[] Justification: Broader impacts are discussed in App.~\ref{app:limitations_broader_impacts}.
    \item[] Guidelines:
    \begin{itemize}
        \item The answer \answerNA{} means that there is no societal impact of the work performed.
        \item If the authors answer \answerNA{} or \answerNo, they should explain why their work has no societal impact or why the paper does not address societal impact.
        \item Examples of negative societal impacts include potential malicious or unintended uses (e.g., disinformation, generating fake profiles, surveillance), fairness considerations (e.g., deployment of technologies that could make decisions that unfairly impact specific groups), privacy considerations, and security considerations.
        \item The conference expects that many papers will be foundational research and not tied to particular applications, let alone deployments. However, if there is a direct path to any negative applications, the authors should point it out. For example, it is legitimate to point out that an improvement in the quality of generative models could be used to generate Deepfakes for disinformation. On the other hand, it is not needed to point out that a generic algorithm for optimizing neural networks could enable people to train models that generate Deepfakes faster.
        \item The authors should consider possible harms that could arise when the technology is being used as intended and functioning correctly, harms that could arise when the technology is being used as intended but gives incorrect results, and harms following from (intentional or unintentional) misuse of the technology.
        \item If there are negative societal impacts, the authors could also discuss possible mitigation strategies (e.g., gated release of models, providing defenses in addition to attacks, mechanisms for monitoring misuse, mechanisms to monitor how a system learns from feedback over time, improving the efficiency and accessibility of ML).
    \end{itemize}
    
\item {\bf Safeguards}
    \item[] Question: Does the paper describe safeguards that have been put in place for responsible release of data or models that have a high risk for misuse (e.g., pre-trained language models, image generators, or scraped datasets)?
    \item[] Answer: \answerNA{} % Replace by \answerYes{}, \answerNo{}, or \answerNA{}.
    \item[] Justification: The work proposes a test-time feature-rectification algorithm evaluated on existing public benchmarks.
    \item[] Guidelines:
    \begin{itemize}
        \item The answer \answerNA{} means that the paper poses no such risks.
        \item Released models that have a high risk for misuse or dual-use should be released with necessary safeguards to allow for controlled use of the model, for example by requiring that users adhere to usage guidelines or restrictions to access the model or implementing safety filters. 
        \item Datasets that have been scraped from the Internet could pose safety risks. The authors should describe how they avoided releasing unsafe images.
        \item We recognize that providing effective safeguards is challenging, and many papers do not require this, but we encourage authors to take this into account and make a best faith effort.
    \end{itemize}

\item {\bf Licenses for existing assets}
    \item[] Question: Are the creators or original owners of assets (e.g., code, data, models), used in the paper, properly credited and are the license and terms of use explicitly mentioned and properly respected?
    \item[] Answer: \answerNo{} % Replace by \answerYes{}, \answerNo{}, or \answerNA{}.
    \item[] Justification: Licenses will be included with the code release.
    \item[] Guidelines:
    \begin{itemize}
        \item The answer \answerNA{} means that the paper does not use existing assets.
        \item The authors should cite the original paper that produced the code package or dataset.
        \item The authors should state which version of the asset is used and, if possible, include a URL.
        \item The name of the license (e.g., CC-BY 4.0) should be included for each asset.
        \item For scraped data from a particular source (e.g., website), the copyright and terms of service of that source should be provided.
        \item If assets are released, the license, copyright information, and terms of use in the package should be provided. For popular datasets, \url{paperswithcode.com/datasets} has curated licenses for some datasets. Their licensing guide can help determine the license of a dataset.
        \item For existing datasets that are re-packaged, both the original license and the license of the derived asset (if it has changed) should be provided.
        \item If this information is not available online, the authors are encouraged to reach out to the asset's creators.
    \end{itemize}

\item {\bf New assets}
    \item[] Question: Are new assets introduced in the paper well documented and is the documentation provided alongside the assets?
    \item[] Answer: \answerNA{} % Replace by \answerYes{}, \answerNo{}, or \answerNA{}.
    \item[] Justification: The paper does not introduce a new dataset, benchmark, pretrained model, or other released asset. 
    \item[] Guidelines:
    \begin{itemize}
        \item The answer \answerNA{} means that the paper does not release new assets.
        \item Researchers should communicate the details of the dataset\slash code\slash model as part of their submissions via structured templates. This includes details about training, license, limitations, etc. 
        \item The paper should discuss whether and how consent was obtained from people whose asset is used.
        \item At submission time, remember to anonymize your assets (if applicable). You can either create an anonymized URL or include an anonymized zip file.
    \end{itemize}

\item {\bf Crowdsourcing and research with human subjects}
    \item[] Question: For crowdsourcing experiments and research with human subjects, does the paper include the full text of instructions given to participants and screenshots, if applicable, as well as details about compensation (if any)? 
    \item[] Answer: \answerNA{} % Replace by \answerYes{}, \answerNo{}, or \answerNA{}.
    \item[] Justification: The work does not involve crowdsourcing.
    \item[] Guidelines:
    \begin{itemize}
        \item The answer \answerNA{} means that the paper does not involve crowdsourcing nor research with human subjects.
        \item Including this information in the supplemental material is fine, but if the main contribution of the paper involves human subjects, then as much detail as possible should be included in the main paper. 
        \item According to the NeurIPS Code of Ethics, workers involved in data collection, curation, or other labor should be paid at least the minimum wage in the country of the data collector. 
    \end{itemize}

\item {\bf Institutional review board (IRB) approvals or equivalent for research with human subjects}
    \item[] Question: Does the paper describe potential risks incurred by study participants, whether such risks were disclosed to the subjects, and whether Institutional Review Board (IRB) approvals (or an equivalent approval/review based on the requirements of your country or institution) were obtained?
    \item[] Answer: \answerNA{} % Replace by \answerYes{}, \answerNo{}, or \answerNA{}.
    \item[] Justification:  The work does not involve human-subject research.
    \item[] Guidelines:
    \begin{itemize}
        \item The answer \answerNA{} means that the paper does not involve crowdsourcing nor research with human subjects.
        \item Depending on the country in which research is conducted, IRB approval (or equivalent) may be required for any human subjects research. If you obtained IRB approval, you should clearly state this in the paper. 
        \item We recognize that the procedures for this may vary significantly between institutions and locations, and we expect authors to adhere to the NeurIPS Code of Ethics and the guidelines for their institution. 
        \item For initial submissions, do not include any information that would break anonymity (if applicable), such as the institution conducting the review.
    \end{itemize}

\item {\bf Declaration of LLM usage}
    \item[] Question: Does the paper describe the usage of LLMs if it is an important, original, or non-standard component of the core methods in this research? Note that if the LLM is used only for writing, editing, or formatting purposes and does \emph{not} impact the core methodology, scientific rigor, or originality of the research, declaration is not required.
    %this research? 
    \item[] Answer: \answerNA{} % Replace by \answerYes{}, \answerNo{}, or \answerNA{}.
    \item[] Justification: The core method development, experiments, and scientific claims do not rely on LLMs as an important, original, or non-standard component. Any writing, editing, or formatting assistance does not affect the methodology, scientific rigor, or originality of the research.
    \item[] Guidelines:
    \begin{itemize}
        \item The answer \answerNA{} means that the core method development in this research does not involve LLMs as any important, original, or non-standard components.
        \item Please refer to our LLM policy in the NeurIPS handbook for what should or should not be described.
    \end{itemize}

\end{enumerate}

%% file: nips2026.bib
@inproceedings{chen2022contrastive,
  title={Contrastive test-time adaptation},
  author={Chen, Dian and Wang, Dequan and Darrell, Trevor and Ebrahimi, Sayna},
  booktitle={Proceedings of the IEEE/CVF Conference on Computer Vision and Pattern Recognition},
  pages={295--305},
  year={2022}
}

@inproceedings{sanqing2024LEAD,
  title={LEAD: Learning Decomposition for Source-free Universal Domain Adaptation},
  author={Qu, Sanqing and Zou, Tianpei and He, Lianghua and Röhrbein, Florian and Knoll, Alois and Chen, Guang and Jiang, Changjun},
  booktitle={CVPR},
  year={2024},
}

@article{hendrycks2021many,
  title={The Many Faces of Robustness: A Critical Analysis of Out-of-Distribution Generalization},
  author={Dan Hendrycks and Steven Basart and Norman Mu and Saurav Kadavath and Frank Wang and Evan Dorundo and Rahul Desai and Tyler Zhu and Samyak Parajuli and Mike Guo and Dawn Song and Jacob Steinhardt and Justin Gilmer},
  journal={ICCV},
  year={2021}
}

@inproceedings{wang2019learning,
        title={Learning Robust Global Representations by Penalizing Local Predictive Power},
        author={Wang, Haohan and Ge, Songwei and Lipton, Zachary and Xing, Eric P},
        booktitle={Advances in Neural Information Processing Systems},
        pages={10506--10518},
        year={2019}
}

@inproceedings{tang2024domain,
  title={Domain-conditioned transformer for fully test-time adaptation},
  author={Tang, Yushun and Chen, Shuoshuo and Jia, Jiyuan and Zhang, Yi and He, Zhihai},
  booktitle={Proceedings of the 32nd ACM International Conference on Multimedia},
  pages={6260--6269},
  year={2024}
}

@inproceedings{yuan2024tea,
  title={Tea: Test-time energy adaptation},
  author={Yuan, Yige and Xu, Bingbing and Hou, Liang and Sun, Fei and Shen, Huawei and Cheng, Xueqi},
  booktitle={Proceedings of the IEEE/CVF Conference on Computer Vision and Pattern Recognition},
  pages={23901--23911},
  year={2024}
}

@article{lee2024entropy,
  title={Entropy is not enough for test-time adaptation: From the perspective of disentangled factors},
  author={Lee, Jonghyun and Jung, Dahuin and Lee, Saehyung and Park, Junsung and Shin, Juhyeon and Hwang, Uiwon and Yoon, Sungroh},
  journal={arXiv preprint arXiv:2403.07366},
  year={2024}
}

@inproceedings{boudiaf2022parameter,
  title={Parameter-free online test-time adaptation},
  author={Boudiaf, Malik and Mueller, Romain and Ben Ayed, Ismail and Bertinetto, Luca},
  booktitle={Proceedings of the IEEE/CVF Conference on Computer Vision and Pattern Recognition},
  pages={8344--8353},
  year={2022}
}

@article{niu2024test,
  title={Test-time model adaptation with only forward passes},
  author={Niu, Shuaicheng and Miao, Chunyan and Chen, Guohao and Wu, Pengcheng and Zhao, Peilin},
  journal={arXiv preprint arXiv:2404.01650},
  year={2024}
}

@article{gui2024active,
  title={Active test-time adaptation: Theoretical analyses and an algorithm},
  author={Gui, Shurui and Li, Xiner and Ji, Shuiwang},
  journal={arXiv preprint arXiv:2404.05094},
  year={2024}
}

@inproceedings{wang2025effortless,
  title={Effortless active labeling for long-term test-time adaptation},
  author={Wang, Guowei and Ding, Changxing},
  booktitle={Proceedings of the Computer Vision and Pattern Recognition Conference},
  pages={25633--25642},
  year={2025}
}

@article{lee2025test,
  title={Test-Time Adaptation with Binary Feedback},
  author={Lee, Taeckyung and Chottananurak, Sorn and Kim, Junsu and Shin, Jinwoo and Gong, Taesik and Lee, Sung-Ju},
  journal={arXiv preprint arXiv:2505.18514},
  year={2025}
}

@inproceedings{Gong20264DPC,
  title={4D Point Cloud Segmentation via Active Test-Time Adaptation},
  author={Mingrong Gong and Chao-Qi Chen and Lu-Yao Tang and Yuxiao Wang and Sergio Escalera},
  booktitle={AAAI Conference on Artificial Intelligence},
  year={2026},
  url={https://api.semanticscholar.org/CorpusID:286691532}
}

@inproceedings{wang2022continual,
  title={Continual test-time domain adaptation},
  author={Wang, Qin and Fink, Olga and Van Gool, Luc and Dai, Dengxin},
  booktitle={Proceedings of the IEEE/CVF Conference on Computer Vision and Pattern Recognition},
  pages={7201--7211},
  year={2022}
}

@inproceedings{wang2023multiple,
  title={Multiple teacher model for continual test-time domain adaptation},
  author={Wang, Ran and Zuo, Hua and Fang, Zhen and Lu, Jie},
  booktitle={Australasian Joint Conference on Artificial Intelligence},
  pages={304--314},
  year={2023},
  organization={Springer}
}

@inproceedings{cao2023multi,
  title={Multi-modal continual test-time adaptation for 3d semantic segmentation},
  author={Cao, Haozhi and Xu, Yuecong and Yang, Jianfei and Yin, Pengyu and Yuan, Shenghai and Xie, Lihua},
  booktitle={Proceedings of the IEEE/CVF international conference on computer vision},
  pages={18809--18819},
  year={2023}
}

@article{lee2024becotta,
  title={Becotta: Input-dependent online blending of experts for continual test-time adaptation},
  author={Lee, Daeun and Yoon, Jaehong and Hwang, Sung Ju},
  journal={arXiv preprint arXiv:2402.08712},
  year={2024}
}

@inproceedings{chen2024each,
  title={Each test image deserves a specific prompt: Continual test-time adaptation for 2d medical image segmentation},
  author={Chen, Ziyang and Pan, Yongsheng and Ye, Yiwen and Lu, Mengkang and Xia, Yong},
  booktitle={Proceedings of the IEEE/CVF conference on computer vision and pattern recognition},
  pages={11184--11193},
  year={2024}
}

@inproceedings{hong2023mecta,
  title={Mecta: Memory-economic continual test-time model adaptation},
  author={Hong, Junyuan and Lyu, Lingjuan and Zhou, Jiayu and Spranger, Michael},
  booktitle={2023 International Conference on Learning Representations},
  year={2023}
}

@inproceedings{niu2022efficient,
  title={Efficient test-time model adaptation without forgetting},
  author={Niu, Shuaicheng and Wu, Jiaxiang and Zhang, Yifan and Chen, Yaofo and Zheng, Shijian and Zhao, Peilin and Tan, Mingkui},
  booktitle={International conference on machine learning},
  pages={16888--16905},
  year={2022},
  organization={PMLR}
}

@article{niu2023towards,
  title={Towards stable test-time adaptation in dynamic wild world},
  author={Niu, Shuaicheng and Wu, Jiaxiang and Zhang, Yifan and Wen, Zhiquan and Chen, Yaofo and Zhao, Peilin and Tan, Mingkui},
  journal={arXiv preprint arXiv:2302.12400},
  year={2023}
}

@article{jiang2025feature,
  title={Feature-Based Instance Neighbor Discovery: Advanced Stable Test-Time Adaptation in Dynamic World},
  author={Jiang, Qinting and Ye, Chuyang and Wei, Dongyan and Wang, Bingli and Xue, Yuan and Jiang, Jingyan and Wang, Zhi},
  journal={arXiv preprint arXiv:2506.06782},
  year={2025}
}

@inproceedings{gan2023decorate,
  title={Decorate the newcomers: Visual domain prompt for continual test time adaptation},
  author={Gan, Yulu and Bai, Yan and Lou, Yihang and Ma, Xianzheng and Zhang, Renrui and Shi, Nian and Luo, Lin},
  booktitle={Proceedings of the AAAI conference on artificial intelligence},
  volume={37},
  pages={7595--7603},
  year={2023}
}

@article{zhang2024dpcore,
  title={Dpcore: Dynamic prompt coreset for continual test-time adaptation},
  author={Zhang, Yunbei and Mehra, Akshay and Niu, Shuaicheng and Hamm, Jihun},
  journal={arXiv preprint arXiv:2406.10737},
  year={2024}
}

@article{liu2023vida,
  title={Vida: Homeostatic visual domain adapter for continual test time adaptation},
  author={Liu, Jiaming and Yang, Senqiao and Jia, Peidong and Zhang, Renrui and Lu, Ming and Guo, Yandong and Xue, Wei and Zhang, Shanghang},
  journal={arXiv preprint arXiv:2306.04344},
  year={2023}
}

@article{wang2025paid,
  title={PAID: Pairwise Angular-Invariant Decomposition for Continual Test-Time Adaptation},
  author={Wang, Kunyu and Fu, Xueyang and Bao, Yuanfei and Ge, Chengjie and Cao, Chengzhi and Zhai, Wei and Zha, Zheng-Jun},
  journal={arXiv preprint arXiv:2506.02453},
  year={2025}
}

@inproceedings{ni2025maintaining,
  title={Maintaining consistent inter-class topology in continual test-time adaptation},
  author={Ni, Chenggong and Lyu, Fan and Tan, Jiayao and Hu, Fuyuan and Yao, Rui and Zhou, Tao},
  booktitle={Proceedings of the Computer Vision and Pattern Recognition Conference},
  pages={15319--15328},
  year={2025}
}

@inproceedings{parkrethinking,
  title={Rethinking Entropy in Test-Time Adaptation: The Missing Piece from Energy Duality},
  author={Park, Mincheol and Won, Heeji and Ro, Won Woo and Kim, Suhyun},
  booktitle={The Thirty-ninth Annual Conference on Neural Information Processing Systems},
  year={2025}
}

@article{liu2025test,
  title={Test-Time Adaptation by Causal Trimming},
  author={Liu, Yingnan and Qiao, Rui and Lee, Mong Li and Hsu, Wynne},
  journal={arXiv preprint arXiv:2510.11133},
  year={2025}
}

@article{Maharana2024PALMPA,
  title={PALM: Pushing Adaptive Learning Rate Mechanisms for Continual Test-Time Adaptation},
  author={Sarthak Kumar Maharana and Baoming Zhang and Yunhui Guo},
  journal={ArXiv},
  year={2024},
  volume={abs/2403.10650},
  url={https://api.semanticscholar.org/CorpusID:268513104}
}

@article{
maharana2026continual,
title={Continual Test-Time Adaptation in Computer Vision: Methods, Benchmarks, and Future Directions},
author={Sarthak Kumar Maharana and Shambhavi Mishra and Yunbei Zhang and Shuaicheng Niu and Taki Hasan Rafi and Jihun Hamm and Marco Pedersoli and Jose Dolz and Yunhui Guo},
journal={Transactions on Machine Learning Research},
issn={2835-8856},
year={2026},
url={https://openreview.net/forum?id=mM3r03Xw1V},
note={Survey Certification}
}

@inproceedings{yu2024stamp,
  title={Stamp: Outlier-aware test-time adaptation with stable memory replay},
  author={Yu, Yongcan and Sheng, Lijun and He, Ran and Liang, Jian},
  booktitle={European Conference on Computer Vision},
  pages={375--392},
  year={2024},
  organization={Springer}
}

@inproceedings{lee2023towards,
  title={Towards open-set test-time adaptation utilizing the wisdom of crowds in entropy minimization},
  author={Lee, Jungsoo and Das, Debasmit and Choo, Jaegul and Choi, Sungha},
  booktitle={Proceedings of the IEEE/CVF International Conference on Computer Vision},
  pages={16380--16389},
  year={2023}
}

@inproceedings{gao2024unified,
  title={Unified entropy optimization for open-set test-time adaptation},
  author={Gao, Zhengqing and Zhang, Xu-Yao and Liu, Cheng-Lin},
  booktitle={Proceedings of the IEEE/CVF Conference on Computer Vision and Pattern Recognition},
  pages={23975--23984},
  year={2024}
}

@article{zhang2024come,
  title={COME: Test-time adaption by conservatively minimizing entropy},
  author={Zhang, Qingyang and Bian, Yatao and Kong, Xinke and Zhao, Peilin and Zhang, Changqing},
  journal={arXiv preprint arXiv:2410.10894},
  year={2024}
}

@article{cao2025noisy,
  title={Noisy Test-Time Adaptation in Vision-Language Models},
  author={Cao, Chentao and Zhong, Zhun and Zhou, Zhanke and Liu, Tongliang and Liu, Yang and Zhang, Kun and Han, Bo},
  journal={arXiv preprint arXiv:2502.14604},
  year={2025}
}

@article{sreenivas2024efficient,
  title={Efficient Open Set Single Image Test Time Adaptation of Vision Language Models},
  author={Sreenivas, Manogna and Biswas, Soma},
  journal={arXiv preprint arXiv:2406.00481},
  year={2024}
}

@inproceedings{li2023robustness,
  title={On the robustness of open-world test-time training: Self-training with dynamic prototype expansion},
  author={Li, Yushu and Xu, Xun and Su, Yongyi and Jia, Kui},
  booktitle={Proceedings of the IEEE/CVF International Conference on Computer Vision},
  pages={11836--11846},
  year={2023}
}

@article{Zhao2026MitigatingTI,
  title={Mitigating the ID–OOD Tradeoff in Open-Set Test-Time Adaptation},
  author={Wenjie Zhao and Jia Li and Xinyao Dong and Yapeng Tian and Yu Xiang and Yunhui Guo},
  journal={2026 IEEE/CVF Conference on Computer Vision and Pattern Recognition- FINDINGS Track (CVPRF)},
  year={2026},
  pages={6792-6801},
  url={https://api.semanticscholar.org/CorpusID:287072352}
}

@InProceedings{Yang_2026_CVPR,
    author    = {Yang, Yingkai and Chen, Chaoqi and Huang, Hui},
    title     = {Back to Source: Open-Set Continual Test-Time Adaptation via Domain Compensation},
    booktitle = {Proceedings of the IEEE/CVF Conference on Computer Vision and Pattern Recognition (CVPR)},
    month     = {June},
    year      = {2026},
    pages     = {7957-7966}
}

@inproceedings{
zou2026good,
title={{GOOD}: Geometry-guided Out-of-Distribution Modeling for Open-set Test-time Adaptation in Point Cloud Semantic Segmentation},
author={Tianpei Zou and Guo Yu and Ya Wu and Fan Lu and Zhongcong Xu and Zhang Bo and Ziqiao Wang and Sanqing Qu and Guang Chen},
booktitle={The Fourteenth International Conference on Learning Representations},
year={2026},
url={https://openreview.net/forum?id=HyNWlZd4iO}
}

@article{eckart1936approximation,
  title={The approximation of one matrix by another of lower rank},
  author={Eckart, Carl and Young, Gale},
  journal={Psychometrika},
  volume={1},
  number={3},
  pages={211--218},
  year={1936},
  publisher={Springer-Verlag}
}

@book{woodbury1950inverting,
  title={Inverting modified matrices},
  author={Woodbury, Max A},
  year={1950},
  publisher={Department of Statistics, Princeton University}
}

@article{henderson1981deriving,
  title={On deriving the inverse of a sum of matrices},
  author={Henderson, Harold V and Searle, Shayle R},
  journal={SIAM review},
  volume={23},
  number={1},
  pages={53--60},
  year={1981},
  publisher={SIAM}
}

@book{higham2002accuracy,
  title={Accuracy and stability of numerical algorithms},
  author={Higham, Nicholas J},
  year={2002},
  publisher={SIAM}
}

@article{hastie2005elements,
  title={The elements of statistical learning: data mining, inference and prediction},
  author={Hastie, Trevor and Tibshirani, Robert and Friedman, Jerome and Franklin, James},
  journal={The Mathematical Intelligencer},
  volume={27},
  number={2},
  pages={83--85},
  year={2005},
  publisher={Springer}
}

@inproceedings{xiao2010sun,
  title={Sun database: Large-scale scene recognition from abbey to zoo},
  author={Xiao, Jianxiong and Hays, James and Ehinger, Krista A and Oliva, Aude and Torralba, Antonio},
  booktitle={2010 IEEE computer society conference on computer vision and pattern recognition},
  pages={3485--3492},
  year={2010},
  organization={IEEE}
}

@article{hendrycks2019robustness,
  title={Benchmarking Neural Network Robustness to Common Corruptions and Perturbations},
  author={Dan Hendrycks and Thomas Dietterich},
  journal={Proceedings of the International Conference on Learning Representations},
  year={2019}
}

@article{li2025laion,
  title={LAION-C: An Out-of-Distribution Benchmark for Web-Scale Vision Models},
  author={Li, Fanfei and Klein, Thomas and Brendel, Wieland and Geirhos, Robert and Zimmermann, Roland S},
  journal={arXiv preprint arXiv:2506.16950},
  year={2025}
}

@article{liu2020energy,
  title={Energy-based out-of-distribution detection},
  author={Liu, Weitang and Wang, Xiaoyun and Owens, John and Li, Yixuan},
  journal={Advances in neural information processing systems},
  volume={33},
  pages={21464--21475},
  year={2020}
}

@inproceedings{cimpoi2014describing,
  title={Describing textures in the wild},
  author={Cimpoi, Mircea and Maji, Subhransu and Kokkinos, Iasonas and Mohamed, Sammy and Vedaldi, Andrea},
  booktitle={Proceedings of the IEEE conference on computer vision and pattern recognition},
  pages={3606--3613},
  year={2014}
}

@inproceedings{van2018inaturalist,
  title={The inaturalist species classification and detection dataset},
  author={Van Horn, Grant and Mac Aodha, Oisin and Song, Yang and Cui, Yin and Sun, Chen and Shepard, Alex and Adam, Hartwig and Perona, Pietro and Belongie, Serge},
  booktitle={Proceedings of the IEEE conference on computer vision and pattern recognition},
  pages={8769--8778},
  year={2018}
}

@article{zhou2017places,
  title={Places: A 10 million image database for scene recognition},
  author={Zhou, Bolei and Lapedriza, Agata and Khosla, Aditya and Oliva, Aude and Torralba, Antonio},
  journal={IEEE transactions on pattern analysis and machine intelligence},
  volume={40},
  number={6},
  pages={1452--1464},
  year={2017},
  publisher={IEEE}
}

@article{bitterwolf2023or,
  title={In or out? fixing imagenet out-of-distribution detection evaluation},
  author={Bitterwolf, Julian and Mueller, Maximilian and Hein, Matthias},
  journal={arXiv preprint arXiv:2306.00826},
  year={2023}
}

@InProceedings{vaze2022openset,
               title={Open-Set Recognition: a Good Closed-Set Classifier is All You Need?},
               author={Sagar Vaze and Kai Han and Andrea Vedaldi and Andrew Zisserman},
               booktitle={International Conference on Learning Representations},
               year={2022}}

@InProceedings{Yun_2025_CVPR,
    author    = {Yun, Seokju and Chae, Seunghye and Lee, Dongheon and Ro, Youngmin},
    title     = {SoMA: Singular Value Decomposed Minor Components Adaptation for Domain Generalizable Representation Learning},
    booktitle = {Proceedings of the Computer Vision and Pattern Recognition Conference (CVPR)},
    month     = {June},
    year      = {2025},
    pages     = {25602-25612}
}

@inproceedings{adachi2025testtime,
title={Test-time Adaptation for Regression by Subspace Alignment},
author={Kazuki Adachi and Shin'ya Yamaguchi and Atsutoshi Kumagai and Tomoki Hamagami},
booktitle={The Thirteenth International Conference on Learning Representations},
year={2025},
url={https://openreview.net/forum?id=SXtl7NRyE5}
}
